%% file: main.tex
\documentclass{article}

\input{common}
\input{math_commands.tex}

\usepackage[accepted]{icml2022}

\icmltitlerunning{Data poisoning and Byzantine gradients}

\begin{document}

\twocolumn[
\icmltitle{An Equivalence Between Data Poisoning and Byzantine Gradient Attacks}

\icmlsetsymbol{equal}{*}

\begin{icmlauthorlist}
\icmlauthor{Sadegh Farhadkhani}{equal,epfl}
\icmlauthor{Rachid Guerraoui}{epfl}
\icmlauthor{Lê-Nguyên Hoang}{equal,epfl}
\icmlauthor{Oscar Villemaud}{equal,epfl}
\end{icmlauthorlist}

\icmlaffiliation{epfl}{IC Schoold, EPFL, Lausanne, Switzerland}

\icmlcorrespondingauthor{Sadegh Farhadkhani}{sadegh.farhadkhani@epfl.ch}
\icmlcorrespondingauthor{Lê-Nguyên Hoang}{le.hoang@epfl.ch}

\icmlkeywords{Data poisoning, federated learning, personalized, Byzantine}

\vskip 0.3in
]

\printAffiliationsAndNotice{\icmlEqualContribution} %

\input{abstract}

\input{introduction}

\input{related_work}

\input{model}

\input{reduction}

\input{vulnerability}
\input{PAC}

\input{counter_gradient_attack}

\input{data_poisoning}

\input{conclusion}

\input{acknowledgment} 
\input{ethics}

\bibliography{references}
\bibliographystyle{icml2022}

\onecolumn

\begin{center}
{\huge Appendix}
\end{center}

\appendix
\input{convexity_lemmas}

\input{reduction_proof}
\input{GlobalConvergence_proof}
\input{vulnerability_proof}

\input{convergence_sum_over_average}
\input{specific_models_proofs}

\input{PAC_proof}

\input{cga_proof}

\input{cga_mnist}
\input{cifar}

\input{data_poisoning_least_square}
\input{data_poisoning_appendix}

\end{document}

%% file: common.tex
\usepackage[utf8]{inputenc}
\usepackage{amsmath,amssymb,amsthm}
\usepackage{bm,bbm,xfrac,url,hyperref}
\usepackage{dsfont}
\usepackage{wrapfig}
\usepackage[dvipsnames]{xcolor}

\usepackage{graphicx}
\usepackage{subcaption}

\usepackage{todonotes}

\newcommand{\mathsep}{,~}
\newcommand{\st}{\,\middle|\,}
\newcommand{\set}[1]{\left\lbrace #1 \right\rbrace}
\newcommand{\card}[1]{\left\lvert{#1}\right\rvert}
\newcommand{\absv}[1]{\card{#1}}
\newcommand{\norm}[2]{\left\lVert{#1}\right\rVert_{#2}}

\newcommand{\setR}{\mathbb R}

\newcommand{\probability}[1]{\mathbb P\left[#1\right]}
\newcommand{\expect}[1]{\mathbb E\left[#1\right]}

\newcommand{\indicator}[1]{\mathbbm{1}\left[#1\right]}

\newcommand{\expectVariable}[2]{\mathbb E_{#1}\left[#2\right]}

\newtheorem{definition}{Definition}
\newtheorem{proposition}{Proposition}
\newtheorem{theorem}{Theorem}
\newtheorem{lemma}{Lemma}
\newtheorem{corollary}{Corollary}

\newtheorem{assumption}{Assumption}

\newcommand{\param}{\theta}
\newcommand{\paramsub}[1]{\param_{#1}}
\newcommand{\paramfamily}{\vec \theta}
\newcommand{\common}{\rho}

\newcommand{\optimumsub}[1]{\param_{#1}^*}
\newcommand{\optimumfamily}{\vec{\param}^*}

\newcommand{\optcommon}{\common^*}

\newcommand{\trueparam}{\param^\dagger}
\newcommand{\trueparamsub}[1]{\param_{#1}^\dagger}
\newcommand{\trueparamfamily}{\vec{\param}^\dagger}

\newcommand{\query}[1]{\mathcal Q_{#1}}

\newcommand{\querydistribution}[1]{\tilde{\mathcal Q}_{#1}}
\newcommand{\answer}[1]{\mathcal A_{#1}}

\newcommand{\varx}{x}

\newcommand{\regweightsub}[1]{\lambda_{#1}}
\newcommand{\regweightcommon}[1]{\mu}
\newcommand{\reglocalweight}[1]{\nu}

\newcommand{\regnormsub}[1]{q_{#1}}
\newcommand{\regnormcommon}[1]{\regnormsub{0}}
\newcommand{\regpowersub}[1]{p_{#1}}
\newcommand{\regpowercommon}[1]{\regpowersub{0}}

\newcommand{\commonnorm}[1]{\mathcal N_{0}}
\newcommand{\regularization}{\mathcal R}

\newcommand{\Probability}{P}

\newcommand{\parambound}{\mathcal K}
\newcommand{\event}{\mathcal E}

\newcommand{\strategicvote}{\paramsub{\strategicnode}^\spadesuit}

\newcommand{\targetvalue}{\trueparam_{\strategicnode}}

\newcommand{\nodeinput}{i}
\newcommand{\NODEINPUT}[1]{\mathcal{I}_{#1}}

\newcommand{\data}[1]{\mathcal D_{#1}}
\newcommand{\datafamily}[1]{\vec{\mathcal D}_{#1}}

\newcommand{\node}{n}
\newcommand{\NODE}{N}

\newcommand{\strategicnode}{s}

\newcommand{\honest}{h}
\newcommand{\HONEST}{\mathcal{H}}
\newcommand{\byzantine}{f}
\newcommand{\BYZANTINE}{F}

\newcommand{\noise}[1]{\xi_{#1}}
\newcommand{\unitvector}[1]{{\bf u}_{#1}}
\newcommand{\unitvectorbis}[1]{{\bf v}_{#1}}

\newcommand{\utility}[1]{u}

\newcommand{\ball}{\mathcal B}
\newcommand{\averagingconstant}{C}

\newcommand{\lossperinput}{\ell}
\newcommand{\independentloss}[1]{\mathcal L_{#1}}

\newcommand{\globalloss}{\textsc{Loss}}
\newcommand{\ModifiedLoss}[1]{\textsc{Loss}_{#1}}
\newcommand{\localloss}[1]{\mathcal{L}_{#1}}
\newcommand{\minimizedLoss}[1]{\textsc{Loss}^1_{#1}}
\newcommand{\minimizedLosstwo}[1]{\textsc{Loss}^2_{#1}}

\newcommand{\distbound}{M_{\query{}}}

\newcommand{\empiricalcovariance}{\hat{\Sigma}}
\newcommand{\logisticdistance}{\Delta}

\newcommand{\spectrum}{\textsc{Sp}}
\newcommand{\sphere}[1]{\mathbb S^{#1}}
\newcommand{\support}{\textsc{Supp}}

\newcommand{\CounterGradientAttack}{\textsc{CGA}}
\newcommand{\alg}{\textsc{ALG}}

\newcommand{\datapoint}{x}
\newcommand{\gradient}[2]{g_{#1}^{#2}}
\newcommand{\minimizedgradient}[2]{G_{#1}^{#2}}
\newcommand{\nonnormalizedgradient}[2]{h_{#1}^{#2}}
\newcommand{\truegradient}[2]{g_{#1}^{\dagger,#2}}
\newcommand{\estimatedgradient}[2]{\hat g_{#1}^{#2}}
\newcommand{\GRADIENT}{\textsc{Grad}}
\newcommand{\iteration}{t}
\newcommand{\learningrate}[1]{\eta_{#1}}

\newcommand{\error}[2]{\zeta_{#1}^{#2}}

\newcommand{\identitymatrix}{I}

%% file: math_commands.tex
\usepackage{amsmath,amsfonts,bm}

\def\eqref#1{equation~\ref{#1}}

\def\1{\bm{1}}

\DeclareMathAlphabet{\mathsfit}{\encodingdefault}{\sfdefault}{m}{sl}
\SetMathAlphabet{\mathsfit}{bold}{\encodingdefault}{\sfdefault}{bx}{n}

\newcommand{\sigmoid}{\sigma}

\DeclareMathOperator*{\argmin}{arg\,min}

%% file: abstract.tex
\begin{abstract}
To study the resilience of distributed learning, 
the ``Byzantine" literature considers a strong threat model 
where workers can report arbitrary gradients to the parameter server. 
Whereas this model helped obtain several fundamental results, 
it has sometimes been considered unrealistic, when the workers 
are mostly trustworthy machines. 
In this paper, we show a surprising equivalence between this model 
and data poisoning, a threat considered much more realistic. 
More specifically, we prove that every gradient attack can be reduced 
to data poisoning,
in any personalized federated learning system with PAC guarantees 
(which we show are both desirable and realistic).
This equivalence makes it possible to obtain 
new impossibility results on the resilience of \emph{any} ``robust'' learning algorithm to data poisoning
{ in highly heterogeneous applications}, 
as corollaries of existing impossibility theorems on Byzantine machine learning.
Moreover, using our equivalence, we derive a practical attack that we show (theoretically and empirically) can be very effective against classical personalized federated learning models.
\end{abstract}

%% file: introduction.tex
\section{Introduction}

Learning algorithms typically leverage data generated by a large number of users \citep{SmithSPKCL13,WangPNSMHLB19,WangSMHLB19} 
to often learn a common model that fits a large population~\citep{KonecnyMR15}, but also sometimes to construct a \emph{personalized} model for each individual~\citep{RicciRS11}.
Autocompletion~\citep{LehmannB21}, conversational~\citep{ShumHL18} and recommendation~\citep{IeJWNAWCCB19} schemes are examples of such personalization algorithms already deployed at scale.
To be effective, besides huge amounts of  data~\citep{BrownMRSKDNSSAA20,FedusZS21}, these algorithms require customization, motivating research into the promising but challenging field of \emph{personalized federated learning}~\citep{FallahMO20,HanzelyHHR20,DinhTN20}.

{Now, classical learning algorithms generally regard as desirable to fit all available data.
However, this approach dangerously fails in the context of user-generated data, as goal-oriented users may provide \emph{untrustworthy data} to reach their objectives.}
{In fact}, in applications such as content recommendation,
activists, companies, and politicians have strong incentives to do so to promote 
certain views, products or ideologies~\citep{Hoang20,HoangFE21}.
Perhaps unsurprisingly, this led to the proliferation of fabricated activities to bias algorithms~\citep{bradshaw19,neudert2019}, e.g. through ``fake reviews''~\citep{WuNWW20}.
The scale of this phenomenon is well illustrated by the case of Facebook which, in 2019 alone, reported the removal of around 6~billion fake accounts from its platform~\citep{facebook_fake_accounts}.
This is highly concerning in the era of ``stochastic parrots''~\citep{BenderGMS21}: 
climate denialists are incentivized to pollute textual datasets with claims like ``climate change is a hoax'', rightly assuming that autocompletion, conversational and recommendation algorithms trained on such data will more likely spread these views~\citep{McGuffieNewhouse20}.
This raises serious concerns about the vulnerability of personalized federated learning to misleading data.
Data poisoning attacks clearly constitute now a major machine learning security issue { \emph{in already deployed systems}}~\citep{KumarNLMGCSX20}.

Overall, in adversarial environments like social media, and given the advent of \emph{deep fakes}~\citep{JohnsonD21}, we should expect  \emph{most data to be strategically crafted and labeled}. 
In this context, the authentication of the data provider is critical.
In particular, the safety of learning algorithms arguably demands that they be trained solely on {\it cryptographically signed} data, namely, data that provably come from a known source.
But even signed data cannot be wholeheartedly trusted since users typically have preferences over what ought to be recommended to others. Naturally, {even ``authentic''} users have incentives to behave strategically in order to promote certain views or products.

To study resilience, the {\it Byzantine} learning literature usually assumes that each federated learning worker may behave arbitrarily~\citep{BlanchardMGS17,YinCRB18,KarimireddyHJ21,YangL21}.
To understand the implication of this assumption, 
recall that at each iteration of a federated learning stochastic gradient descent, every worker is given the updated model, and asked to compute the gradient of the loss function with respect to (a batch of) its local data.
Byzantine learning assumes that a  worker may report \emph{any} gradient; without having to certify that the gradient was generated through data poisoning.
Whilst very general, and widely studied in the last few years, this gradient attack threat model has been argued to be unrealistic in practical federated learning~\citep{ShejwalkarHKR21},    especially when the workers are machines owned by trusted entities~\citep{kairouz2021advances}.

We prove in this paper a somewhat surprising equivalence between 
gradient attacks and data poisoning, {in a convex setting}. 
Essentially, we give the first practically compelling 
argument for the necessity to protect learning against gradient attacks. 
Our result enables us to carry over results on Byzantine gradient attacks to the data poisoning world. 
For instance, the impossibility result of~\citet{collaborative_learning}, 
combined with our equivalence result, 
implies that the more \emph{heterogeneous} the data,  
{the more vulnerable \emph{any} ``robust'' learning algorithm is.}
Also, we derive 
concrete data poisoning attacks from gradient ones.

\paragraph{Contributions.}
As a preamble of our main result,  we formalize local PAC* learning\footnote{We omit complexity considerations for the sake of generality. We define PAC* to be PAC without such considerations.}~\citep{Valiant84} for personalized learning, and prove that a simple and general solution to personalized federated linear regression and classification is indeed locally PAC* learning.
Our proof leverages a new concept called \emph{gradient-PAC* learning}.
We prove that gradient PAC* learning, which is verified by basic learning algorithms like linear and logistic regression,
is sufficient to guarantee local PAC* learning.
This is an important and nontrivial contribution of this paper.

Our main contribution is to then prove that local PAC* { convex} learning in personalized federated learning essentially implies an equivalence between data poisoning and gradient attacks.
More precisely, we show how any (converging) gradient attack can be turned into a data poisoning attack, with equal harm.
 As a corollary, we derive new impossibility theorems on what any robust personalized learning algorithm can guarantee, given heterogeneous genuine users and under data poisoning.
Given how easy it generally is to create fake accounts on web platforms
and to inject poisonous data through fabricated activities, our results arguably greatly increase the concerns about the vulnerabilities of learning from user-generated data,
even when ``Byzantine learning algorithms'' are used,
especially on controversial issues like hate speech moderation, 
where genuine users will inevitably provide conflicting reports on which words are abusive and ought to be removed.

Finally, we present a simple but very general \emph{strategic} gradient attack, called the \emph{counter-gradient attack},
which any participant to federated learning can deploy to bias the global model towards any target model that better suits their interest.
We prove the effectiveness of this attack under fairly general assumptions, which apply to many proposed personalized learning frameworks including~\citet{HanzelyHHR20,DinhTN20}.
We then show empirically how this attack can be turned into a devastating data poisoning attack, with remarkably few data\footnote{
{The code can be found at \url{https://github.com/LPD-EPFL/Attack_Equivalence}.}
}.
{Our experiment also shows the effectiveness of a simple protection, which prevents attackers from arbitrarily manipulating the trained algorithm. 
Namely, it suffices to replace the $\ell_2^2$ regularization with a (smooth) $\ell_2$ regularization.
Note that this solution is strongly related to the Byzantine resilience of the geometric median~\cite{geometric_median,AcharyaH0SDT22}.}

%% file: related_work.tex
 \paragraph{Related work.} 
 
 Collaborative PAC learning was introduced by~\citet{BlumHPQ17}, and then extensively studied~\citep{Chen0Z18,NguyenZ18}, sometimes assuming Byzantine collaborating users~\citep{Qiao18,JainO20a,KonstantinovFAL20}. It was however assumed that all honest users have the same labeling function.
 In other words, all honest users agree on how every query should be answered.
 This is a very unrealistic assumption in many critical applications, like content moderation or language processing.
 In fact, in such applications, removing outliers can be argued to amount to ignoring minorities' views, which would be highly unethical.
 The very definition of PAC learning must then be adapted, which is precisely what we do in this paper (by also adapting it to parameterized models).

A large literature has focused on \emph{data poisoning}, with either a focus on \emph{backdoor}~\citep{DaiCL19,ZhaoMZ0CJ20,SeveriMCO21,TruongJHAPJNT20,schwarzschild21a} or \emph{triggerless} attacks~\citep{BiggioNL12,Munoz-GonzalezB17,ShafahiHNSSDG18,ZhuHLTSG19,HuangGFTG20,Barreno06,aghakhani2021bullseye,geiping2021witches}.
 However, most of this research analyzed data poisoning without \emph{signed} data.
 A noteworthy exception is~\citet{MahloujifarMM19}, whose universal attack amplifies the probability of a (bad) property.
 Our work bridges the gap, for the first time, between that line of work and what has been called Byzantine resilience~\citep{MhamdiGR18,baruch19,XieKG19,MhamdiGR20}. 
 Results in this area typically establish the resilience against a minority of \emph{adversarial} users and many of them apply almost straightforwardly to personalized federated learning~\citep{El-MhamdiGGHR20,collaborative_learning}. 
  
 The attack we present in this paper considers a specific kind of Byzantine player, namely a \emph{strategic} one~\cite{SuyaMS0021}, whose aim is to bias the learned models towards a specific target model.
 The resilience of learning algorithms to such \emph{strategic} users has been studied in many special cases, including regression \citep{ChenPPS18, DEKEL2010759,PEROTE2004153,ben17}, classification \citep{MEIR2012123, chen2020, Meir11, hardt16}, statistical estimation \citep{yang15}, and clustering \citep{Perote04}. 
While some papers provide positive results in settings where each user can only provide a single data point~\cite{ChenPPS18,PEROTE2004153},~\citet{SuyaMS0021} show how to arbitrarily manipulate convex learning models through multiple data injections, when a single model is learned from all data at once.

\paragraph{Structure of the paper.} The rest of the paper is organized as follows.
Section~\ref{sec:model} presents a general model of personalized learning, formalizes local PAC* learning and describes a general federated gradient descent algorithm.
Section~\ref{sec:reduction} proves the equivalence between data poisoning and gradient attacks, under local PAC* learning.
Section~\ref{sec:pac} proves the local PAC* learning properties 
for federated linear regression and classification.
Section~\ref{sec:data_poisoning_attack} describes a simple and general data poisoning attack, and shows its effectiveness against $\ell_2^2$, both theoretically and empirically.
Section~\ref{sec:conclusion} concludes.
Proofs of our theoretical results and details about our experiments are given in the Appendix.

%% file: model.tex
\section{A General Personalized Learning Framework}
\label{sec:model}

We consider a set $[\NODE] = \set{1, \ldots, \NODE}$ of users.
Each user $\node \in [\NODE]$ has a local \emph{signed} dataset $\data{\node}$, and learns a local model $\paramsub{\node} \in \setR^d$.
Users may collaborate
to improve their models.
Personalized learning must then input a tuple of users' local datasets $\datafamily{} \triangleq (\data{1}, \ldots, \data{\NODE})$, and output a tuple of local models $\optimumfamily{} \triangleq (\optimumsub{1}, \ldots, \optimumsub{\NODE})$.
Like many others, we assume that the users perform federated learning to do so, by leveraging the computation of a common global model $\common \in \setR^d$.
Intuitively, the global model is an aggregate of all users' local models, which users can leverage to improve their local models.
This model typically allows users with too few data to obtain an effective local model, while it may be mostly discarded by users whose local datasets are large.

More formally, we consider a personalized learning framework which generalizes the models proposed by~\citet{DinhTN20} and \citet{HanzelyHHR20}.
Namely, we consider that the personalized learning algorithm outputs a global minimum $(\optcommon, \optimumfamily{})$ of a global loss given by
\begin{equation}
\label{eq:loss_global}
    \globalloss{} (\common, \paramfamily{}, \datafamily{})
    \triangleq \sum_{\node \in [\NODE]}
    \localloss{\node} (\paramsub{\node}, \data{\node})
    + \sum_{\node \in [\NODE]} \regularization (\common, \paramsub{\node}),
\end{equation}
where $\regularization$ is a regularization, typically with a minimum at $\paramsub{\node} = \common$.
For instance, \citet{HanzelyHHR20} and \citet{DinhTN20} define $\regularization (\common, \paramsub{\node}) \triangleq \regweightsub{} \norm{\common - \paramsub{\node}}{2}^2$, which we shall call the $\ell_2^2$ regularization.
But other regularizations may be considered, like the $\ell_2$ regularization $\regularization (\common, \paramsub{\node}) \triangleq \regweightsub{} \norm{\common - \paramsub{\node}}{2}$, or the smooth-$\ell_2$ regularization $\regularization (\common, \paramsub{\node}) \triangleq \regweightsub{} \sqrt{1 + \norm{\common - \paramsub{\node}}{2}^2}$.
Note that, for all such regularizations, 
the limit $\regweightsub{} \rightarrow \infty$
essentially yields the classical non-personalized federated learning framework.

\subsection{Local PAC* Learning}

{We consider that each honest user $\node$ has a \emph{preferred model} $\trueparamsub{\node}$, 
and that they provide \emph{honest datasets} $\data{\node}$ that are consistent with their preferred models.}
We then focus on personalized learning algorithms that provably recover a user $\node$'s \emph{preferred model} $\trueparamsub{\node}$, if the user provides a large enough \emph{honest dataset}.
Such honest datasets $\data{\node}$ could typically be obtained by repeatedly drawing random queries (or features), and by using the user's preferred model $\trueparamsub{\node}$ to provide (potentially noisy) answers (or labels).
We refer to Section~\ref{sec:pac} for examples. 
The model recovery condition is then formalized as follows.

\begin{definition}
\label{def:pac}
  A personalized learning algorithm is locally PAC* learning if,
  for any subset $\HONEST \subset [\NODE]$ of users, any preferred models $\trueparamfamily_{\HONEST}$, any $\varepsilon, \delta>0$, and any datasets $\datafamily{-\HONEST}$ from other users $\node \notin \HONEST$,
  there exists $\NODEINPUT{}$ such that,
  if all users $\honest \in \HONEST$ provide honest datasets $\data{\honest}$ with at least $\card{\data{\honest}} \geq \NODEINPUT{}$ data points, then, with probability at least $1-\delta$, 
  we have $\norm{\optimumsub{\honest}\left( \datafamily{} \right) - \trueparamsub{\honest}}{2} \leq \varepsilon$ for all users $\honest \in \HONEST$.
\end{definition}

Local PAC* learning is arguably a very desirable property.
Indeed, it guarantees that any honest active user will not be discouraged to participate in federated learning as they will eventually learn their preferred model by providing more and more data. 
Note that the required number of data points $\NODEINPUT{}$ also depends on the datasets provided by other users $\datafamily{-\HONEST}$. This implies that a locally PAC* learning algorithm is still vulnerable to poisoning attacks as the attacker's data set is not  a priori fixed.
In Section~\ref{sec:pac}, we will show how local PAC* learning can be achieved in practice, by considering specific local loss functions $\localloss{\node}$.

\subsection{Federated Gradient Descent}

While the computation of $\optcommon$ and $\optimumfamily{}$ could be done by a single machine, which first collects the datasets $\datafamily{}$ and then minimizes the global loss $\globalloss{}$ defined in~(\ref{eq:loss_global}), 
modern machine learning deployments often rather rely on \emph{federated} (stochastic) gradient descent (or variants), with a central trusted parameter server.
In this setting, each user $\node$ keeps their data $\data{\node}$ locally.
At each iteration $\iteration$, the parameter server sends the latest global model $\common^\iteration$ to the users.
Each user $\node$ is then expected to update its local model given the global model $\common^\iteration$, 
either by solving $\paramsub{\node}^\iteration \triangleq \argmin_{\paramsub{\node}} \localloss{\node} (\paramsub{\node}, \data{\node}) + \regularization(\common^\iteration, \paramsub{\node})$~\cite{DinhTN20} 
or by making a (stochastic) gradient step from the previous local model $\paramsub{\node}^{\iteration - 1}$ ~\cite{hanzely2021federated}. 
User $\node$ is then expected to report the gradient $\gradient{\node}{\iteration} = \nabla_\common \regularization(\common^\iteration, \paramsub{\node}^\iteration)$ of the global model to the parameter server.
The parameter server then updates the global model, using a gradient step, i.e. it computes $\common^{\iteration +1} \triangleq \common^\iteration - \learningrate{\iteration} \sum_{\node \in [\NODE]} \gradient{\node}{\iteration}$, where $\learningrate{\iteration}$ is the learning rate at iteration $\iteration$.
For simplicity, here, and since our goal is to show the vulnerability of personalized federated learning even in good conditions, we assume that the network is synchronous and that no node can crash.
Note also that our setting could be generalized to fully decentralized collaborative learning, as was done by~\citet{collaborative_learning}.

Users are only allowed to send plausible gradient vectors.
More precisely, we denote $$\GRADIENT(\common) \triangleq \overline{\set{\nabla_\common \regularization(\common, \param) \st \param \in \setR^d}},$$ the closure set of plausible (sub)gradients at $\common$.
If user $\node$'s gradient $\gradient{\node}{\iteration}$ is not in the set $\GRADIENT(\common^\iteration)$, the parameter server can easily detect the malicious behavior and $\gradient{\node}{\iteration}$ will be ignored at iteration $\iteration$.
In the case of an $\ell_2^2$ regularization, where $\regularization(\common, \param) = \regweightsub{}\norm{\common - \param}{2}^2$, we clearly have $\GRADIENT(\common) = \setR^d$ for all $\common \in \setR^d$.
It can be easily shown that, for $\ell_2$ and smooth-$\ell_2$ regularizations, $\GRADIENT(\common)$ is the closed ball $\ball(0,\regweightsub{})$.
Nevertheless, even then, a strategic user $\strategicnode \in [\NODE]$ can deviate from its expected behavior, to bias the global model in their favor.
We identify, in particular, three sorts of attacks.
\begin{description}
\item[Data poisoning:] Instead of collecting an honest dataset, $\strategicnode$ fabricates any \emph{strategically crafted} dataset  $\data{\strategicnode}$, and then performs all other operations as expected.
\item[Model attack:] At each iteration $\iteration$, $\strategicnode$ fixes $\paramsub{\strategicnode}^\iteration \triangleq \strategicvote$, where $\strategicvote$ is any \emph{strategically crafted} model. All other operations would then be executed as expected.
\item[Gradient attack:] At each iteration $\iteration$, $\strategicnode$ sends any (plausible) \emph{strategically crafted} gradient $\gradient{\strategicnode}{\iteration}$. {The gradient attack is said to \emph{converge}, if the sequence $\gradient{\strategicnode}{\iteration}$ converges.}
\end{description}

Gradient attacks are intuitively most harmful, as the strategic user can adapt their attack based on what they observe during training.
However, because of this, gradient attacks are more likely to be flagged as suspicious behaviors.
At the other end, data poisoning may seem much less harmful. 
But it is also harder to detect, as the strategic user can report their entire dataset, and prove that they rigorously performed the expected computations.
In fact, data poisoning can be executed, even if users directly provide the data to a (trusted) central authority, which then executes (stochastic) gradient descent.
This is typically what is done to construct recommendation algorithms, where users' data are their online activities (what they view, like and share).
Crucially, especially in applications with no clear ground truth, such as content moderation or language processing, the strategic user can always argue that their dataset is ``honest''; not strategically crafted.
Ignoring the strategic user's data on the basis that it is an ``outlier'' may then be regarded as \emph{unethical}, as it amounts to rejecting minorities' viewpoints.

%% file: reduction.tex
\section{The Equivalence Between Data Poisoning and Gradient Attacks}
\label{sec:reduction}

We now present our main result,
considering ``model-targeted attacks'', i.e., the attacker aims to bias the global model towards a target model $\trueparamsub{\strategicnode}$. 
This attack was also previously studied by ~\citet{SuyaMS0021}.

\begin{theorem}[Equivalence between gradient attacks and data poisoning]
\label{th:equivalence}
Assume local PAC* learning, and $\ell_2^2$, $\ell_2$ or smooth-$\ell_2$ regularization.
Suppose that each loss $\localloss{\node}$ is convex 
and that the learning rate $\learningrate{\iteration}$ is constant.
Consider any datasets $\datafamily{-\strategicnode}$ provided by users $\node \neq \strategicnode$.
Then, for any target model $\trueparamsub{\strategicnode} \in \setR^d$, there exists a converging gradient attack of strategic user $\strategicnode$ such that $\common^\iteration \rightarrow \trueparamsub{\strategicnode}$, if and only if, for any $\varepsilon > 0$, there exists a dataset $\data{\strategicnode}$ such that $\norm{\optcommon(\datafamily{}) - \trueparamsub{\strategicnode}}{2} \leq \varepsilon$.
\end{theorem}

For the sake of exposition, our results are stated for $\ell_2^2$ or smooth-$\ell_2$ regularization only.
But the proof, in Appendix~\ref{app:reduction}, holds for all continuous regularizations $\regularization$ with $\regularization(\common, \param) \rightarrow \infty$ as $\norm{\common - \param}{2} \rightarrow \infty$.
We now sketch our proof, which goes through
model attacks.

\subsection{Data Poisoning and Model Attacks}
To study the model attack, we define the modified loss with directly strategic user $\strategicnode$'s reported model $\strategicvote$ as 
\begin{equation}
  \ModifiedLoss{\strategicnode} (\common, \paramfamily_{-\strategicnode}, \strategicvote, \datafamily{-\strategicnode})
  \triangleq
  \globalloss{} (\common, (\strategicvote, \paramfamily_{-\strategicnode}), (\emptyset, \datafamily{-\strategicnode}))
\end{equation}
where $\paramfamily_{-\strategicnode}$ and $\datafamily{-\strategicnode}$ are variables and datasets for users $\node \neq \strategicnode$.
Denote $\optcommon{} (\strategicvote, \datafamily{-\strategicnode})$ and $\optimumfamily_{-\strategicnode}(\strategicvote, \datafamily{-\strategicnode})$ a minimum of the modified loss function
and $\optimumsub{\strategicnode}(\strategicvote, \datafamily{-\strategicnode}) \triangleq \strategicvote$.

\begin{lemma}[Reduction from model attack to data poisoning]
\label{th:reduction_model_to_data}
  Consider any data $\datafamily{}$ and user $\strategicnode \in [\NODE]$.
  Assume the global loss has a global minimum $(\optcommon, \optimumfamily)$.
  Then $(\optcommon, \optimumfamily_{-\strategicnode})$ is also a global minimum of the modified loss with datasets $\datafamily{-\strategicnode}$ and strategic reporting $\strategicvote \triangleq \optimumsub{\strategicnode}(\datafamily{})$.
\end{lemma}

Now, intuitively, 
by virtue of local PAC* learning, strategic user $\strategicnode$ can essentially guarantee that the personalized learning framework will be learning $\optimumsub{\strategicnode} \approx \strategicvote$.
In the sequel, we show that this is the case.

\begin{lemma}[Reduction from data poisoning to model attack]
\label{th:reduction_data_to_model}
  Assume $\ell_2^2$, $\ell_2$ or smooth-$\ell_2$ regularization,
  and assume local PAC* learning.
  Consider any datasets $\data{-\strategicnode}$ and any attack model $\strategicvote$ such that the modified loss $\ModifiedLoss{\strategicnode}$ has a unique minimum $\optcommon(\strategicvote, \datafamily{-\strategicnode}), \optimumfamily_{-\strategicnode}(\strategicvote, \datafamily{-\strategicnode})$.
  Then, for any $\varepsilon > 0$,
  there exists a dataset $\data{\strategicnode}$ such that
  we have
  \begin{align}
    &\norm{\optcommon{}(\datafamily{}) - \optcommon{}(\strategicvote, \datafamily{-\strategicnode})}{2} \leq \varepsilon
    ~~\text{and}~~ \nonumber \\
    \forall \node \neq \strategicnode \mathsep &\norm{\optimumsub{\node}(\datafamily{}) - \optimumsub{\node}(\strategicvote, \datafamily{-\strategicnode})}{2} \leq \varepsilon.
  \end{align}
\end{lemma}

\begin{proof}[Sketch of proof]
  Given local PAC*, for a large dataset $\data{\strategicnode}$ constructed from $\strategicvote$, $\strategicnode$ can guarantee $\optimumsub{\strategicnode} (\datafamily{}) \approx \strategicvote$.
  By carefully bounding the effect of the approximation on the loss using the Heine-Cantor theorem, we show that this implies $\optcommon{} (\datafamily{}) \approx \optcommon{} (\strategicvote, \datafamily{-\strategicnode})$ and
  $\optimumsub{\node} (\datafamily{}) \approx \optimumsub{\node} (\strategicvote, \datafamily{-\strategicnode})$ for all $\node \neq \strategicnode$ too.
  The precise analysis is nontrivial.
\end{proof}

\subsection{Model Attacks and Gradient Attacks}

We now prove that any successful converging model-targeted gradient attack can be transformed into an equivalently successful model attack.

\begin{lemma}[Reduction from model attack to gradient attack]
\label{th:gradient_attack_convergence}
  Assume that $\localloss{\node}$ is convex 
  for all users $\node \in [\NODE]$, and that we use $\ell_2^2$, $\ell_2$ or smooth-$\ell_2$ regularization. Consider a converging gradient attack $\gradient{\strategicnode}{\iteration}$ with limit $\gradient{\strategicnode}{\infty}$ that makes the global model $\common^t$ converge to $\common^\infty$ with a constant learning rate $\learningrate{}$.
  Then for any $\varepsilon > 0$, there is $\strategicvote \in \setR^d$ such that $\norm{\common^\infty - \optcommon{}(\strategicvote, \datafamily{-\strategicnode})}{2} \leq \varepsilon$.
\end{lemma}
\begin{proof}[Sketch of proof]
The proof is based on the observation that since $\GRADIENT$ is closed and $\gradient{\strategicnode}{\infty} \in \GRADIENT$, we can construct $\strategicvote$ which approximately yields the gradient $\gradient{\strategicnode}{\infty}$.
\end{proof}

Since any model attack can clearly be achieved by the corresponding honest gradient attack {for a sufficiently small and constant learning rate}, model attacks and gradient attacks are thus equivalent.
In light of our previous results, this implies that gradient attacks are essentially equivalent to data poisoning (Theorem~\ref{th:equivalence}).

%% file: vulnerability.tex
\subsection{Convergence of the Global Model}

Note that Theorem \ref{th:equivalence} (and Lemma \ref{th:gradient_attack_convergence}) assumes that the global model converges.
Here, we prove that this assumption is automatically satisfied for converging gradients,
at least when local models $\paramsub{\node}^\iteration$ are fully optimized given $\common^\iteration$, at each iteration $\iteration$, in the manner of~\cite{DinhTN20}, and under smoothness assumptions.

\begin{proposition}
\label{prop:convergence}
 Assume that $\localloss{\node}$ is convex and $L$-smooth for all users $\node \in [\NODE]$, and that we use $\ell_2^2$ or smooth-$\ell_2$ regularization. If $\gradient{\strategicnode}{\iteration}$ converges and if $\learningrate{\iteration} = \learningrate{}$ is a constant small enough, then $\common^\iteration$ will converge too.
\end{proposition}

\begin{proof}[Sketch of proof]
Denote $\gradient{\strategicnode}{\infty}$ the limit of $\gradient{\strategicnode}{\iteration}$.
    Gradient descent then behaves as though it was minimizing the loss plus $\common^T \gradient{\strategicnode}{\infty}$ (and ignoring $\regularization(\common, \paramsub{\strategicnode})$).
Essentially, classical gradient descent theory then guarantees $\common^\iteration \rightarrow \common^\infty$, though the precise proof is nontrivial (see Appendix \ref{sec:Global_convergence}).
\end{proof}

\subsection{Impossibility Corollaries}

Given our equivalence, impossibility theorems on (heterogeneous) federated learning under (converging) gradient attacks imply impossibility results under data poisoning.
For instance, \citet{collaborative_learning} and \citet{He2020} proved theorems saying that the more heterogeneous the learning, the more vulnerable it is in a Byzantine context,
{even when ``Byzantine-resilient'' algorithms are used~\cite{BlanchardMGS17}}.
In fact, and interestingly, \citet{collaborative_learning} and \citet{He2020}  actually leverage a model attack.
Before translating the corresponding result, some work is needed to formalize what Byzantine resilience may mean in our setting.

\begin{definition}
\label{def:Byzantine-learning}
A personalized learning algorithm $\alg$
achieves $(\BYZANTINE, \NODE, \averagingconstant)$-Byzantine learning if,
for any subset $\HONEST \subset [\NODE]$ of honest users with $\card{\HONEST} = \NODE - \BYZANTINE$,
any honest vectors $\trueparamfamily_{\HONEST} \in (\setR^d)^{\card{\HONEST}}$,
given any $\varepsilon, \delta > 0$, there exists $\NODEINPUT{}$ such that,
when each honest user $\honest \in \HONEST$ provides honest datasets $\data{\honest}^\dagger$ by answering $\NODEINPUT{}$ queries with model $\trueparamsub{\honest}$,
then, with probability at least $1-\delta$,
for any poisoning datasets $\datafamily{\BYZANTINE}^\spadesuit$ provided by Byzantine users $\byzantine \notin \HONEST$,
denoting $\common^\alg \triangleq \common^\alg(\datafamily{\HONEST}^\dagger, \datafamily{\BYZANTINE}^\spadesuit)$
and $\paramfamily^\alg \triangleq \paramfamily^\alg(\datafamily{\HONEST}^\dagger, \datafamily{\BYZANTINE}^\spadesuit)$,
we have the guarantee
\begin{equation}
\norm{\common^\alg - \overline{\trueparamsub{\HONEST}}}{2}^2
    \leq C^2 \max_{\honest, \honest' \in \HONEST} \norm{\trueparamsub{\honest} - \trueparamsub{\honest'}}{2}^2 + \varepsilon,
\end{equation}
where $\overline{\trueparamsub{\HONEST}}$ is the average of honest users' preferred models.
\end{definition}

Note that $\overline{\trueparamsub{\HONEST}}$ is what we would have learned, under local PAC* and $\ell_2^2$ regularization,
in the absence of Byzantine users $\byzantine \in [\NODE] - \HONEST$,
in the limit where all honest users $\honest \in \HONEST$ provide a very large amount of data.
Meanwhile, $\max_{\honest, \honest' \in \HONEST} \norm{\trueparamsub{\honest} - \trueparamsub{\honest'}}{2}^2$ is a reasonable measure of the heterogeneity among honest users.
Thus, our definition
captures well the robustness of the algorithm $\alg$, for heterogeneous learning under data poisoning.
Interestingly, our equivalence theorem allows to translate the model-attack-based impossibility theorems of \citet{collaborative_learning} into an impossibility theorem on data poisoning resilience.

\begin{corollary}
\label{cor:impossibility_resilience}
No algorithm achieves $(\BYZANTINE, \NODE, \averagingconstant)$-Byzantine learning with $\BYZANTINE \geq \NODE / 2$.
\end{corollary}

\begin{corollary}
\label{cor:impossibility}
No algorithm achieves $(\BYZANTINE, \NODE, \averagingconstant)$-Byzantine learning with $\averagingconstant < \BYZANTINE / (\NODE - \BYZANTINE)$.
\end{corollary}

The proofs are given in Appendix~\ref{sec:impossibility_corollary}.

%% file: PAC.tex
\section{Examples of Locally PAC* Learning Systems}
\label{sec:pac}

To the best of our knowledge, though similar to collaborative PAC learning~\citep{BlumHPQ17}, local PAC* learnability is a new concept in the context of personalized federated learning.
It is thus important to show that it is not unrealistic.
To achieve this, in this section, we provide \emph{sufficient} conditions for a personalized learning model to be locally PAC* learnable.
First, we construct local losses $\localloss{\node}$ as sums of losses per input, i.e.
\begin{equation}
  \localloss{\node} (\paramsub{\node}, \data{\node})
  = \reglocalweight{} \norm{\paramsub{\node}}{2}^2 +
  \sum_{\datapoint \in \data{\node}} \lossperinput (\paramsub{\node}, \datapoint),
\end{equation}
for some ``loss per input'' function $\lossperinput$ and a weight $\reglocalweight{} > 0$.
Appendix~\ref{app:sum_average} gives theoretical and empirical arguments are provided for using such a sum (as opposed to an expectation).
Remarkably, for linear or logistic regression, given such a loss, local PAC* learning can then be guaranteed.

\begin{theorem}[Personalized least square linear regression is locally PAC* learning]
\label{th:linear_regression}
  Consider $\ell_2^2$, $\ell_2$ or smooth-$\ell_2$ regularization.
  Assume that, to generate a data $\datapoint_{\nodeinput}$, 
  a user with preferred parameter $\trueparam \in \setR^d$ first independently draws a random vector query
  $\query{\nodeinput} \in \setR^d$ from a bounded query distribution $\querydistribution{}$,
  with positive definite matrix\footnote{In fact, in Appendix~\ref{app:least_square}, we prove a more general result with any sub-Gaussian query distribution $\querydistribution{}$, 
  with parameter $\sigma_{\query{}}$.} $\Sigma = \expect{\query{\nodeinput}\query{\nodeinput}^T}$.
  Assume that the user labels $\query{\nodeinput}$ with 
  answer
  $\answer{\nodeinput} = \query{\nodeinput}^T \trueparam + \noise{\nodeinput}$, 
  where $\noise{\nodeinput}$ is a zero-mean sub-Gaussian random noise with parameter $\sigma_{\noise{}}$, independent from $\query{\nodeinput}$ and other data points.
  Finally, assume that $\lossperinput(\param, (\query{\nodeinput}, \answer{\nodeinput})) = \frac{1}{2}(\param^T \query{\nodeinput} - \answer{\nodeinput})^2$. 
  Then the personalized learning algorithm is locally PAC* learning.
\end{theorem}

\begin{theorem}[Personalized logistic regression is locally PAC*-learning]
\label{th:logistic_regression}
  Consider $\ell_2^2$, $\ell_2$ or smooth-$\ell_2$ regularization.
  Assume that, to generate a data $\datapoint_{\nodeinput}$, 
  a user with preferred parameter $\trueparam \in \setR^d$ first independently draws a random vector query
  $\query{\nodeinput} \in \setR^d$ from a query distribution $\querydistribution{}$, 
  whose support $\support(\querydistribution{})$ is bounded and spans the full vector space $\setR^d$.
  Assume that the user then labels $\query{\nodeinput}$ with answer
  $\answer{\nodeinput} = 1$ with probability $\sigmoid (\query{\nodeinput}^T \trueparam)$, 
  and labels it $\answer{\nodeinput} = -1$ otherwise, 
  where $\sigmoid (z) \triangleq (1+e^{-z})^{-1}$.
  Finally, assume that $\lossperinput(\param, (\query{\nodeinput}, \answer{\nodeinput})) = - \ln(\sigmoid (\answer{\nodeinput} \param^T \query{\nodeinput}))$. 
  Then the personalized learning algorithm is locally PAC* learning.
\end{theorem}

\subsection{Proof Sketch}

The full proofs of theorems~\ref{th:linear_regression} and~\ref{th:logistic_regression} are given in Appendix~\ref{sec:models_proof}.
Here, we provide proof outlines.
In both cases, we leverage the following stronger form of PAC* learning.

\begin{definition}[Gradient-PAC*]
Let $\event(\data{}, \trueparam, \NODEINPUT{}, A, B, \alpha)$ the event defined by
\begin{align}
\label{eq:gradient-pac}
    &\forall \param \in \setR^d \mathsep
    \left( \param - \trueparam \right)^T \nabla \independentloss{} \left(\param, \data{} \right) \geq \nonumber \\
    &A \NODEINPUT{} \min \left\lbrace \norm{\param-\trueparam}{2}, \norm{\param-\trueparam}{2}^2 \right\rbrace - B \NODEINPUT{}^\alpha \norm{\param - \trueparam}{2}. \nonumber 
\end{align}
The loss $\localloss{}$ is gradient-PAC* if,
for any $\parambound > 0$,
there exist constants $A_\parambound, B_\parambound >0$ and $\alpha_\parambound <1$, such that
for any $\trueparam \in \setR^d$ with $\norm{\trueparam}{2} \leq \parambound$,
assuming that the dataset $\data{}$ is obtained by honestly collecting and labeling $\NODEINPUT{}$ data points according to the preferred model $\trueparam$,
the probability of the event $\event(\data{}, \trueparam, \NODEINPUT{}, A_\parambound, B_\parambound, \alpha_\parambound)$
goes to $1$
as $\NODEINPUT{} \rightarrow \infty$.
\label{ass:unbiased}
\end{definition}

Intuitively, this definition asserts that, as we collect more data from a user, then, with high probability, the gradient of the loss at any point $\param$ too far from $\trueparam$ will point away from $\trueparam$.
In particular, gradient descent is then essentially guaranteed to draw $\param$ closer to $\trueparam$.
The right-hand side of the equation defining $\event(\data{}, \trueparam, \NODEINPUT{}, A, B, \alpha)$ is subtly chosen to be strong enough to guarantee local PAC*, and weak enough to be verified by linear and logistic regression.

\begin{lemma}
\label{lemma:gradient-pac}
Logistic and linear regression, defined in theorems~\ref{th:linear_regression} and~\ref{th:logistic_regression}, are gradient PAC* learning.
\end{lemma}

\begin{proof}[Sketch of proof]
For linear regression, remarkably, the discrepancy between the empirical and the expected loss functions depends only on a few key random variables, such as $\min{} \textsc{Sp} \left( \frac{1}{\NODEINPUT{}} \sum \query{\nodeinput} \query{\nodeinput}^T \right)$ and $\sum \noise{\nodeinput} \query{\nodeinput}$, which can be controlled by appropriate concentration bounds.
Meanwhile, for logistic regression, for $\absv{b} \leq \parambound$, we observe that $(a-b)(\sigmoid (a)-\sigmoid(b)) \geq c_\parambound \min( \absv{a-b}, \absv{a-b}^2)$.
Essentially, this proves that gradient-PAC* would hold if the empirical loss was replaced by the expected loss.
The actual proofs, however, are nontrivial, especially in the case of logistic regression, 
which leverages topological considerations to derive a critical uniform concentration bound.
\end{proof}

Now, under very mild assumptions on the regularization $\regularization$ (not even convexity!), which are verified by the $\ell_2^2$, $\ell_2$ and smooth-$\ell_2$ regularizations, we prove that the gradient-PAC* learnability through $\lossperinput$ suffices to guarantee that personalized learning will be locally PAC* learning.

\begin{lemma}
\label{th:pac}
  Consider $\ell_2^2$, $\ell_2$ or smooth-$\ell_2$ regularization.
  If $\lossperinput$ is gradient-PAC* and nonnegative, then personalized learning is locally PAC*-learning.
\end{lemma}

\begin{proof}[Sketch of proof]
  Given other users' datasets, $\regularization$ yields a fixed bias.
  But as the user provides more data, by gradient-PAC*, the local loss dominates, thereby guaranteeing local PAC*-learning.
  Appendix~\ref{sec:pac_proof} provides a full proof. 
\end{proof}

Combining the two lemmas clearly yields theorems~\ref{th:linear_regression} and~\ref{th:logistic_regression} as special cases.
Note that our result actually applies to a more general set of regularizations and losses.

\subsection{The Case of Deep Neural Networks}

Deep neural networks generally do \emph{not} verify gradient PAC*.
After all, because of symmetries like neuron swapping, different values of the parameters might compute the same neural network function.
Thus the ``preferred model'' $\trueparam$ is arguably ill-defined for neural networks\footnote{Evidently, our definition could be modified to focus on the computed function, rather than to the model parameters.}.
Nevertheless, we may consider a strategic user who only aims to bias the last layer.
In particular, assuming that all layers but the last one of a neural network are pretrained and fixed, 
{ thereby defining a ``shared representation''~\cite{CollinsHMS21},
and assuming the last layer performs a linear regression or classification,
then our theory essentially applies to the fine-tuning of the parameters of the last layer (sometimes known as the ``head'').}

{Note that for our data poisoning reconstruction (see Section~\ref{sec:data_poisoning_attack}) to be applicable, 
the attacker would need to have the capability to generate a data point whose vector representation matches any given predefined latent vector.
In certain applications, this can be achieved through generative networks~\cite{GoodfellowPMXWO20}.
If so, then our data poisoning attacks would apply as well to deep neural network head tuning.}

%% file: counter_gradient_attack.tex
\section{A Practical Data Poisoning Attack}
\label{sec:data_poisoning_attack}

We now construct a practical data poisoning attack, by introducing a new gradient attack, and by then leveraging our equivalence to turn it into a data poisoning attack.

\subsection{The Counter-Gradient Attack}
\label{sec:counter_gradient_attack}

We define a simple, general and practical gradient attack, which we call the counter-gradient attack (\CounterGradientAttack{}). 
Intuitively, this attack estimates the sum $\truegradient{-\strategicnode}{\iteration}$ of the gradients of other users based on its value at the previous iteration, which can be inferred from the way the global model $\common^{\iteration -1}$ was updated into $\common^\iteration$.
More precisely, apart from initialization $\estimatedgradient{-\strategicnode}{1} \triangleq 0$, \CounterGradientAttack{} makes the estimation
\begin{equation}
\label{eq:estimated_gradient}
    \estimatedgradient{-\strategicnode}{\iteration} \triangleq \frac{\common^{\iteration-1} - \common^{\iteration}}{\learningrate{\iteration - 1}} - \gradient{\strategicnode}{\iteration-1}
    = \truegradient{-\strategicnode}{\iteration-1}.
\end{equation}
Strategic user $\strategicnode$ then reports the plausible gradient that moves the global model closest to the user's target model $\trueparamsub{\strategicnode}$, assuming others report $\estimatedgradient{-\strategicnode}{\iteration}$.
In other words, at every iteration, \CounterGradientAttack{} 
reports
\begin{equation}
\label{eq:counter-gradient-attack}
    \gradient{\strategicnode}{\iteration}
    \in \argmin_{\gradient{}{} \in \GRADIENT(\common^\iteration)}
    \norm{\common^\iteration - \learningrate{\iteration}
    (\estimatedgradient{-\strategicnode}{\iteration} + \gradient{}{})
    - \trueparamsub{\strategicnode}}{2}.
\end{equation}
Note that this attack only requires user $\strategicnode$ to know the learning rates $\learningrate{\iteration-1}$ and $\learningrate{\iteration}$, the global models $\common^{\iteration-1}$ and $\common^{\iteration}$, and their target model $\trueparamsub{\strategicnode}$.

\paragraph{Computation of \CounterGradientAttack{}.} 
Define $\nonnormalizedgradient{\strategicnode}{\iteration} \triangleq \gradient{\strategicnode}{\iteration-1}
  + \frac{\common^\iteration - \trueparamsub{\strategicnode}}{\learningrate{\iteration}}
  - \frac{\common^{\iteration-1} - \common^{\iteration}}{\learningrate{\iteration - 1}}$.
For convex sets $\GRADIENT(\common^\iteration)$, it is straightforward to see that \CounterGradientAttack{} boils down to computing the orthogonal projection of $\nonnormalizedgradient{\strategicnode}{\iteration}$ on $\GRADIENT(\common^\iteration)$.
This yields very simple computations for $\ell_2^2$, $\ell_2$ and smooth-$\ell_2$ regularizations.

\begin{proposition}
\label{prop:cga_smooth_ell2}
  For $\ell_2^2$ regularization, \CounterGradientAttack{} reports $\gradient{\strategicnode}{\iteration} =
    \nonnormalizedgradient{\strategicnode}{\iteration}$.
  For $\ell_2$ or smooth-$\ell_2$ regularization,
  \CounterGradientAttack{} reports $\gradient{\strategicnode}{\iteration} = \nonnormalizedgradient{\strategicnode}{\iteration} \min \set{1, \regweightsub{} / \norm{\nonnormalizedgradient{\strategicnode}{\iteration}}{2}}$.
\end{proposition}

\begin{proof}
  Equation (\ref{eq:counter-gradient-attack}) boils down to minimizing the distance between $\frac{\common^\iteration - \trueparamsub{\strategicnode}}{\learningrate{\iteration}} - \estimatedgradient{-\strategicnode}{\iteration}$ and $\GRADIENT(\common)$,
  which is the ball $\ball(0, \regweightsub{})$.
  This minimum is the orthogonal projection.
\end{proof}

\paragraph{Theoretical analysis.}

We   prove that \CounterGradientAttack{} is perfectly successful against $\ell_2^2$ regularization.
To do so, we suppose that, at each iteration $\iteration$ and for each user $\node \neq \strategicnode$, the local models $\paramsub{\node}$ are fully optimized with respect to $\common^\iteration$, and the honest gradients of $\truegradient{\node}{\iteration}$ are used to update $\common$.

\begin{theorem}
\label{th:counter_gradient_manipulates_arbitrarily}
    Consider $\ell_2^2$ regularization.
    Assume that $\lossperinput$ is convex and $L_\lossperinput$-smooth, and that $\learningrate{\iteration} = \learningrate{}$ is small enough.
    Then \CounterGradientAttack{} is converging and optimal, as $\common^\iteration \rightarrow \trueparamsub{\strategicnode}$.
\end{theorem}

\begin{proof}[Sketch of proof]
The main challenge is to guarantee that the other users' gradients $\truegradient{\node}{\iteration}$ for $\node \neq \strategicnode$ remain sufficiently stable over time to guarantee convergence, which can be done by leveraging $L$-smoothness.
The full proof, with the necessary upper-bound on $\learningrate{}$, is given in Appendix~\ref{app:cga_vs_l22}.
\end{proof}

The analysis of the convergence against smooth-$\ell_2$ is unfortunately significantly more challenging.
Here, we simply make a remark about \CounterGradientAttack{} at convergence.

\begin{proposition}
    If \CounterGradientAttack{} against smooth-$\ell_2$ regularization converges for $\learningrate{\iteration} = \learningrate{}$, then it either achieves perfect manipulation, or it is eventually partially honest, in the sense that the gradient by \CounterGradientAttack{} correctly points towards $\trueparamsub{\strategicnode}$.
\end{proposition}

\begin{proof}
   Denote $P$ the projection onto the closed ball $\ball (0, \regweightsub{})$.
   If \CounterGradientAttack{} converges, then, by Proposition~\ref{prop:cga_smooth_ell2}, $P \left(\gradient{\strategicnode}{\infty} + \frac{\common^\infty - \trueparamsub{\strategicnode}}{\learningrate{}} \right) = \gradient{\strategicnode}{\infty}$.
   Thus $\common^\infty - \trueparamsub{\strategicnode}$ and $\gradient{\strategicnode}{\infty}$ must be colinear.
   If perfect manipulation is not achieved (i.e. $\common^\infty \neq \trueparamsub{\strategicnode}$), 
   then we must have $\gradient{\strategicnode}{\infty} = \regweightsub{} \frac{\common^\infty - \trueparamsub{\strategicnode}}{\norm{\common^\infty - \trueparamsub{\strategicnode}}{2}}$.
\end{proof}

{ 
It is interesting that, against smooth-$\ell_2$, \CounterGradientAttack{} actually favors partial honesty.
Overall, this condition is critical for the safety of learning algorithms, as they are usually trained to generalize their training data.
However, it should be stressed that this is evidence that \CounterGradientAttack{} is suboptimal,
as \cite{geometric_median} instead showed that the geometric median rather (slightly) incentivizes untruthful strategic behaviors.
The problem of designing general \emph{strategyproof} learning algorithms is arguably still mostly open, despite recent progress~\cite{MEIR2012123,ChenPPS18,FarhadkhaniGH21}.
}

\paragraph{Empirical evaluation of \CounterGradientAttack{}.}
We deployed \CounterGradientAttack{} to bias the federated learning of MNIST.
We consider a strategic user whose target model is one that labels $0$'s as $1$'s, $1$'s as $2$'s, and so on, until $9$'s that are labeled as $0$'s.
In particular, this target model has a nil accuracy.
Figure~\ref{fig:attack_acc} shows that such a user effectively hacks the $\ell_2^2$ regularization against 10 honest users who each have 6,000 data points of MNIST, in the case where local models only undergo a single gradient step at each iteration, but fails to hack the $\ell_2$ regularization. This suggests the effectiveness of simple defense strategies like the geometric median~\cite{geometric_median,AcharyaH0SDT22}. 
See Appendix~\ref{app:CGA_MNIST} for more details.
We also ran a similar successful attack on the last layer of a deep neural network trained on cifar-10, which is detailed in Appendix~\ref{app:cifar}.

\begin{figure}[h]
    \centering
     \vspace{-0.8mm}
    \includegraphics[width=.8\linewidth]{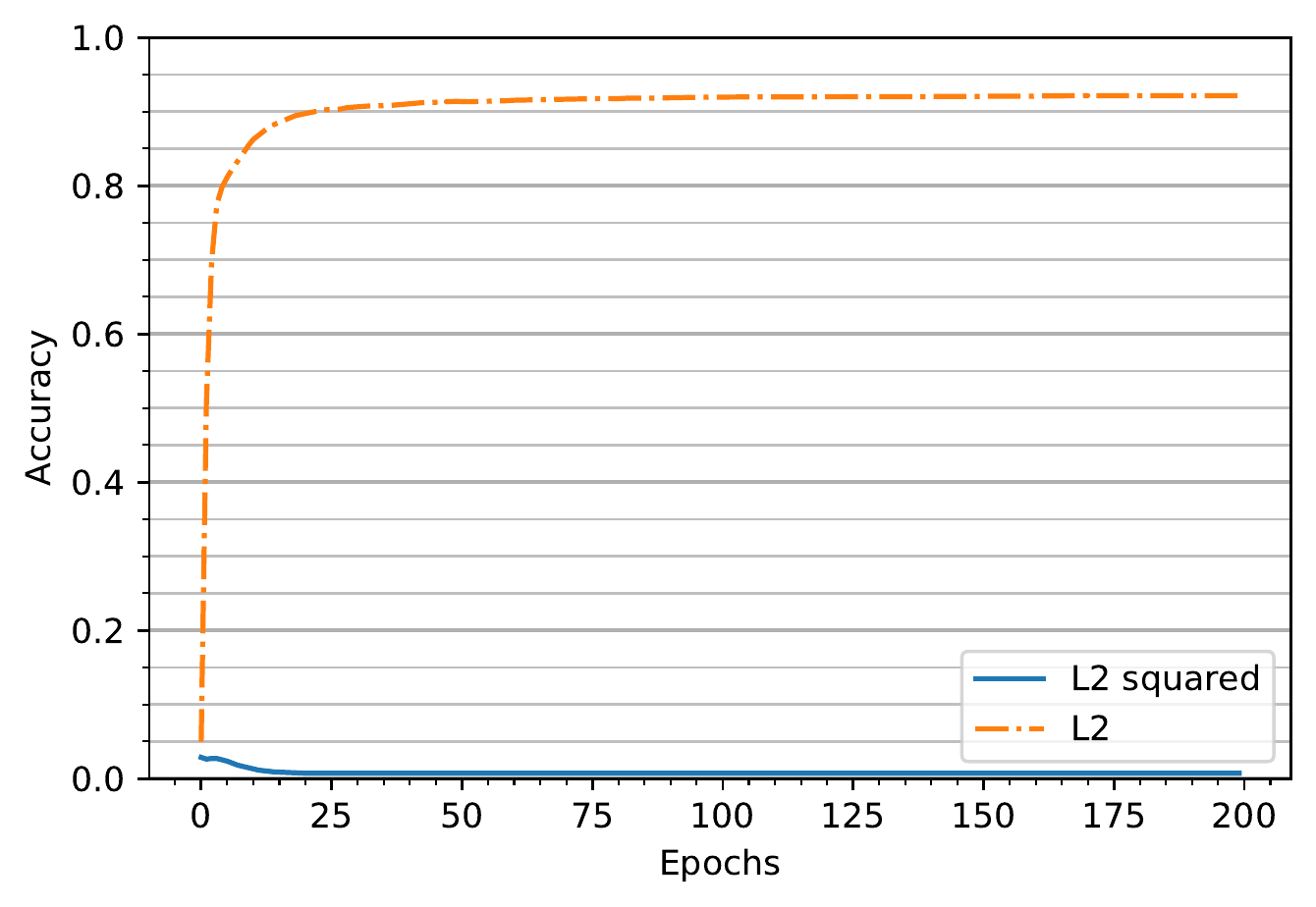}
    \caption{Accuracy of the global model under attack by \CounterGradientAttack{}.}
    \vspace{-2mm}
     \label{fig:attack_acc}
\end{figure}

\begin{figure*}[!ht]
    \centering
        {(a)~\includegraphics[width=0.34\linewidth]{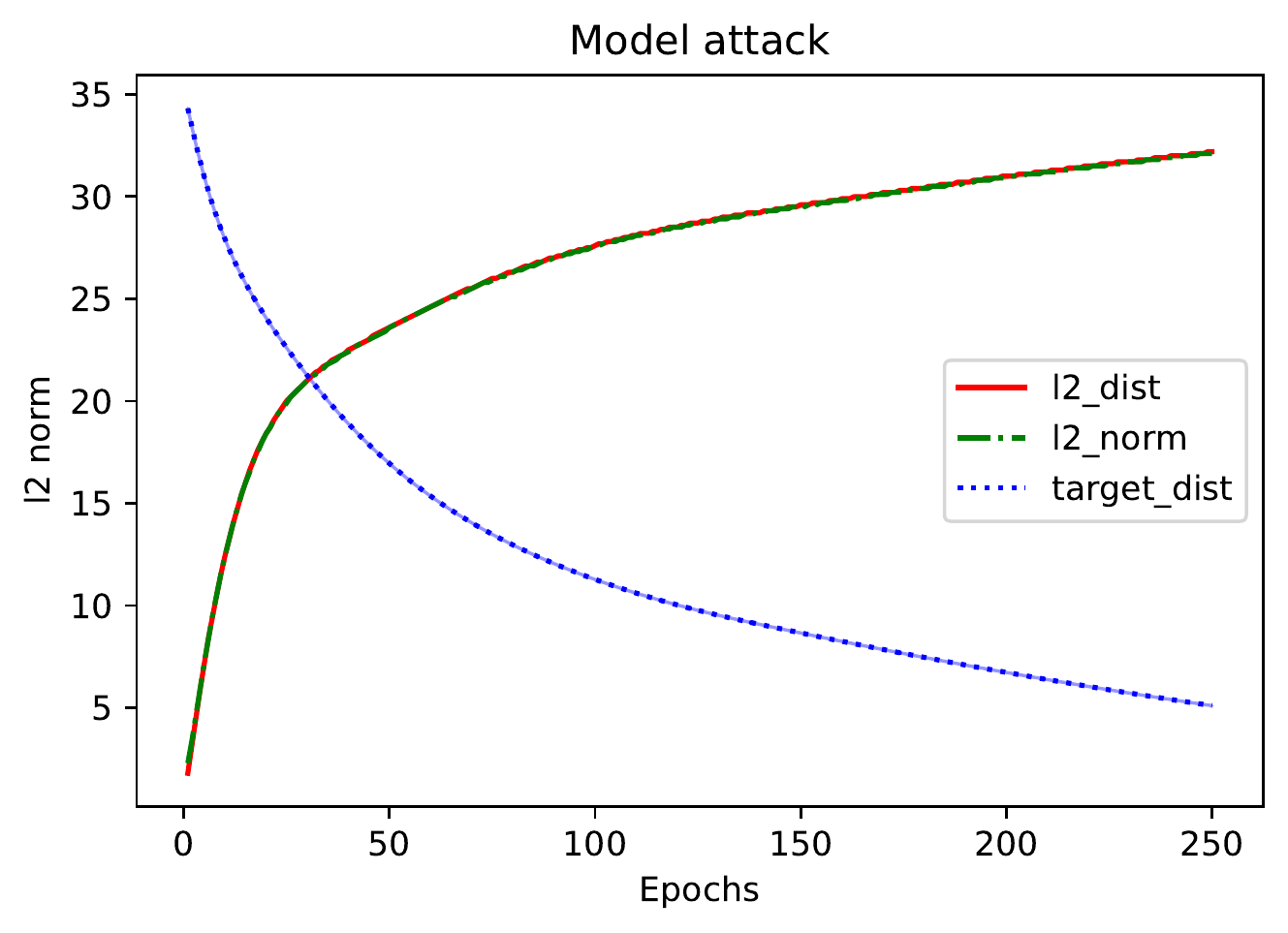}
        \phantomsubcaption\label{fig:model_attack_distance}} \quad \quad
        {(b)~\includegraphics[width=0.34\linewidth]{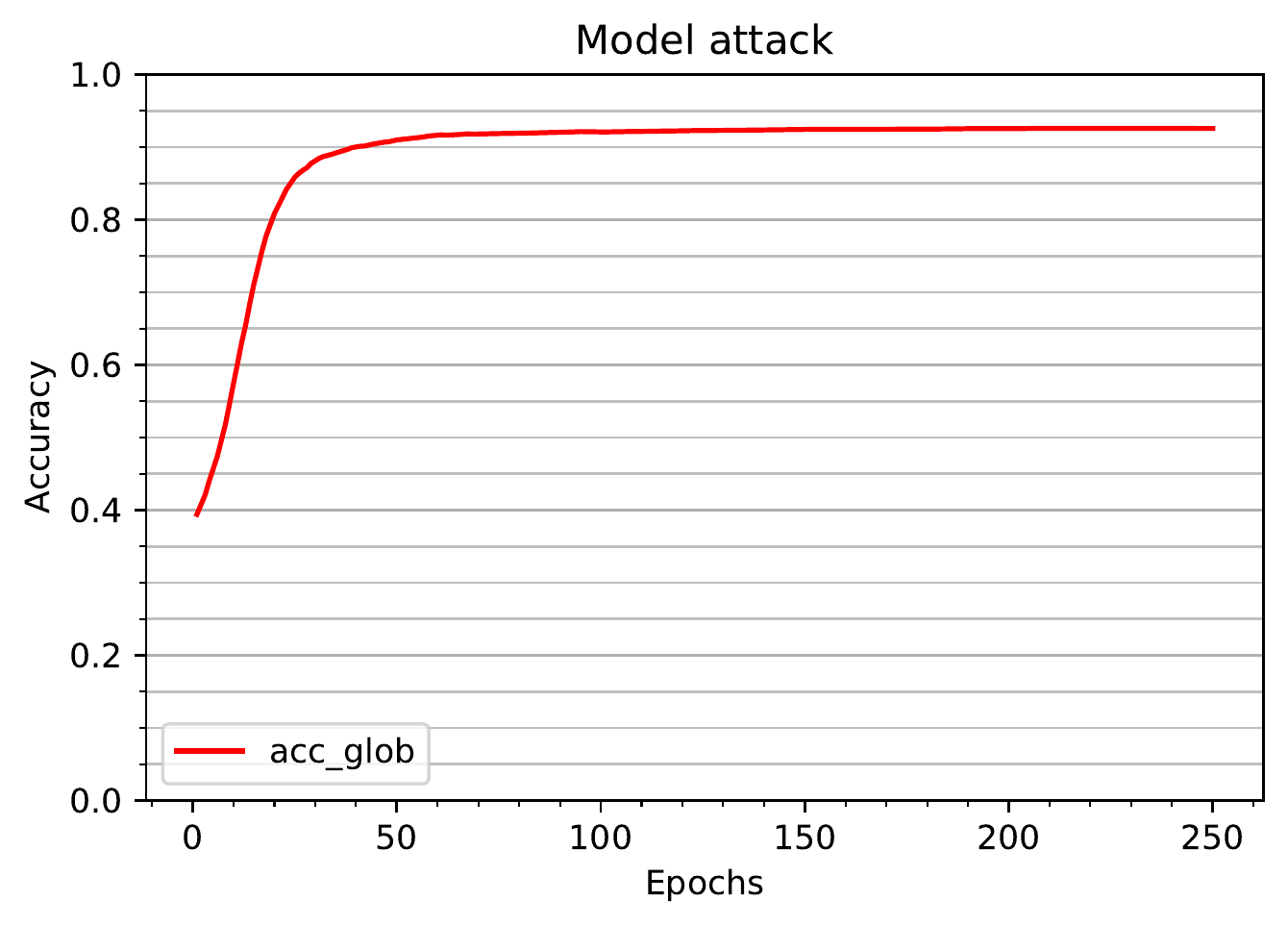}
        \phantomsubcaption\label{fig:model_attack_accuracy}}
        {(c)~\includegraphics[width=0.34\linewidth]{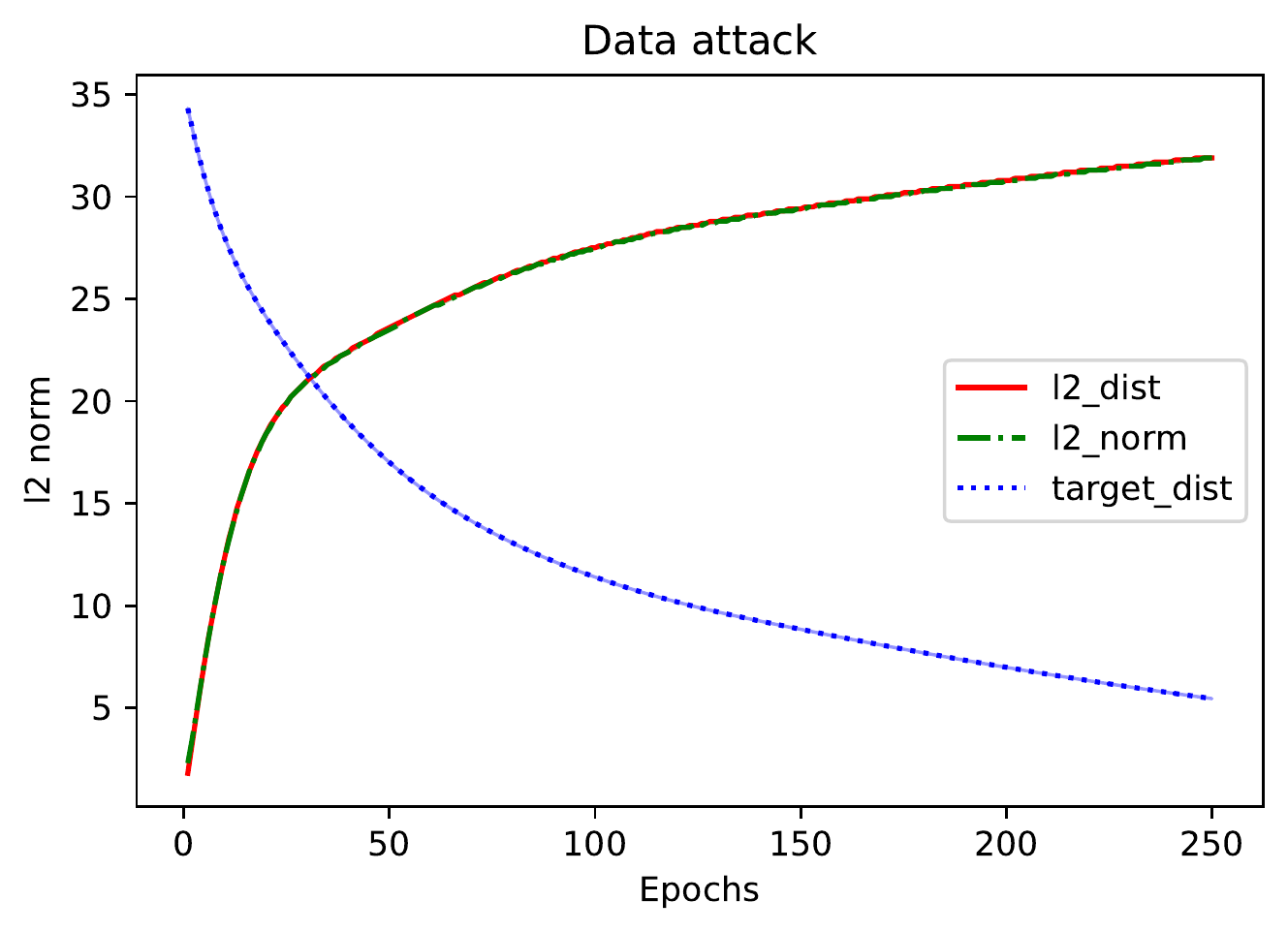}
        \phantomsubcaption\label{fig:gradient_attack}} \quad \quad
        {(d)~\includegraphics[width=0.34\linewidth]{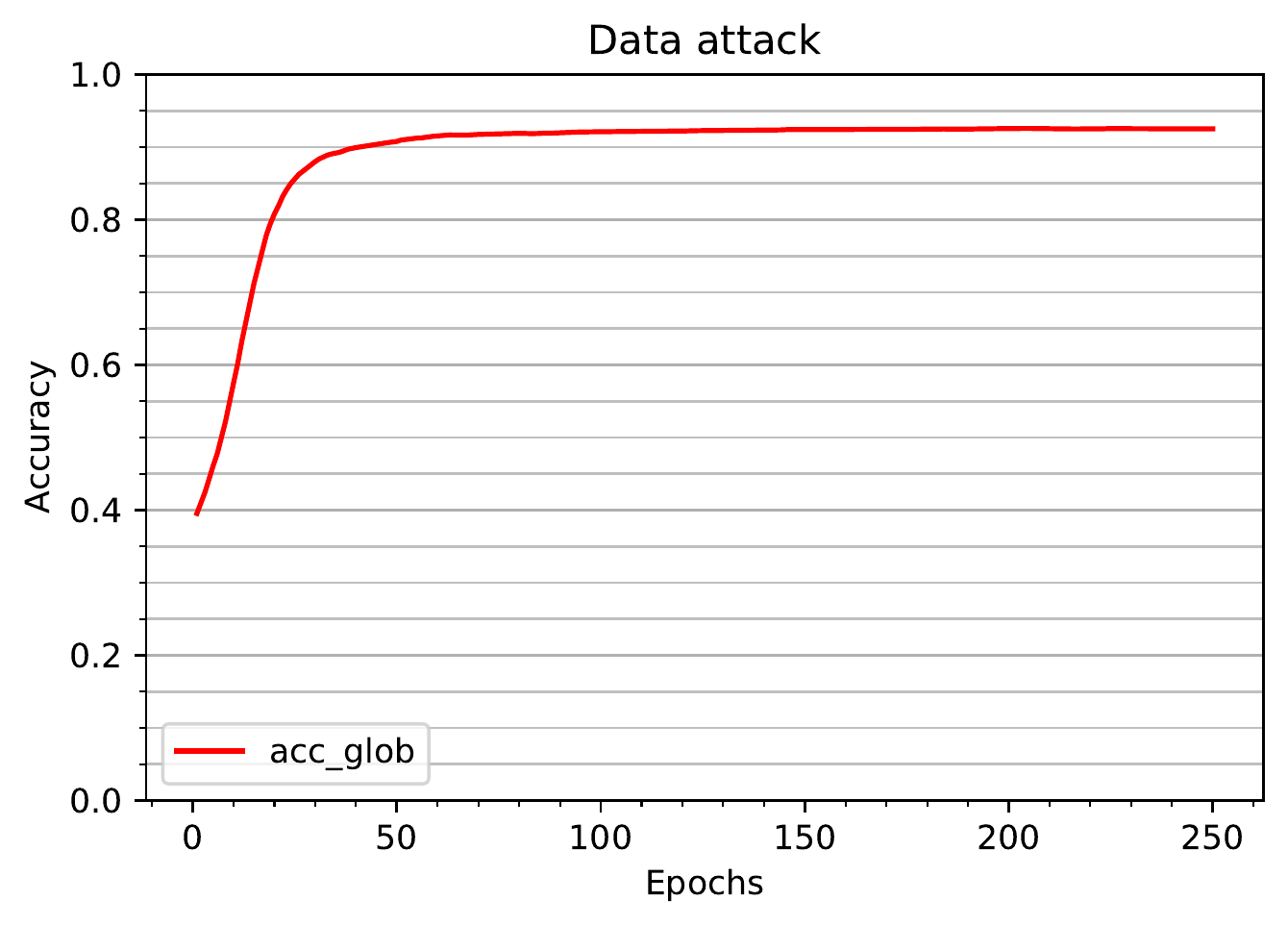}
        \phantomsubcaption\label{fig:data_attack}}
        \vspace{-2mm}
    \caption{
        (\subref{fig:model_attack_distance}) Distance between $\common^\iteration$ and $\trueparamsub{\strategicnode}$ (target\_dist), under model attack (combining \CounterGradientAttack{} and Proposition~\ref{prop:gradient_to_model_ell22}).
        (\subref{fig:model_attack_accuracy}) Accuracy of $\common^\iteration$ according to $\trueparamsub{\strategicnode}$ (which relabels $0\rightarrow1\rightarrow2\rightarrow ...\rightarrow 9 \rightarrow0$), under model attack (combining \CounterGradientAttack{} and Proposition~\ref{prop:gradient_to_model_ell22}).
        (\subref{fig:gradient_attack}) Distance between the global model $\common^\iteration$ and the target model $\trueparamsub{\strategicnode}$ (target\_dist), under our data poisoning attack.
        (\subref{fig:data_attack}) Accuracy of $\common^\iteration$ according to $\trueparamsub{\strategicnode}$ (which relabels $0\rightarrow1\rightarrow2\rightarrow ...\rightarrow 9 \rightarrow0$), under our data poisoning attack.
    }
    \vspace{-2mm}
    \label{fig:main}
\end{figure*}

\subsection{From Gradient Attack to Model Attack Against \texorpdfstring{$\ell_2^2$}{l2}}
\label{sec:data_poisoning}

%% file: data_poisoning.tex
We now show how to turn a gradient attack into model attack, against $\ell_2^2$ regularization.
It is trivial to transform any gradient $\gradient{\strategicnode}{\infty}$ such that $\common^\infty = \trueparamsub{\strategicnode}$ into a model attack by setting $\strategicvote \triangleq \trueparamsub{\strategicnode} - \frac{1}{2} \gradient{\strategicnode}{\infty}$, as guaranteed by the following result, and as depicted by Figure~\ref{fig:model_attack_distance} and Figure~\ref{fig:model_attack_accuracy}.

\begin{proposition}
\label{prop:gradient_to_model_ell22}
  Consider the $\ell_2^2$ regularization.
  Suppose that $\gradient{\strategicnode}{\iteration} \rightarrow \gradient{\strategicnode}{\infty}$ and $\common^\iteration \rightarrow \trueparamsub{\strategicnode}$, with a constant learning rate $\learningrate{\iteration} = \learningrate{}$.
  Then, under the model attack $\strategicvote \triangleq \trueparamsub{\strategicnode} - \frac{1}{2 \regweightsub{}} \gradient{\strategicnode}{\infty}$,
  the gradient at $\common = \trueparamsub{\strategicnode}$ vanishes, i.e. $\nabla_\common \globalloss{} (\trueparamsub{\strategicnode}, \optimumfamily_{-\strategicnode} (\trueparamsub{\strategicnode}, \datafamily{-\strategicnode}), \strategicvote, \data{-\strategicnode}) = 0$.
\end{proposition}

\begin{proof}
  Given a constant learning rate, the convergence $\common^\iteration \rightarrow \trueparamsub{\strategicnode}$ implies that the sum of honest users' gradients at $\common = \trueparamsub{\strategicnode}$ equals $-\gradient{\strategicnode}{\infty}$.
  Therefore, to achieve $\optcommon = \trueparamsub{\strategicnode}$, it suffices to send $\strategicvote$ such that the gradient of $\regweightsub{} \norm{\common - \strategicvote}{2}^2$ with respect to $\common$ at $\common = \trueparamsub{\strategicnode}$ equals $\gradient{\strategicnode}{\infty}$.
  Since the gradient is $\regweightsub{} (\trueparamsub{\strategicnode} - \strategicvote)$, $\strategicvote \triangleq \trueparamsub{\strategicnode} - \frac{1}{2\regweightsub{}} \gradient{\strategicnode}{\infty}$ does the trick.
\end{proof}

\subsection{From Model Attack to Data Poisoning Against \texorpdfstring{$\ell_2^2$}{l2}}
\label{sec:model_to_data_l22}

\paragraph{The case of linear regression.} 
In linear regression, any model attack can be turned into a \emph{single data} poisoning attack, as proved by the following theorem
whose proof is given in Appendix~\ref{app:data_poisoning_linear_regression}.

\begin{theorem}
\label{th:singleDataAttach}
  Consider the $\ell_2^2$ regularization and linear regression.
  For any data $\data{-\strategicnode}$ and any target value $\trueparamsub{\strategicnode}$,
  there is a datapoint $(\query{}, \answer{})$ to be injected by user $\strategicnode$ such that
  $\optcommon(\set{(\query{}, \answer{})}, \data{-\strategicnode}) = \trueparamsub{\strategicnode}$.
\end{theorem}

{
\begin{proof}[Sketch of proof]
  We first identify the sum $\gradient{}{}$ of honest users' gradients, if the global model $\common$ took the target value $\trueparamsub{\strategicnode}$.
  We then determine the value $\strategicvote$ that the strategic user's model must take, to counteract other users' gradients.
  Reporting datapoint $(\query{}, \answer{}) \triangleq (\gradient{}{}, \gradient{}{T} \strategicvote - 1)$ then guarantees that the strategic user's learned model will equal $\strategicvote$.
\end{proof}
Note that this single datapoint attack requires reporting a query $\query{}$ whose norm grows as $\Theta(\NODE)$, 
while the answer $\answer{}$ grows as $\Theta(\NODE^2)$.
Assuming a large number of users, this query will fall out of the distribution of users' queries, 
and could thus be flagged by basic outlier detection techniques.
We stress, however, that our proof can be trivially transformed into an attack with $\Theta(\NODE^2)$ data points,
all of which have a query whose norm is $\mathcal O(1)$.
}

\paragraph{The case of linear classification.}
We now consider linear classification, with the case of MNIST.
By Lemma~\ref{th:reduction_data_to_model}, any model attack can be turned into data poisoning, by (mis)labeling sufficiently many (random) data points,
However, this may require creating too many data labelings, especially if the norm of $\strategicvote$ is large (which holds if $\strategicnode$ faces many active users), as suggested by Theorem~\ref{th:logistic_regression}.

For efficient data poisoning, define the indifference affine subspace $V \subset \setR^d$ as the set of images with equiprobable labels.
Intuitively, labeling images close to $V$ is very informative, as it informs us directly about the separating hyperplanes.
To generate images, we draw random images, project them orthogonally on $V$ and add a small noise.
We then label the image probabilistically with model $\strategicvote$.

Figure~\ref{fig:data_attack} shows the effectiveness of the resulting data poisoning attack, with only 2,000 data points, 
as opposed to the 60,000 honestly labeled data points that the 10 other users cumulatively have.
Remarkably, complete data relabeling was achieved by poisoning merely 3.3\% of the total database.
More details are given in Appendix~\ref{app:data_poisoning_linear_classification}.

{ 
Note that this attack leads us to consider images not in $[0,1]^d$.
{ 
In Appendix \ref{sec:clipp_attack}, we report another equivalently effective attack, which only reports images in $[0,1]^d$,
though it requires significantly more data injection.
}
}

{
\subsection{Gradient Attack on Local Models}

Note that \CounterGradientAttack{} aims to merely bias the global model.
However, the attacker may instead prefer to bias other users' local models.
To this end, we present here a variant of \CounterGradientAttack{}, which targets the \emph{average} of other users' local models.
At each iteration of this variant, the attacker reports
\begin{equation*}
  \gradient{\strategicnode}{\iteration} 
  \in \argmin_{\gradient{}{} \in \GRADIENT(\common^{\iteration})}
  \norm{ \common^{\iteration} - \learningrate{\iteration} (\estimatedgradient{-\strategicnode}{\iteration} + \gradient{}{}) - \trueparamsub{\strategicnode} - \frac{ \estimatedgradient{-\strategicnode}{\iteration} }{2 \regweightsub{} (\NODE -1)} }{2}.
\end{equation*}
Figure \ref{fig:cga_local} shows the effectiveness of this attack.
{This gradient attack can evidently be turned into data poisoning similar to what was achieved for \CounterGradientAttack{}.}

}

\begin{figure}[h]
    \centering
    \includegraphics[width=.8\linewidth]{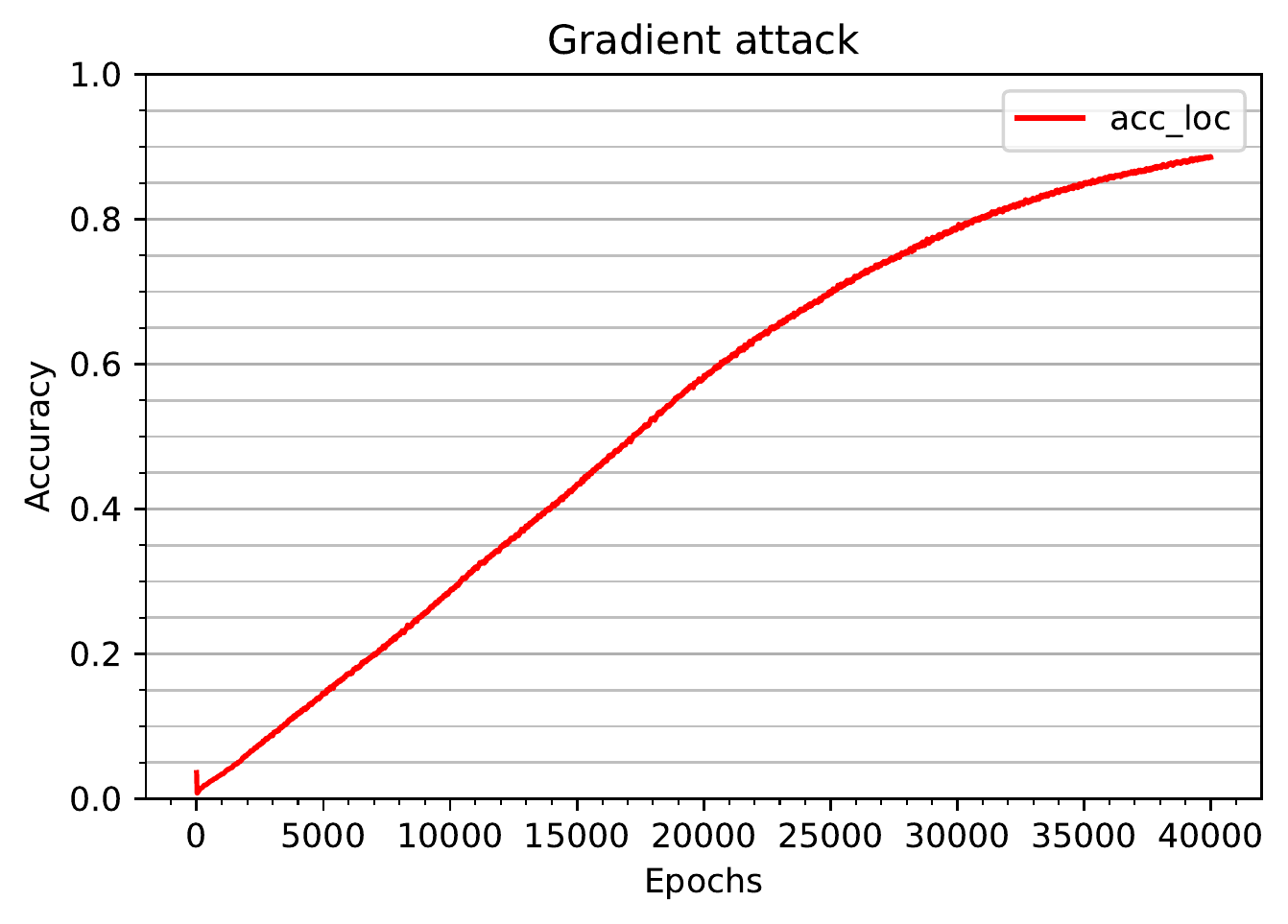}
    \caption{Accuracy of other users' average local models according to $\trueparamsub{\strategicnode}$ (which relabels $0\rightarrow1\rightarrow2\rightarrow ...\rightarrow 9 \rightarrow0$), when attacked by CGA variant.} 
    \label{fig:cga_local}
    \vspace{-5mm}
\end{figure}

%% file: conclusion.tex
\section{Conclusion}
\label{sec:conclusion}

We showed that, unlike what has been argued, e.g., \citet{ShejwalkarHKR21}, the gradient attack threat is not unrealistic.
More precisely, for personalized federated learning with local PAC* guarantees,
effective gradient attacks can be derived from
strategic data reporting, with potentially surprisingly few data.
In fact, by leveraging our newly found equivalence, we derived new impossibility theorems on what any robust learning algorithm can guarantee, under data poisoning attacks{, especially, in highly-heterogeneous  settings.
Yet such attacks are known to be ubiquitous for high-risk applications, many of which are known to feature especially high heterogeneity, like online content recommendation.}
Arguably, a lot more security measures are urgently needed to make large-scale learning algorithms safe.

%% file: acknowledgment.tex
\section*{Acknowledgement}
This work has been supported in part by the Swiss National Science Foundation (SNSF) project 200021\_200477. The authors are thankful to the anonymous reviewers of ICLR 2022 and ICML 2022 for their constructive comments.

%% file: ethics.tex
\section*{Ethics statement}

The safety of algorithms is arguably a prerequisite to their ethics.
After all, an arbitrarily manipulable large-scale algorithm will unavoidably endanger the targets of the entities that successfully design such algorithms.
Typically, unsafe large-scale recommendation algorithms may be hacked by health disinformation campaigns that aim to promote non-certified products, e.g., by falsely pretending that they cure COVID-19.
Such algorithms must not be regarded as ethical, even if they were designed with the best intentions.
We believe that our work helps understand the vulnerabilities of such algorithms, and will motivate further research in the ethics and security of machine learning.

%% file: convexity_lemmas.tex
\section{Convexity Lemmas}
\label{app:convex}

\subsection{General Lemmas}

\begin{definition}
    We say that $f : \setR^d \rightarrow \setR$ is locally strongly convex if,
    for any convex compact set $C \subset \setR^d$,
    there exists $\mu > 0$ such that $f$ is $\mu$-strongly convex on $C$,
    i.e. for any $x,y \in C$ and any $\lambda \in [0,1]$, we have
    \begin{equation}
      f(\lambda x + (1-\lambda) y)
      \leq \lambda f(x) + (1-\lambda) f(y) - \frac{\mu}{2} \lambda (1-\lambda) \norm{x-y}{2}^2.
    \end{equation}
    It is well-known that if $f$ is differentiable, this condition amounts to saying that $\norm{\nabla f(x) - \nabla f(y)}{2} \geq \mu \norm{x-y}{2}$ for all $x,y \in C$. 
    And if $f$ is twice differentiable, then it amounts to saying $\nabla^2 f(x) \succeq \mu \identitymatrix$ for all $x \in C$.
\end{definition}

\begin{lemma}
\label{lemma:additivity_strongly_convex}
  If $f$ is locally strongly convex and $g$ is convex, then $f+g$ is locally strongly convex.
\end{lemma}

\begin{proof}
  Indeed, $(f+g)(\lambda x + (1-\lambda) y)
  \leq \lambda f(x) + (1-\lambda) f(y) - \frac{\mu}{2} \lambda (1-\lambda) \norm{x-y}{2}^2 + \lambda g(x) + (1-\lambda) g(y)
  = \lambda (f+g)(x) + (1-\lambda) (f+g)(y) - \frac{\mu}{2} \lambda (1-\lambda) \norm{x-y}{2}^2$.
\end{proof}

\begin{definition}
    We say that $f : \setR^d \rightarrow \setR$ is $L$-smooth
    if it is differentiable and if its gradient is $L$-Lipschitz continuous, i.e. for any $x,y \in \setR^d$,
    \begin{equation}
      \norm{\nabla f(x) - \nabla f(y)}{2} \leq L \norm{x-y}{2}.
    \end{equation}
\end{definition}

\begin{lemma}
\label{lemma:additivity_smoothness}
  If $f$ is $L_f$-smooth and $g$ is $L_g$-smooth, then $f+g$ is $(L_f+L_g)$-smooth.
\end{lemma}

\begin{proof}
  Indeed, $\norm{\nabla (f+g)(x) - \nabla (f+g)(y)}{2}
  \leq \norm{\nabla f(x) - \nabla f(y)}{2} + \norm{\nabla g(x) - \nabla g(y)}{2}
  \leq L_f \norm{x-y}{2} + L_g \norm{x-y}{2}
  = (L_f + L_g) \norm{x-y}{2}.$
\end{proof}

\begin{lemma}
\label{lemma:optimum_lipschitz}
  Suppose that $f : \setR^d \times \setR^{d'} \mapsto \setR$ is locally strongly convex and $L$-smooth, 
  and that, for any $x \in X$, where $X \subset \setR^d$ is a convex compact subset, the map $y \mapsto f(x,y)$ has a minimum $y^*(x)$.
  Note that local strong convexity guarantees the uniqueness of this minimum.
  Then, there exists $K$ such that the function $y^*$ is $K$-Lipschitz continuous on $X$.
\end{lemma}

\begin{proof}
  The existence and uniqueness of $y^*(x)$ hold by strong convexity.
  Fix $x,x'$.
  By optimality of $y^*$, we know that $\nabla_y f(x,y^*(x)) = \nabla_y f(x', y^*(x')) = 0$.
  We then have the following bounds
  \begin{align}
    \mu \norm{y^*(x) - y^*(x')}{2}
    &\leq \norm{\nabla_y f(x,y^*(x)) - \nabla_y f(x,y^*(x'))}{2}
    = \norm{\nabla_y f(x,y^*(x'))}{2} \\
    &= \norm{\nabla_y f(x,y^*(x')) - \nabla_y f(x', y^*(x'))}{2} \\
    &\leq \norm{\nabla f(x,y^*(x')) - \nabla f(x', y^*(x'))}{2} \\
    &\leq L \norm{(x-x', y^*(x')-y^*(x'))}{2}
    = L \norm{x-x'}{2},
  \end{align}
  where we first used the local strong convexity assumption, then the fact that $\nabla_y f(x,y^*(x)) = 0$, then the fact that $\nabla_y f(x', y^*(x')) = 0$, and then the $L$-smooth assumption.
\end{proof}

\begin{lemma}
\label{lemma:minDerivative}
  Suppose that $f : \setR^d \times \setR^{d'} \mapsto \setR$ is locally strongly convex and $L$-smooth,
  and that, for any $x \in X$, where $X \subset \setR^d$ is a convex compact subset, the map $y \mapsto f(x,y)$ has a minimum $y^*(x)$.
  Define $g(x) \triangleq \min_{y \in Y} f(x,y)$.
  Then $g$ is convex and differentiable on $X$ and $\nabla g(x) = \nabla_x f(x,y^*(x))$.
\end{lemma}

\begin{proof}
   First we prove that $g$ is convex.
   Let $x_1, x_2 \in \setR^d$, and $\lambda_1, \lambda_2 \in [0,1]$ with $\lambda_1+\lambda_2 = 1$.
   For any $y_1, y_2 \in \setR^{d'}$, we have
   \begin{align}
       g(\lambda_1 x_1 + \lambda_2 x_2)
       &= \min_{y \in \setR^{d'}} f(\lambda_1 x_1 + \lambda_2 x_2, y) \\
       &\leq f(\lambda_1 x_1 + \lambda_2 x_2, \lambda_1 y_1 + \lambda_2 y_2) \\
       &\leq \lambda_1 f(x_1, y_1) + \lambda_2 f(x_2, y_2).
   \end{align}
   Taking the infimum of the right-hand side over $y_1$ and $y_2$ yields $g(\lambda_1 x_1 + \lambda_2 x_2) \leq \lambda_1 g(x_1) + \lambda_2 g(x_2)$, which proves the convexity of $g$.
   
   Now denote $h(x) = \nabla_x f(x,y^*(x))$.
   We aim to show that $\nabla g(x) = h(x)$.
   Let $\varepsilon \in \setR^d$ small enough so that $x +\varepsilon \in X$.
   Now note that we have
   \begin{align}
       g(x+\varepsilon) &= \min_{y \in \setR^{d'}} f(x+\varepsilon, y)
       \leq f(x+\varepsilon, y^*(x)) \\
       &= f(x,y^*(x)) + \varepsilon^T \nabla_x f(x,y^*(x)) + o(\norm{\varepsilon}{2}) \\
       &= g(x) + \varepsilon^T h(x) + o(\norm{\varepsilon}{2}),
   \end{align}
   which shows that $h(x)$ is a superderivative of $g$ at $x$.
   We now show that it is also a subderivative.
   To do so, first note that its value at $x+\varepsilon$ is approximately the same, i.e.
   \begin{align}
       \norm{h(x+\varepsilon) - h(x)}{2} 
       &\leq \norm{\nabla_x f(x+\varepsilon, y^*(x+\varepsilon)) - \nabla_x f(x, y^*(x+\varepsilon))}{2} \nonumber \\
       &\qquad \qquad \qquad + \norm{\nabla_x f(x, y^*(x+\varepsilon)) - \nabla_x f(x,y^*(x))}{2} \\
       &\leq L \norm{\varepsilon}{2} + L \norm{y^*(x+\varepsilon) - y^*(x)}{2}
       \leq \left( L + \frac{L^2}{\mu} \right) \norm{\varepsilon}{2},
   \end{align}
   where we used the $L$-smoothness of $f$ and Lemma~\ref{lemma:optimum_lipschitz}.
   Now notice that 
   \begin{align}
       g(x) &= \min_{y \in \setR^{d'}} f(x, y)
       \leq f(x, y^*(x + \varepsilon))
       = f((x+\varepsilon) - \varepsilon, y^*(x + \varepsilon)) \\
       &= f(x+\varepsilon, y^*(x+\varepsilon)) - \varepsilon^T \nabla_x f(x+\varepsilon, y^*(x+\varepsilon)) + o(\norm{\varepsilon}{2}) \\
       &= g(x+\varepsilon) - \varepsilon^T h(x)
       - \varepsilon^T \left( h(x+\varepsilon) - h(x) \right) + o(\norm{\varepsilon}{2}),
   \end{align}
   But we know that $\norm{h(x+\varepsilon) - h(x)}{2} = \mathcal O(\norm{\varepsilon}{2})$.
   Rearranging the terms then yields
   \begin{equation}
       g(x+\varepsilon) \geq g(x) + \varepsilon^T h(x) - o(\norm{\varepsilon}{2}),
   \end{equation}
   which shows that $h(x)$ is also a subderivative.
   Therefore, we know that $g(x+\varepsilon) = g(x) + \varepsilon^T h(x) + o(\norm{\varepsilon}{2})$, which boils down to saying that $g$ is differentiable in $x \in X$, and that $\nabla g(x) = h(x)$.
\end{proof}

\begin{lemma}
\label{lemma:infimum_strongly_convex}
Suppose that $f : X \times \setR^{d'} \rightarrow \setR$ is $\mu$-strongly convex, where $X \subset \setR^d$ is closed and convex.
Then $g : X \rightarrow \setR$, defined by $g(x) = \inf_{y \in Y} f(x,y)$, is well-defined and $\mu$-strongly convex too.
\end{lemma}

\begin{proof}
    The function $y \mapsto f(x,y)$ is still strongly convex, which means that it is at least equal to a quadratic approximation around 0, which is a function that goes to infinity in all directions as $\norm{y}{2} \rightarrow \infty$.
    This proves that the infimum must be reached within a compact set, which implies the existence of a minimum.
    Thus $g$ is well-defined.
    Moreover, for any $x_1, x_2 \in X, y_1, y_2 \in \setR^{d'}$, and $\lambda_1, \lambda_2 \geq 0$ with $\lambda_1 + \lambda_2 = 1$, we have
    \begin{align}
        g(\lambda_1 x_1 + \lambda_2 x_2)
        &= \inf_y f(\lambda_1 x_1 + \lambda_2 x_2, y) \\
        &\leq f(\lambda_1 x_1 + \lambda_2 x_2, \lambda_1 y_1 + \lambda_2 y_2) \\
        &\leq \lambda_1 f(x_1, y_1) + \lambda_2 f(x_2, y_2) - \frac{\mu}{2} \lambda_1 \lambda_2 \norm{(x_1 - x_2, y_1 - y_2)}{2}^2 \\
        &\leq \lambda_1 f(x_1, y_1) + \lambda_2 f(x_2, y_2) - \frac{\mu}{2} \lambda_1 \lambda_2 \norm{x_1 - x_2}{2}^2,
    \end{align}
    where we used the $\mu$-strong convexity of $f$.
    Taking the infimum over $y_1, y_2$ implies the $\mu$-strong convexity of $g$.
\end{proof}

\subsection{Applications to \globalloss{}}
Now instead of proving our theorems for different cases separately, we make the following assumptions on the components of the global loss that encompasses both $\ell_2^2$ and smooth-$\ell_2$ regularization, a well as linear regression and logistic regression.

\begin{assumption}
\label{ass:assumption1}
  Assume that $\lossperinput$ is convex and $L_\lossperinput$-smooth,
  and that $\regularization(\common, \param) = \regularization_0 (\common - \param)$, where $\regularization_0 : \setR^d \rightarrow \setR$ is locally strongly convex (i.e. strongly convex on any convex compact set), $L_{\regularization_0}$-smooth 
  and satisfy $\regularization_0(z) = \Omega(\norm{z}{2})$ as $\norm{z}{2} \rightarrow \infty$.
\end{assumption}

\begin{lemma}
\label{lemma:smooth_globalloss}
  Under Assumption~\ref{ass:assumption1}, $\globalloss{}$ is locally strongly convex and $L$-smooth.
\end{lemma}

\begin{proof}
  All terms of $\globalloss{}$ are $L_0$-smooth, for an appropriate value of $L_0$.
  By Lemma~\ref{lemma:additivity_smoothness}, their sum is thus also $L$-smooth, for an appropriate value of $L$.
  Now, given Lemma~\ref{lemma:additivity_strongly_convex}, to prove that $\globalloss{}$ is locally strongly convex, it suffices to prove that $\reglocalweight{} \sum \norm{\paramsub{\node}}{2}^2 + \regularization_0 (\common - \paramsub{1})$ is locally strongly convex.
  Consider any convex compact set $C \subset \setR^{d \times (1+\NODE)}$.
  Since $\regularization_0$ is locally strongly convex, we know that there exists $\mu > 0$ such that $\nabla^2 \regularization_0 \succeq \mu I$.
  As a result,
  \begin{align}
      (\common, \paramfamily{})^T &\left(\nabla^2 \globalloss{}\right) (\common, \paramfamily{})
      \geq \reglocalweight{} \sum_{\node \in [\NODE]} \norm{\paramsub{\node}}{2}^2 
      + \mu \norm{\common - \paramsub{1}}{2}^2 \\
      &= \reglocalweight{} \norm{\paramsub{1}}{2}^2 + \mu \norm{\common}{2}^2 + \mu \norm{\paramsub{1}}{2}^2 - 2 \mu \common^T \paramsub{1} + \reglocalweight{} \sum_{\node \neq 1} \norm{\paramsub{\node}}{2}^2.
  \end{align}
  Now define $\alpha \triangleq \sqrt{\frac{2 \mu}{\reglocalweight{} + 2\mu}}$.
  Clearly, $0 < \alpha < 1$.
  Moreover, $0 \leq \norm{\frac{1}{\alpha} \paramsub{1} - \alpha \common}{2}^2 = \frac{1}{\alpha^2} \norm{\paramsub{1}}{2}^2 + \alpha^2 \norm{\common}{2}^2 - 2 \common^T \paramsub{1}$.
  Therefore $2 \common^T \paramsub{1} \leq \alpha^2 \norm{\common}{2}^2 + \frac{1}{\alpha^2} \norm{\paramsub{1}}{2}^2$, which thus implies
  \begin{align}
      (\common, \paramfamily{})^T &\left(\nabla^2 \globalloss{}\right) (\common, \paramfamily{})
      \geq \left(\reglocalweight{} + \mu \left(1- \alpha^{-2} \right)\right) \norm{\paramsub{1}}{2}^2 + \mu \left(1- \alpha^2 \right) \norm{\common}{2}^2 
      + \reglocalweight{} \sum_{\node \neq 1} \norm{\paramsub{\node}}{2}^2 \\
      &\geq \frac{\reglocalweight{}}{2} \norm{\paramsub{1}}{2}^2 
      + \frac{2 \reglocalweight{} \mu}{\reglocalweight{} + 2\mu} \norm{\common}{2}^2 
      + \reglocalweight{} \sum_{\node \neq 1} \norm{\paramsub{\node}}{2}^2 
      \geq \min \set{\frac{\reglocalweight{}}{2}, \frac{2 \reglocalweight{} \mu}{\reglocalweight{} + 2\mu} } \norm{(\common, \paramfamily{})}{2}^2,
  \end{align}
  which proves that $\nabla^2 \globalloss{} \succeq \kappa \identitymatrix$, with $\kappa > 0$.
  This shows that $\globalloss{}$ is locally strongly convex.
\end{proof}

\begin{lemma}
\label{lemma:lipschitz_local_optima}
    Under Assumption~\ref{ass:assumption1}, 
    $\common \mapsto \optimumfamily{} (\common, \datafamily{})$ is Lipchitz continuous on any compact set.
\end{lemma}

\begin{proof}
   Define $f_\node(\common, \paramsub{\node}) \triangleq \reglocalweight{} \norm{\paramsub{\node}}{2}^2 + \sum_{\varx \in \data{\node}} \lossperinput(\paramsub{\node}, \varx) + \regweightsub{} \norm{\common - \paramsub{\node}}{2}^2$.
   If $\lossperinput$ is $L$-smooth, then $f_\node$ is clearly $(\card{\data{\node}} L + \reglocalweight{} + \regweightsub{})$-smooth.
   Moreover, if $\lossperinput$ is convex, then for any $\common$, the function $\paramsub{\node} \mapsto f_\node (\common, \paramsub{\node})$ is at least $\reglocalweight{}$-strongly convex.
   Thus Lemma~\ref{lemma:optimum_lipschitz} applies, which guarantees that $\common \mapsto \optimumfamily{} (\common, \datafamily{})$ is Lipchitz.
\end{proof}

\begin{lemma}
\label{lemma:minimizedLossSmooth}
    Under Assumption~\ref{ass:assumption1}, $\common \mapsto \globalloss{} (\common, \optimumfamily (\common, \datafamily{}), \datafamily{})$ is $L$-smooth and locally strongly convex.
\end{lemma}

\begin{proof}
   By Lemma~\ref{lemma:smooth_globalloss}, the global loss is known to be $L$-smooth, for some value of $L$ and locally strongly convex.
   Denoting $f : \common \mapsto \globalloss{} (\common, \optimumfamily (\common, \datafamily{}), \datafamily{})$, we then have
   \begin{align}
       \norm{\nabla f(\common) - \nabla f(\common')}{2}
       &\leq \norm{\nabla_\common \globalloss{} (\common, \optimumfamily (\common, \datafamily{}), \datafamily{}) - \nabla_\common \globalloss{} (\common', \optimumfamily (\common', \datafamily{}), \datafamily{})}{2} \\
       &\leq L \norm{(\common, \optimumfamily (\common, \datafamily{})) - (\common', \optimumfamily (\common', \datafamily{}))}{2} \\
       &\leq L \norm{\common - \common'}{2},
   \end{align}
   which proves that $f$ is $L$-smooth.
   
   For strong convexity, note that since the global loss function is locally strongly convex, for any compact convex set $C$, there exists $\mu$ such that $\globalloss{} (\common, \paramfamily{}, \datafamily{})$ is $\mu$-strongly convex on $C=(C_1,C_2)\subset(\setR^d,\setR^{N\times d})$, therefore, by Lemma \ref{lemma:infimum_strongly_convex}, $f(\common)$ will also be $\mu$-strongly convex on $C_1$ which means that $f(\common)$ is locally strongly convex.
\end{proof}

%% file: reduction_proof.tex
\section{Proof of the Equivalence}
\label{app:reduction}

\subsection{Proof of the Reduction from Model Attack to Data Poisoning}
\label{app:reduction_model_to_data}

\begin{proof}[Proof of Lemma~\ref{th:reduction_model_to_data}]
  We omit making the dependence of the optima on $\datafamily{}$ explicit,
  and we consider any other models $\common$ and $\paramfamily_{-\strategicnode}$.
  We have the following inequalities:
  \begin{align}
    \ModifiedLoss{\strategicnode} &(\optcommon, \optimumfamily_{-\strategicnode}, \strategicvote, \datafamily{})
    = \globalloss{} (\optcommon, \optimumfamily{}, \datafamily{})
    - \localloss{} (\optimumsub{\strategicnode}, \data{\strategicnode}) \\
    &\leq \globalloss{} (\common, (\paramfamily_{-\strategicnode}, \optimumsub{\strategicnode}), \datafamily{})
    - \localloss{} (\optimumsub{\strategicnode}, \data{\strategicnode})
    = \ModifiedLoss{\strategicnode} (\common, \paramfamily_{-\strategicnode}, \strategicvote, \datafamily{}),
  \end{align}
  where we used the optimality of $(\optcommon, \optimumfamily{})$ in the second line, and where we repeatedly used the fact that $\optimumsub{\strategicnode} = \strategicvote$.
  This proves that $(\optcommon, \optimumfamily_{-\strategicnode})$ is a global minimum of the modified loss.
\end{proof}

\subsection{Proof of the Reduction from Data Poisoning to Model Attack}
\label{app:data_to_model}

First, we define the following modified loss function:
\begin{equation}
\label{eq:modified_loss}
  \ModifiedLoss{\strategicnode} (\common, \paramfamily_{-\strategicnode}, \strategicvote, \datafamily{-\strategicnode})
  \triangleq
  \globalloss{} (\common, (\strategicvote, \paramfamily_{-\strategicnode}), (\emptyset, \datafamily{-\strategicnode}))
\end{equation}
where $\paramfamily_{-\strategicnode}$ and $\datafamily{-\strategicnode}$ are variables and datasets for users $\node \neq \strategicnode$.
We then define $\optcommon{} (\strategicvote, \datafamily{-\strategicnode})$ and $\optimumfamily_{-\strategicnode}(\strategicvote, \datafamily{-\strategicnode})$ as a minimum of the modified loss function,
and $\optimumsub{\strategicnode}(\strategicvote, \datafamily{-\strategicnode}) \triangleq \strategicvote$.
We now prove a slightly more general version of Lemma~\ref{th:reduction_data_to_model}, which applies to a larger class of regularizations.
It also shows how to construct the strategic's user data poisoning attack.

\begin{lemma}[Reduction from data poisoning to model attack]
\label{th:general_reduction_data_to_model}
  Assume local PAC* learning.
  Suppose also that $\regularization$ is continuous and that $\regularization(\common, \param) \rightarrow \infty$ when $\norm{\common - \param}{2} \rightarrow \infty$.
  Consider any datasets $\data{-\strategicnode}$ and any attack model $\strategicvote$ such that the modified loss $\ModifiedLoss{\strategicnode}$ has a unique minimum $\optcommon(\strategicvote, \datafamily{-\strategicnode}), \optimumfamily_{-\strategicnode}(\strategicvote, \datafamily{-\strategicnode})$.
  Then, for any $\varepsilon, \delta >0$,
  there exists $\NODEINPUT{}$ such that
  if user $\strategicnode$'s dataset $\data{\strategicnode}$ contains at least $\NODEINPUT{}$ inputs drawn from model $\strategicvote$,
  then, with probability at least $1-\delta$,
  we have
  \begin{equation}
    \norm{\optcommon{}(\datafamily{}) - \optcommon{}(\strategicvote, \datafamily{-\strategicnode})}{2} \leq \varepsilon
    ~~\text{and}~~
    \forall \node \neq \strategicnode \mathsep \norm{\optimumsub{\node}(\datafamily{}) - \optimumsub{\node}(\strategicvote, \datafamily{-\strategicnode})}{2} \leq \varepsilon.
  \end{equation}
\end{lemma}

Clearly, $\ell_2^2$, $\ell_2$ and smooth-$\ell_2$ are continuous regularizations, and verify $\regularization(\common, \param) \rightarrow \infty$ when $\norm{\common - \param}{2} \rightarrow \infty$.
Moreover, setting $\delta \triangleq 1/2$ shows that the probability that the dataset $\data{\strategicnode}$ satisfies the inequalities of Lemma~\ref{th:general_reduction_data_to_model} is positive.
This implies in particular that there must be a dataset $\data{\strategicnode}$ that satisfies these inequalities.
All in all, this shows that Lemma~\ref{th:general_reduction_data_to_model} implies Lemma~\ref{th:reduction_data_to_model}.

\begin{proof}[Proof of Lemma \ref{th:general_reduction_data_to_model}]
  Let $\varepsilon, \delta > 0$ and $\strategicvote \in \setR^d$.
  Denote $\common^\spadesuit \triangleq \optcommon(\strategicvote, \datafamily{-\strategicnode})$
  and $\paramfamily^\spadesuit \triangleq \optimumfamily (\strategicvote, \datafamily{-\strategicnode})$ the result of strategic user $\strategicnode$'s model attack.
  We define the compact set $C$ by
  \begin{equation}
    C \triangleq \set{\common, \paramfamily_{-\strategicnode} \st
    \norm{\common{} - \common^\spadesuit}{2} \leq \varepsilon \wedge
    \forall \node \neq \strategicnode \mathsep \norm{\paramsub{\node} - \paramsub{\node}^\spadesuit}{2} \leq \varepsilon
    }
  \end{equation}
  We define $D \triangleq \overline{\setR^{d \times N} - C}$ the closure of the complement of $C$.
  Clearly, $\common^\spadesuit, \paramfamily_{-\strategicnode}^\spadesuit \notin D$.
  We aim to show that, when strategic user $\strategicnode$ reveals a large dataset $\data{\strategicnode}$ whose answers are provided using the attack model $\strategicvote$, then the same holds for any global minimum of the global loss $\optcommon{}(\datafamily{}), \optimumfamily_{-\strategicnode}(\datafamily{}) \in C$.
  Note that, to prove this, it suffices to prove that the modified loss takes too large values, even when $\strategicvote$ is replaced by $\optimumsub{\strategicnode} (\datafamily{})$.

  Let us now formalize this.
  Denote $L^\spadesuit \triangleq \ModifiedLoss{\strategicnode} (\common^\spadesuit, \paramfamily_{-\strategicnode}^\spadesuit, \strategicvote, \datafamily{-\strategicnode})$.
  We define
  \begin{equation}
    \eta \triangleq \inf_{\common, \paramfamily_{-\strategicnode} \in D}
    \ModifiedLoss{\strategicnode} (\common, \paramfamily_{-\strategicnode}, \strategicvote, \datafamily{-\strategicnode})
    - L^\spadesuit.
  \end{equation}

  By a similar argument as that of Lemma \ref{th:pac}, using the assumption $\regularization \rightarrow \infty$ at infinity, we know that the infimum is actually a minimum.
  Moreover, given that the minimum of the modified loss $\ModifiedLoss{\strategicnode}$ is unique, we know that the value of the loss function at this minimum is different from its value at $\common^\spadesuit, \paramfamily_{-\strategicnode}^\spadesuit$.
  As a result, we must have $\eta >0$.

  Now, since the function $\regularization$ is differentiable, it must be continuous.
  By the Heine–Cantor theorem, it is thus uniformly continuous on all compact sets.
  Thus, there must exist $\kappa > 0$ such that,
  for all models $\paramsub{\strategicnode}$ satisfying $\norm{\paramsub{\strategicnode} - \strategicvote}{2} \leq \kappa$,
  we have
  \begin{equation}
    \absv{
      \regularization(\paramsub{\strategicnode}, \common^\spadesuit)
      - \regularization(\strategicvote, \common^\spadesuit)
    }
    \leq \eta /3.
  \end{equation}
  Now, Lemma \ref{th:pac} guarantees the existence of $\NODEINPUT{}$ such that,
  if user $\strategicnode$ provides a dataset $\data{\strategicnode}$ of least $\NODEINPUT{}$ answers with the model $\strategicvote$,
  then with probability at least $1-\delta$,
  we will have $\norm{\optimumsub{\strategicnode}(\datafamily{}) - \strategicvote}{2} \leq \min(\kappa, \varepsilon)$.
  Under this event, we then have
  \begin{equation}
    \ModifiedLoss{\strategicnode} \left(
    \common^\spadesuit,
    \paramfamily_{-\strategicnode}^\spadesuit,
    \optimumsub{\strategicnode}(\datafamily{}),
    \datafamily{-\strategicnode} \right)
    \leq L^\spadesuit + \eta / 3.
  \end{equation}
  Then
  \begin{align}
    \inf_{\common, \paramfamily_{-\strategicnode} \in D}
    &\ModifiedLoss{\strategicnode} (\common, \paramfamily_{-\strategicnode}, \optimumsub{\strategicnode}(\datafamily{}), \datafamily{-\strategicnode})
    \geq
    \inf_{\common, \paramfamily_{-\strategicnode} \in D}
    \ModifiedLoss{\strategicnode} (\common, \paramfamily_{-\strategicnode}, \strategicvote, \datafamily{-\strategicnode}) - \eta/3 \\
    &\geq L^\spadesuit + \eta - \eta/3
    \geq L^\spadesuit + 2 \eta /3 \\
    &> \ModifiedLoss{\strategicnode} \left(
    \common^\spadesuit,
    \paramfamily_{-\strategicnode}^\spadesuit,
    \optimumsub{\strategicnode}(\datafamily{}),
    \datafamily{-\strategicnode} \right).
  \end{align}
  This shows that there is a high probability event under which the minimum of $\common, \paramfamily_{-\strategicnode} \mapsto \ModifiedLoss{\strategicnode} \left(
  \common, \paramfamily_{-\strategicnode},
  \optimumsub{\strategicnode}(\datafamily{}),
  \datafamily{-\strategicnode} \right)$ cannot be reached in $D$.
  This is equivalent to what the theorem we needed to prove states.
\end{proof}

\subsection{Proof of Reduction from Model Attack to Gradient Attack}
\label{app:reduction_model_to_gradient}

\color{black}
\begin{proof}[Proof of Lemma \ref{th:gradient_attack_convergence}]
We define
\begin{align}
  \minimizedLoss{\strategicnode}(\common) 
  &\triangleq \inf_{\paramfamily_{-\strategicnode}} \set{\globalloss{} (\common, \paramfamily{}, \datafamily{}) -
    \localloss{\strategicnode} (\paramsub{\strategicnode}, \data{\strategicnode})
    - \regularization (\common, \paramsub{\strategicnode})}+\common^T\gradient{\strategicnode}{\infty} \\
  &= \inf_{\paramfamily_{-\strategicnode}} \set{ \sum_{\node \neq \strategicnode} \localloss{\node} (\paramsub{\node}, \data{\node}) + \sum_{\node \neq \strategicnode} \regularization(\common, \paramsub{\node}) } + \common^T \gradient{\strategicnode}{\infty},
\end{align}
 By Lemma \ref{lemma:minimizedLossSmooth}, we know that $\minimizedLoss{\strategicnode}(\common)$ is locally strongly convex and has a unique minimum. By the definition of $\common^{\infty}$, we must have $\sum_{\node \neq \strategicnode} \nabla_\common \regularization(\common^\infty, \optimumsub{\node}(\common^\infty)) + \gradient{\strategicnode}{\infty} = 0$, and thus $\nabla_\common \minimizedLoss{\strategicnode}(\common^\infty)  = 0$. Now define 
  \begin{align}
     \minimizedLosstwo{\strategicnode}(\common,\paramsub{\strategicnode}) &\triangleq
     \inf_{\paramfamily_{-\strategicnode}} \set{\globalloss{} (\common, \paramfamily{}, \datafamily{}) -
    \localloss{\strategicnode} (\paramsub{\strategicnode}, \data{\strategicnode})
   }\\
    &= \minimizedLoss{\strategicnode}(\common) + \regularization(\common,\paramsub{\strategicnode}) -\common^T\gradient{\strategicnode}{\infty},
  \end{align}
  and $\common^*(\paramsub{\strategicnode})$, its minimizer.
  Therefore, we have 
  \begin{equation}
    \nabla_\common\minimizedLosstwo{\strategicnode}(\common,\paramsub{\strategicnode}) =  \nabla_\common\minimizedLoss{\strategicnode}(\common) + \nabla_\common \regularization(\common,\paramsub{\strategicnode}) - \gradient{\strategicnode}{\infty}.
  \end{equation}
   By Lemma \ref{lemma:minimizedLossSmooth}, we know that $\minimizedLosstwo{\strategicnode}$ is locally strongly convex. Therefore, there exists $\mu_1>0$ such that $\minimizedLosstwo{\strategicnode}(\common,\paramsub{\strategicnode})$ is $\mu_1$-strongly convex in $\set{(\paramsub{\strategicnode},\common): \norm{\nabla_\common \regularization(\common^\infty,\paramsub{\strategicnode}) - \gradient{\strategicnode}{\infty}}{2} \leq \ \varepsilon_2 ,\norm{\common-\common^*(\paramsub{\strategicnode})}{2}\leq1}$ for $\varepsilon_2$ small enough. 
   Therefore, since $\nabla_\common\minimizedLosstwo{\strategicnode}(\common^*(\paramsub{\strategicnode}),\paramsub{\strategicnode}) = 0$, for any $0<\varepsilon<1$, if $\norm{\common^\infty-\common^*(\paramsub{\strategicnode})}{2}>\varepsilon$, we then have
   \begin{align}
      \varepsilon \norm{\nabla_\common \minimizedLosstwo{\strategicnode}(\common^\infty,\paramsub{\strategicnode})}{2} &\geq (\common^\infty-\common^*(\paramsub{\strategicnode}))^T\nabla_\common \minimizedLosstwo{\strategicnode}(\common^\infty,\paramsub{\strategicnode})\\
      &\geq {\mu_1} \norm{\common^\infty-\common^*(\paramsub{\strategicnode})}{2}^2  \geq \mu_1 \varepsilon^2,
      \label{eq:sssd}
   \end{align}
  and thus $\norm{\nabla_\common \minimizedLosstwo{\strategicnode}(\common^\infty,\paramsub{\strategicnode})}{2}\geq \mu_1 \varepsilon$.

  Now since $\gradient{\strategicnode}{\infty} \in \GRADIENT(\common^\infty)$ there exists $\strategicvote \in \setR^d$ such that\footnote{In fact, if $\gradient{\strategicnode}{\infty}$ belongs to the interior of $\GRADIENT(\common^\infty)$, we can guarantee $\nabla_\common \regularization(\common^\infty,\strategicvote) = \gradient{\strategicnode}{\infty}$.} $\norm{\nabla_\common \regularization(\common^\infty,\strategicvote) - \gradient{\strategicnode}{\infty}}{2} \leq \min \set{\varepsilon_2, \frac{\mu_1 \varepsilon}{2} }$ which yields  
     \begin{align}
    \norm{\nabla_\common\minimizedLosstwo{\strategicnode}(\common^\infty,\strategicvote)}{2} &=  \norm{\nabla_\common\minimizedLoss{\strategicnode}(\common^\infty) 
     + \nabla_\common \regularization(\common^\infty,\strategicvote) - \gradient{\strategicnode}{\infty}}{2}\\
     &= \norm{\nabla_\common \regularization(\common^\infty,\strategicvote) - \gradient{\strategicnode}{\infty}}{2} \leq \frac{\mu_1 \varepsilon}{2},
     \end{align}
     which contradicts (\ref{eq:sssd}) if $\norm{\common^\infty-\common^*(\strategicvote)}{2}>\varepsilon$. Therefore, we must have $\norm{\common^\infty - \optcommon{}(\strategicvote, \datafamily{-\strategicnode})}{2} \leq \varepsilon$.

\end{proof}
\color{black}

%% file: GlobalConvergence_proof.tex
\color{black}
\section{Proof of Convergence for the Global Model}
\label{sec:Global_convergence}

In this section, we prove a slightly more general result than Proposition \ref{prop:convergence}.
Namely, instead of working with specific regularizations, we consider a more general class of regularizations, identified by Assumption~\ref{ass:assumption1}.

\begin{lemma}
\label{th:general_gradient_attack_convergence}
  Suppose Assumption~\ref{ass:assumption1} holds true. Assume that $\localloss{\node}$ is convex and $L$-smooth for all users $\node \in [\NODE]$. If $\gradient{\strategicnode}{\iteration}$ converges and if $\learningrate{\iteration} = \learningrate{}$ is a constant small enough, then $\common^\iteration$ will converge too.
\end{lemma}

Note that since $\ell_2^2$ and smooth-$\ell_2$ regularizations satisfy Assumption~\ref{ass:assumption1}, Lemma~\ref{th:general_gradient_attack_convergence} clearly implies  Proposition \ref{prop:convergence}.
We now introduce the key objects of the proof of Lemma~\ref{th:general_gradient_attack_convergence}.

Denote $\gradient{\strategicnode}{\infty}$ the limit of the attack gradients $\gradient{\strategicnode}{\iteration}$. We now define
\begin{align}
  \minimizedLoss{\strategicnode}(\common) 
  &\triangleq \inf_{\paramfamily_{-\strategicnode}} \set{\globalloss{} (\common, \paramfamily{}, \datafamily{}) -
    \localloss{\strategicnode} (\paramsub{\strategicnode}, \data{\strategicnode})
    - \regularization (\common, \paramsub{\strategicnode})}+\common^T\gradient{\strategicnode}{\infty} \\
  &= \inf_{\paramfamily_{-\strategicnode}} \set{ \sum_{\node \neq \strategicnode} \localloss{\node} (\paramsub{\node}, \data{\node}) + \sum_{\node \neq \strategicnode} \regularization(\common, \paramsub{\node}) } + \common^T \gradient{\strategicnode}{\infty},
\end{align}
and prove that $\common^{\iteration}$ will converge to the minimizer of $ \minimizedLoss{\strategicnode}(\common)$. By Lemma \ref{lemma:minimizedLossSmooth}, we know that $\minimizedLoss{\strategicnode}(\common)$ is both locally strongly convex and $L$-smooth.

Now define $\error{\strategicnode}{\iteration} \triangleq \gradient{\strategicnode}{\iteration} - \gradient{\strategicnode}{\infty}$. We then have $\error{\strategicnode}{\iteration} \rightarrow 0$ and  $\nabla \minimizedLoss{\strategicnode}(\common^{\iteration}) $ is the sum of all gradient vectors received from all users assuming the strategic user $s$  sends the vector $g_s^\infty$ in all iterations. 
Thus, at iteration $\iteration$ of the optimization algorithm, we will take one step in the direction $\minimizedgradient{}{\iteration} \triangleq \nabla \minimizedLoss{\strategicnode}(\common^\iteration)+\error{\strategicnode}{\iteration}$, i.e.,
\begin{equation}
    \common^{\iteration + 1} = \common^{\iteration} - \learningrate{\iteration} \minimizedgradient{}{\iteration}.
\end{equation}

We now prove the following lemma that bounds the difference between the function value in two successive iterations.
\begin{lemma}
\label{lemma:sufficinet_decrease}
  If $\minimizedLoss{\strategicnode}(\common)$ is $L$-smooth and $\learningrate{\iteration} \leq 1/L$, we have
  \begin{equation}
    \minimizedLoss{\strategicnode}(\common^{\iteration+1}) -  \minimizedLoss{\strategicnode}(\common^{\iteration}) \leq -\frac{\learningrate{\iteration}}{2}\norm{\minimizedgradient{}{\iteration}}{2}^2 + \learningrate{\iteration} {\error{\strategicnode}{\iteration}}^T\minimizedgradient{}{\iteration}.
  \end{equation}
\end{lemma}
\begin{proof}
   Since $\minimizedLoss{\strategicnode}$ is $L$-smooth, we have
   \begin{equation}
       \minimizedLoss{\strategicnode}(\common^{\iteration+1}) \leq \minimizedLoss{\strategicnode}(\common^{\iteration}) + (\common^{\iteration+1}-\common^{\iteration})^T \nabla \minimizedLoss{\strategicnode}(\common^{\iteration}) + \frac{L}{2}\norm{\common^{\iteration+1}-\common^{\iteration}}{2}^2.
   \end{equation}
   Now plugging $\common^{\iteration+1} - \common^{\iteration} = -\learningrate{\iteration} \minimizedgradient{}{\iteration}$ and $\nabla \minimizedLoss{\strategicnode}(\common^{\iteration}) = \minimizedgradient{}{\iteration} -\error{\strategicnode}{\iteration}$ into the inequality implies
   \begin{align}
       \minimizedLoss{\strategicnode}(\common^{\iteration+1}) - \minimizedLoss{\strategicnode}(\common^{\iteration}) &\leq \left(-\learningrate{\iteration} \minimizedgradient{}{\iteration}\right)^T \left(\minimizedgradient{}{\iteration} -\error{\strategicnode}{\iteration}
       \right) + \frac{L}{2}\norm{-\learningrate{\iteration} \minimizedgradient{}{\iteration}}{2}^2\\
       &\leq -\frac{\learningrate{\iteration}}{2}\norm{\minimizedgradient{}{\iteration}}{2}^2 + \learningrate{\iteration} {\error{\strategicnode}{\iteration}}^T\minimizedgradient{}{\iteration},
   \end{align}
   where we used the fact $\learningrate{\iteration} \leq 1/L$.
\end{proof}

\subsubsection{The global model remains bounded}

\begin{lemma}
  There is $M$ such that, for all $\iteration$, $\minimizedLoss{\strategicnode} (\common^\iteration) \leq M$.
\end{lemma}

\begin{proof}
   Consider the closed ball $\ball(\common^*, 1)$ centered on $\common^*$ and of radius 1. 
   By Lemma \ref{lemma:minimizedLossSmooth}, we know that $\minimizedLoss{\strategicnode}$ is locally strongly convex and thus there exists a $\mu_1>0$ such that $\minimizedLoss{\strategicnode}$ is $\mu_1$-strongly convex on $\ball(\common^*, 1)$. 
   Now consider a point $\common_1$ on the boundary of $\ball(\common^*, 1)$. By strong convexity we have
   \begin{equation}
   \label{eq:large_gradient}
     \norm{\nabla \minimizedLoss{\strategicnode}(\common_1)}{2}^2\geq(\common_1-\common^*)^T\nabla \minimizedLoss{\strategicnode}(\common_1)\geq {\mu_1} \norm{\common_1-\common^*}{2}^2 = {\mu_1}.
   \end{equation}
   Now similarly, by the convexity of $\minimizedLoss{\strategicnode}$ on $\setR^d$, for any $\common \in \setR^d-\ball(\common^*, 1)$,  we have $ \norm{\nabla \minimizedLoss{\strategicnode}(\common_1)}{2}\geq\sqrt{{\mu_1}}$. Now since $\error{\strategicnode}{\iteration} \rightarrow 0$, there exists an iteration $T_1$ after which ($\iteration\geq T_1$), we have $\norm{\error{\strategicnode}{\iteration}}{2} \leq \frac{1}{4} \sqrt{{\mu_1}}$, and thus $\norm{\minimizedgradient{}{\iteration}}{2} \geq \norm{ \nabla \minimizedLoss{\strategicnode}(\common^\iteration)}{2}-\norm{\error{\strategicnode}{\iteration}}{2} \geq \frac{3}{4} \sqrt{{\mu_1}}$.
   Thus, Lemma \ref{lemma:sufficinet_decrease} implies that for $\iteration\geq T_1$, if $\norm{\common^\iteration - \common^*}{2}\geq 1$, then
   \begin{align}
     \minimizedLoss{\strategicnode}(\common^{\iteration+1}) - \minimizedLoss{\strategicnode}(\common^{\iteration}) 
     &\leq -\frac{\learningrate{}}{2}\norm{\minimizedgradient{}{\iteration}}{2}^2 + \learningrate{} {\error{\strategicnode}{\iteration}}^T\minimizedgradient{}{\iteration}\\ 
     &\leq -\frac{\learningrate{}}{2}\norm{\minimizedgradient{}{\iteration}}{2}^2 + \learningrate{}\norm{\error{\strategicnode}{\iteration}}{2} \norm{\minimizedgradient{}{\iteration}}{2}\\
     &\leq -\frac{\learningrate{}}{2}\norm{\minimizedgradient{}{\iteration}}{2}\left(\norm{\minimizedgradient{}{\iteration}}{2}- 2\norm{\error{\strategicnode}{\iteration}}{2}\right) \\
     &\leq - \frac{\learningrate{}}{2} \frac{3}{4} \sqrt{\mu_1} \left( \frac{3}{4} \sqrt{\mu_1} - \frac{2}{4} \sqrt{\mu_1} \right) 
     \leq - \frac{3 \learningrate{}}{32} \mu_1 < 0.
   \end{align}
   Thus, for $\norm{\common^\iteration - \common^*}{2}\geq 1$, the loss cannot increase at the next iteration.
   
   Now consider the case $\norm{\common^\iteration - \common^*}{2} < 1$ for $\iteration \geq T_1$. 
   The smoothness of $\minimizedLoss{\strategicnode}$ implies $\norm{\nabla\minimizedLoss{\strategicnode}(\common^{\iteration})}{2} < L$. Therefore, 
   \begin{align}
      \norm{\common^{\iteration+1} - \common^*}{2} &= \norm{\common^{\iteration} - \learningrate{}(\nabla \minimizedLoss{\strategicnode}(\common^\iteration)+\error{\strategicnode}{\iteration}) - \common^*}{2} \\ &\leq \norm{\common^{\iteration+1} - \common^*}{2} 
      + \learningrate{} (L+\frac{1}{4} \sqrt{{\mu_1}}) \leq 1+ \learningrate{} (L+\frac{1}{4} \sqrt{{\mu_1}}).
   \end{align}
   Now we define $M_1 \triangleq \max_{\common \in \ball \left(\common^*, 1+ \learningrate{} (L+\frac{1}{4} \sqrt{{\mu_1}}) \right)} \minimizedLoss{\strategicnode}(\common)$, the maximum function value in the closed ball $\ball \left(\common^*, 1+ \learningrate{} (L+\frac{1}{4} \sqrt{{\mu_1}}) \right)$. 
   Therefore, we have $\minimizedLoss{\strategicnode}(\common^{\iteration +1})\leq M_1$. So far we proved that for $\iteration \geq T_1$, in each iteration of gradient descent either the function value will not increase or it will be upper-bounded by $M_1$. This implies that for all $\iteration$, the function value $\minimizedLoss{\strategicnode}(\common^{\iteration})$ is upper-bounded by
   \begin{equation}
       M \triangleq \max \set{\max_{\iteration \leq T_1}\set{ \minimizedLoss{\strategicnode}(\common^{\iteration})},M_1}.
   \end{equation}
   This concludes the proof.
\end{proof}

\begin{lemma}
\label{lemma:compact_set_convergence}
  There is a compact set $X$ such that, for all $\iteration$, $\common^\iteration \in X$.
\end{lemma}

\begin{proof}
   Now since $\minimizedLoss{\strategicnode}$ is $\mu_1$-strongly convex in $\ball(\common^*, 1)$, for any point $\common \in \setR^d$ such that $\norm{\common-\common^{\iteration}}{2} = 1$, we have
  \begin{equation}
    \minimizedLoss{\strategicnode}(\common) \geq \minimizedLoss{\strategicnode}(\common^*) + \frac{\mu_1}{2} \norm{\common-\common^*}{2}^2 = \minimizedLoss{\strategicnode}(\common^*) + \frac{\mu_1}{2}.
  \end{equation}
  But now by the convexity of $\minimizedLoss{\strategicnode}$ in $\setR^d$, for any $\common$ such that $\norm{\common-\common^*}{2} \geq 1$, we have
  \begin{equation}
     \minimizedLoss{\strategicnode}(\common)\geq \minimizedLoss{\strategicnode}(\common^*) + \norm{\common-\common^*}{2} \frac{\mu_1}{2}.
  \end{equation}
  This implies that if $\norm{\common^\iteration-\common^*}{2}>\frac{2}{\mu_1}\left( M_2-\minimizedLoss{\strategicnode}(\common^*) \right)$, then $\minimizedLoss{\strategicnode}(\common^\iteration)>M_2$. Therefore, we must have $\norm{\common^\iteration-\common^*}{2}\leq\frac{2}{\mu_1}\left( M_2-\minimizedLoss{\strategicnode}(\common^*) \right)$, for all $\iteration\geq0$.
  This describes a closed ball, which is a compact set.
\end{proof}

\subsubsection{Convergence of the global model under converging gradient attack}

\begin{lemma}
\label{lemma:sequence_upper_bound_converges}
Suppose $u_t \geq 0$ verifies $u_{t+1} \leq \alpha u_t + \delta_t$, with $\delta_t \rightarrow 0$.
Then $u_t \rightarrow 0$.
\end{lemma}

\begin{proof}
   We now show that for any $\varepsilon>0$, there exists an iteration $T(\varepsilon)$, such that for $t\geq T(\varepsilon)$, we have $u_t\leq\varepsilon$. For this, note that by induction, we observe that, for all $t \geq 0$,
   \begin{equation}
    u_{t+1} \leq u_0 \alpha^{t+1}+ \sum_{\tau = 0}^{t} \alpha^{\tau} \delta_{t-\tau}.
  \end{equation}
  Since $\delta_t \rightarrow 0$, there exists an iteration $T_2(\varepsilon)$ such that for all $t\geq T_2(\varepsilon)$, we have $\delta_t \leq \frac{\varepsilon (1-\alpha)}{2}$. Therefore, for $t\geq T_2(\varepsilon)$, we have 
  \begin{align}
       u_{t+1} &\leq u_0 \alpha^{t+1}+ \sum_{\tau = 0}^{t-T_2(\varepsilon)}   \alpha^{\tau} \delta_{t-\tau}+ \sum_{\tau = t-T_2(\varepsilon)+1}^{t}   \alpha^{\tau} \delta_{t-\tau} \\
       &\leq u_0 \alpha^{t+1}+ \frac{\varepsilon (1-\alpha)}{2} \sum_{\tau = 0}^{t-T_2(\varepsilon)}  \alpha^{\tau} + \sum_{s = 0}^{T_2(\varepsilon)-1}   \alpha^{t-s} \delta_{s}\\
       &\leq \left(u_0+ \sum_{s = 0}^{T_2(\varepsilon)-1}   \alpha^{-s-1} \delta_{s} \right) \alpha^{t+1} + \frac{\varepsilon (1-\alpha)}{2}\sum_{\tau = 0}^{\infty}   \alpha^{\tau}.
  \end{align}
  Denoting $M_0(\varepsilon) \triangleq  \sum_{s = 0}^{T_2(\varepsilon)-1}   \alpha^{-s-1} \delta_{s}$, we then have
  \begin{equation}
     u_{t+1} \leq\left(u_0 +M_0(\varepsilon) \right) \alpha^{t+1} + \frac{\varepsilon}{2}.
  \end{equation}
  Therefore, for $t\geq \frac{\ln{\frac{\varepsilon}{2(u_0 +M_0(\varepsilon))}}}{\ln \alpha}$, we have 
  \begin{equation}
     u_{t+1} \leq \frac{\varepsilon}{2} + \frac{\varepsilon}{2} = \varepsilon.
  \end{equation}
  This proves that $u_t \rightarrow 0$.
\end{proof}

We now prove Lemma~\ref{th:general_gradient_attack_convergence} (and hence  Proposition \ref{prop:convergence}).

\begin{proof}[Proof of Lemma \ref{th:general_gradient_attack_convergence}]
      Define $X$ based on Lemma~\ref{lemma:compact_set_convergence}.
   Since $\minimizedLoss{\strategicnode}$ is locally strongly convex, there exists $\mu_2>0$ such that $\minimizedLoss{\strategicnode}$ is $\mu_2$-strongly convex in a convex compact set $X$ containing $\common^\iteration$ for all $\iteration \geq 0$.
   By the strong convexity of $\minimizedLoss{\strategicnode}(\common)$, we have
   \begin{align}
      \minimizedLoss{\strategicnode}(\common^{\iteration}) - \minimizedLoss{\strategicnode}(\common^*) &\leq (\common^{\iteration}-\common^{*})^T \nabla \minimizedLoss{\strategicnode}(\common^{\iteration}) - \frac{\mu_2}{2} \norm{\common^{\iteration}-\common^{*}}{2}^2\\
      &= (\common^{\iteration}-\common^{*})^T \left(\minimizedgradient{}{\iteration} -\error{\strategicnode}{\iteration}
       \right) - \frac{\mu_2}{2} \norm{\common^{\iteration}-\common^{*}}{2}^2.
   \end{align}
   
   Now, using the fact 
   \begin{align}
       (\common^{\iteration}-\common^{*})^T \minimizedgradient{}{\iteration} 
       &= \frac{1}{\learningrate{}}(\common^{\iteration}-\common^{*})^T(\common^{\iteration}-\common^{\iteration+1}) \\
       &= \frac{1}{2\learningrate{}}  \left( \norm{\common^\iteration-\common^*}{2}^2 +\norm{\common^\iteration-\common^{\iteration+1}}{2}^2 - \norm{\common^{\iteration+1}-\common^*}{2}^2\right)\\
       &= \frac{1}{2\learningrate{}}  \left( \learningrate{}^2 \norm{\minimizedgradient{}{\iteration}}{2}^2 +\norm{\common^\iteration-\common^{*}}{2}^2 - \norm{\common^{\iteration+1}-\common^*}{2}^2\right)\\
       &= \frac{\learningrate{}}{2} \norm{\minimizedgradient{}{\iteration}}{2}^2 + \frac{1}{2\learningrate{}} \left( \norm{\common^{\iteration}-\common^{*}}{2}^2 - \norm{\common^{\iteration+1}-\common^{*}}{2}^2 \right),
   \end{align}
    we  have
   \begin{align}
       &\minimizedLoss{\strategicnode}(\common^{\iteration}) - \minimizedLoss{\strategicnode}(\common^*) \leq\\ &\frac{\learningrate{}}{2} \norm{\minimizedgradient{}{\iteration}}{2}^2 + \frac{1}{2\learningrate{}} \left( \norm{\common^{\iteration}-\common^{*}}{2}^2 - \norm{\common^{\iteration+1}-\common^{*}}{2}^2 \right) - (\common^{\iteration}-\common^{*})^T \error{\strategicnode}{\iteration} - \frac{\mu_2}{2} \norm{\common^{\iteration}-\common^{*}}{2}^2
       \label{eq:temp}.
   \end{align}
   But now note that $\minimizedLoss{\strategicnode}(\common^{\iteration}) -  \minimizedLoss{\strategicnode}(\common^{*})\geq \minimizedLoss{\strategicnode}(\common^{\iteration}) -  \minimizedLoss{\strategicnode}(\common^{\iteration+1})$. Thus, combining Equation (\ref{eq:temp}) and Lemma \ref{lemma:sufficinet_decrease} yields
   \begin{equation}
       -\learningrate{}{\error{\strategicnode}{\iteration}}^T\minimizedgradient{}{\iteration} \leq \frac{1}{2\learningrate{}} \left( \norm{\common^{\iteration}-\common^{*}}{2}^2 - \norm{\common^{\iteration+1}-\common^{*}}{2}^2 \right) - (\common^{\iteration}-\common^{*})^T \error{\strategicnode}{\iteration} - \frac{\mu_2}{2} \norm{\common^{\iteration}-\common^{*}}{2}^2.
   \end{equation}
   By rearranging the terms, we then have
   \begin{align}
     \norm{\common^{\iteration+1}-\common^{*}}{2}^2 &\leq (1-\mu_2\learningrate{}) \norm{\common^{\iteration}-\common^{*}}{2}^2 - \learningrate{} \left(\common^{\iteration+1}-\common^{*} \right)^T \error{\strategicnode}{\iteration}\\
     &\leq (1-\mu_2\learningrate{}) \norm{\common^{\iteration}-\common^{*}}{2}^2 + \learningrate{} \norm{\common^{\iteration+1}-\common^{*}}{2} \norm{\error{\strategicnode}{\iteration}}{2}.
     \label{eq:sequence}
   \end{align}
   Now note that $ \learningrate{} \leq 1/L < 1/\mu_2$ and thus $0<1-\mu_2\learningrate{}<1$. We now define two sequences $u_t \triangleq \norm{\common^{\iteration}-\common^{*}}{2}$ and $\delta_{\iteration} = \learningrate{}\norm{\error{\strategicnode}{\iteration}}{2}$. We already know that $\delta_t \rightarrow 0$, and we want to show $u_t$ also converges to $0$. By Equation (\ref{eq:sequence}), we have
   \begin{equation}
       u_{t+1}^2 \leq (1-\learningrate{}\mu_2) u_{t}^2 + \delta_{\iteration} u_{t+1},
   \end{equation}
   which implies
   \begin{equation}
    \left(u_{t+1} - \frac{\delta_t}{2} \right)^2   = u_{t+1}^2 - u_{t+1} \delta_t + \frac{\delta_t^2}{4} \leq (1-\learningrate{}\mu_2) u_{t}^2 + \frac{\delta_t^2}{4},
   \end{equation}
   and thus
   \begin{equation}
     u_{t+1} \leq \sqrt{(1-\learningrate{}\mu_2) u_{t}^2 + \frac{\delta_t^2}{4}} + \frac{\delta_t}{2}\leq \sqrt{(1-\learningrate{}\mu_2) u_{t}^2} +\frac{\delta_t}{2}+\frac{\delta_t}{2} \leq \left(1-\frac{\learningrate{}\mu_2}{2} \right) u_t +\delta_t.
   \end{equation}
   Lemma~\ref{lemma:sequence_upper_bound_converges} allows to conclude.
\end{proof}

\color{black}

%% file: vulnerability_proof.tex
\section{Proofs of the Impossibility Corollaries}
\label{sec:impossibility_corollary}

\subsection{Lower Bound on Byzantine Resilience}

\begin{proof}[Proof of Corollary~\ref{cor:impossibility_resilience}]
Assume $\BYZANTINE \geq \NODE /2$, and consider $d=1$.
Denote $H_0 \triangleq \lfloor \NODE / 2 \rfloor$.
Let us define $\trueparamsub{\node} \triangleq -1$ for all users $\node \in [H_0] = \set{1, \ldots, H_0}$,
$\trueparamsub{\node} \triangleq 1$ for all users $\node \in [2 H_0] - [H_0] = \set{H_0 + 1, \ldots, 2 H_0}$
and $\trueparamsub{\node} \triangleq 0$ for all users $\node \in [\NODE] - [2 H_0]$ (which is either empty or contains one element).
Now fix $\varepsilon, \delta > 0$, with $\varepsilon \triangleq 1/4$ and $\delta \triangleq 1/3$.
Consider the honest datasets $\datafamily{}$ of size $\NODEINPUT{}$ that they may have reported, 
where $\NODEINPUT{}$ is chosen to guarantee high-probability $(\BYZANTINE, \NODE, \averagingconstant)$-Byzantine learning, as guaranteed by Definition~\ref{def:Byzantine-learning}. 
Since the guarantee must hold for $\HONEST \subseteq [H_0]$ and for $\HONEST \subseteq [2 H_0] - [H_0]$,
with probability at least $1-2\delta \geq 1/3 > 0$ (so that both guarantees hold),
we must then have $\absv{\common^\alg - (-1)}^2 \leq \varepsilon$ (for $\HONEST \subseteq [H_0]$) 
and $\absv{\common^\alg - 1}^2 \leq \varepsilon$ 
(for $\HONEST \subseteq [2 H_0] - [H_0]$).
But then, by the triangle inequality, we must have
\begin{equation}
    2 = \absv{1 - \common^\alg + \common^\alg - (-1)}
    \leq \absv{1 - \common^\alg} + \absv{\common^\alg - (-1)}
    \leq \sqrt{\varepsilon} + \sqrt{\varepsilon}
    = 1/2 + 1/2 = 1.
\end{equation}
This is a contradiction. 
Thus $(\BYZANTINE, \NODE, \averagingconstant)$-Byzantine learning cannot be guaranteed for $\BYZANTINE \geq \NODE /2$.
\end{proof}

\subsection{Lower Bound on Correctness}

\begin{proof}[Proof of Corollary~\ref{cor:impossibility}]
Consider $d=1$.
Let us define $\trueparamsub{\node} \triangleq 0$ for all users $\node \in [\card{\HONEST}] = \set{1, \ldots, \card{\HONEST}}$,
and $\trueparamsub{\node} \triangleq 1$ for all users $\node \in [\NODE] - [\card{\HONEST}] = \set{\card{\HONEST} + 1, \ldots, \NODE}$.
Now fix $\varepsilon, \delta > 0$, with $\varepsilon < \BYZANTINE^2 / (\NODE - \BYZANTINE)^2$ and $\delta \triangleq 1/3$.
Consider the honest datasets $\datafamily{}$ of size $\NODEINPUT{}$ that they may have reported, 
where $\NODEINPUT{}$ is chosen to guarantee high-probability $(\BYZANTINE, \NODE, \averagingconstant)$-Byzantine learning, as guaranteed by Definition~\ref{def:Byzantine-learning}. 
Since the guarantee must hold for $\HONEST = [\card{\HONEST}]$ and for $\HONEST = [\NODE] - [\BYZANTINE]$,
with probability at least $1-2\delta \geq 1/3 > 0$ (so that both guarantees hold),
we must then have $\absv{\common^\alg}^2 \leq \varepsilon$ (for $\HONEST = [\card{\HONEST}]$) and
\begin{equation}
    \absv{\common^\alg - \frac{\BYZANTINE}{\NODE - \BYZANTINE}}^2
    \leq C^2 + \varepsilon,
\end{equation}
for the case $\HONEST = [\NODE] - [\BYZANTINE]$.
The first inequality implies $\absv{\common^\alg} \leq \BYZANTINE / (\NODE - \BYZANTINE)$,
while the second can then be rewritten
\begin{align}
    C^2 \geq \left( \frac{\BYZANTINE}{\NODE - \BYZANTINE} - \varepsilon \right)^2 - \varepsilon.
\end{align}
But this equation is now deterministic. 
Since it must hold with a strictly positive probability, it must thus hold deterministically.
Moreover, it holds for any $\varepsilon >0$.
Taking the limit $\varepsilon \rightarrow 0$ yields the result.
\end{proof}

%% file: convergence_sum_over_average.tex
\section{Sum over Expectations}
\label{app:sum_average}

In this section, we provide both theoretical and empirical results to argue for using a sum-based local loss over an expectation-based local loss.

\subsection{Theoretical Arguments}

Indeed, intuitively, if one considers an expectation $\expectVariable{\datapoint \sim \data{\node}}{\lossperinput (\paramsub{\node}, \datapoint)}$ rather than a sum, as is done by \cite{HanzelyHHR20}, \cite{DinhTN20} and \cite{collaborative_learning}, then the weight of an honest active user's local loss will not increase as a user provides more and more data, which will hinder the ability of $\paramsub{\node}$ to fit the user's local data.
In fact, intuitively, using an expectation wrongly yields the same influence to any two users, even when one (honest) user provides a much larger dataset $\data{\node}$ than the other, and should thus intuitively be regarded as ``more reliable''.

There is another theoretical argument for using the sum rather than the expectation.
Namely, if the loss is regarded as a Bayesian negative log-posterior, given a prior $\exp \left( - \sum_{\node \in [\NODE]} \reglocalweight{} \norm{\paramsub{\node}}{2} - \sum_{\node \in [\NODE]} \regularization(\common, \paramsub{\node}) \right)$ on the local and global models, 
then the term that fits local data should equal the negative log-likelihood of the data, given the models $(\common, \paramfamily{})$.
Assuming that the distribution of each data point $\datapoint \in \data{\node}$ is independent from all other data points, and depends only on the local model $\paramsub{\node}$, this negative log-likelihood yields a sum over data points; not an expectation.

\subsection{Empirical Results}

We also empirically compared the performances of sum as opposed to the expectation.
To do so, we constructed a setting where 10 ``idle'' users draw randomly 10 data points from the FashionMNIST dataset, while one ``active'' user has all of the FashionMNIST dataset (60,000 data points).
We then learned local and global models, with $\regularization(\common, \param) \triangleq  \regweightsub{} \norm{\common - \param}{2}^2$, $\lambda = 1$.
We compared two different classifiers to which we refer as a ``linear model'' and ``2-layers neural network'', both using \textit{CrossEntropy} loss.
The linear model has $(784 + 1)  \times 10$ parameters.
The neural network has 2 layers of 784 parameters with bias, with \textit{ReLU} activation in between, adding up to ($(784 + 1)  \times 784 + (784 + 1)  \times 10$.

Note also that, in all our experiments, we did not consider any local regularization, i.e. we set $\reglocalweight{} \triangleq 0$. All our experiments are seeded with seed 999.

\subsubsection{Noisy FashionMNIST}

To see a strong difference between sum and average, 
we made the FashionMNIST dataset harder to learn, by randomly labeling 60\% of the training set. 
Table~\ref{tab:sum_average} reports the accuracy of local and global models in the different settings.
Our results clearly and robustly indicate that the use of sums outperforms the use of expectations.

\begin{table}
    \centering
    \begin{tabular}{||c||c|c||c|c||}
        \hline \hline
        & $\mathbb E L$ & $\Sigma L$ & $\mathbb E NN$ & $\Sigma NN$  \\ \hline
        idle user's model & 0.52 & 0.80 & 0.55 & 0.79 \\ \hline
        active user's model & 0.58 & 0.80 & 0.56 & 0.79 \\ \hline
        global model & 0.55 & 0.80 & 0.58 & 0.79 \\ \hline \hline
    \end{tabular}
    \caption{Accuracy of trained models, depending on the use of expectation (denoted $\mathbb E$) or sum ($\Sigma$), and on the use of linear classifier ($L$) or a 2-layer neural net $(NN)$. Here, all users are honest and an $\ell_2^2$ regularization is used, but there is a large heterogeneity in the amount of data per user.}
    \label{tab:sum_average}
\end{table}

On each of the following plots, we display the top-1 accuracy on the MNIST test dataset (10 000 images) for the active user, for the global model and for one of the idle users (in Table \ref{tab:sum_average}, the mean accuracy for idle users is reported), as we vary the value of $\regweightsub{}$.
Intuitively, $\regweightsub{}$ models how much we want the local models to be similar.

In the case of learning FashionMNIST, given that the data is i.i.d., larger values of $\regweightsub{}$ are more meaningful (though our experiments show that they may hinder convergence speed).
However, in settings where users have different data distributions, e.g. because the labels depend on users' preferences, then smaller values of $\regweightsub{}$ may be more relevant.

Note that the use of a common value of $\regweightsub{}$ in both cases is slightly misleading, as using the sum intuitively decreases the comparative weight of the regularization term. To reduce this effect, for this experiment only, we divide the local losses by the average of the number of data points per user for the sum version. This way, if the number of points is equal for all users, the two losses will be exactly the same.
More importantly, our experiments seem to robustly show that using the sum consistently outperforms the expectation, for both a linear classifier and a 2-layer neural network, for the problem of noisy FashionMNIST classification.

\begin{figure}
    \centering
    \subfloat[Using the average] 
    {\includegraphics[scale=0.4]{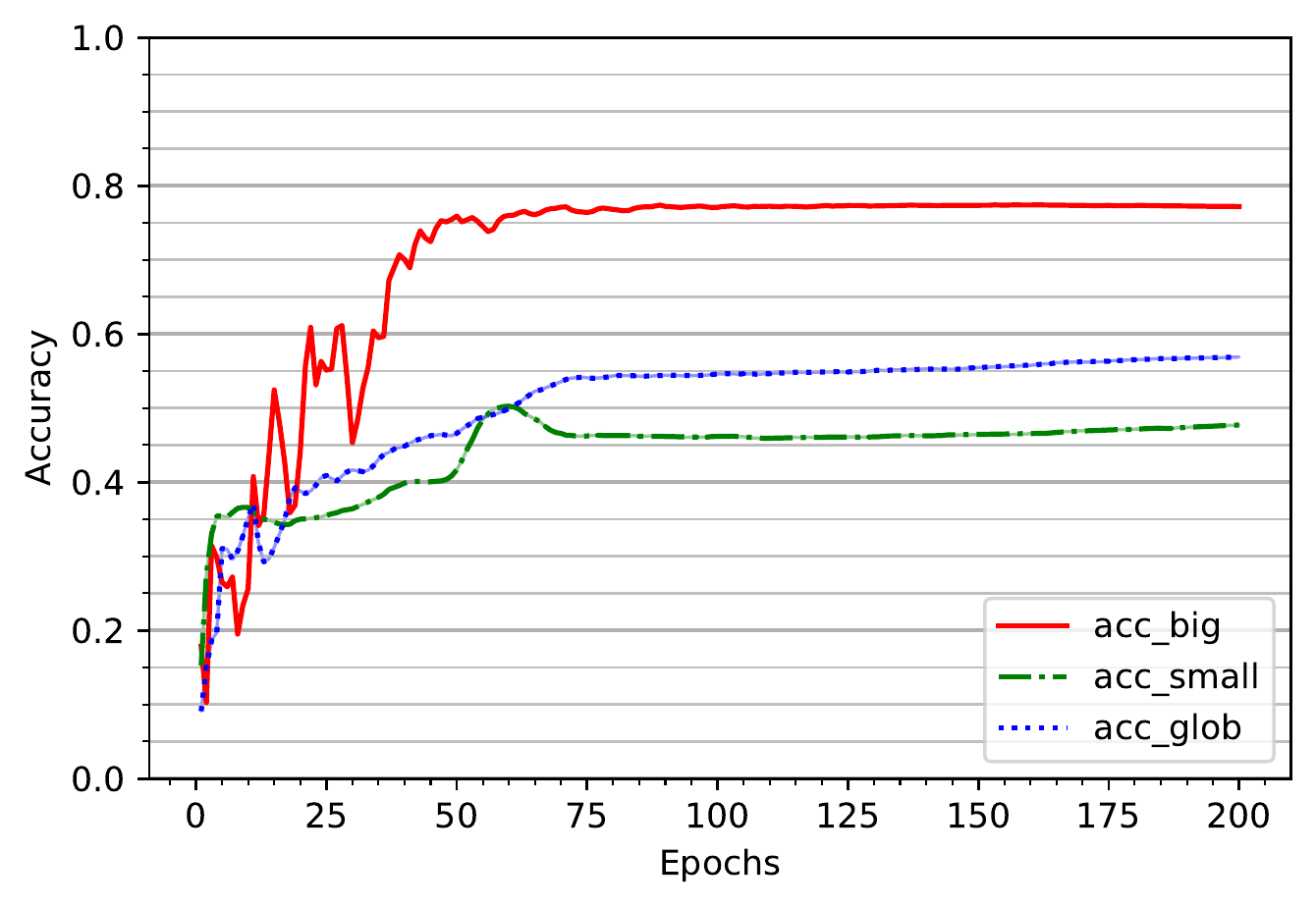}}
      \qquad
    \subfloat[Using the sum] {\includegraphics[scale=0.4]{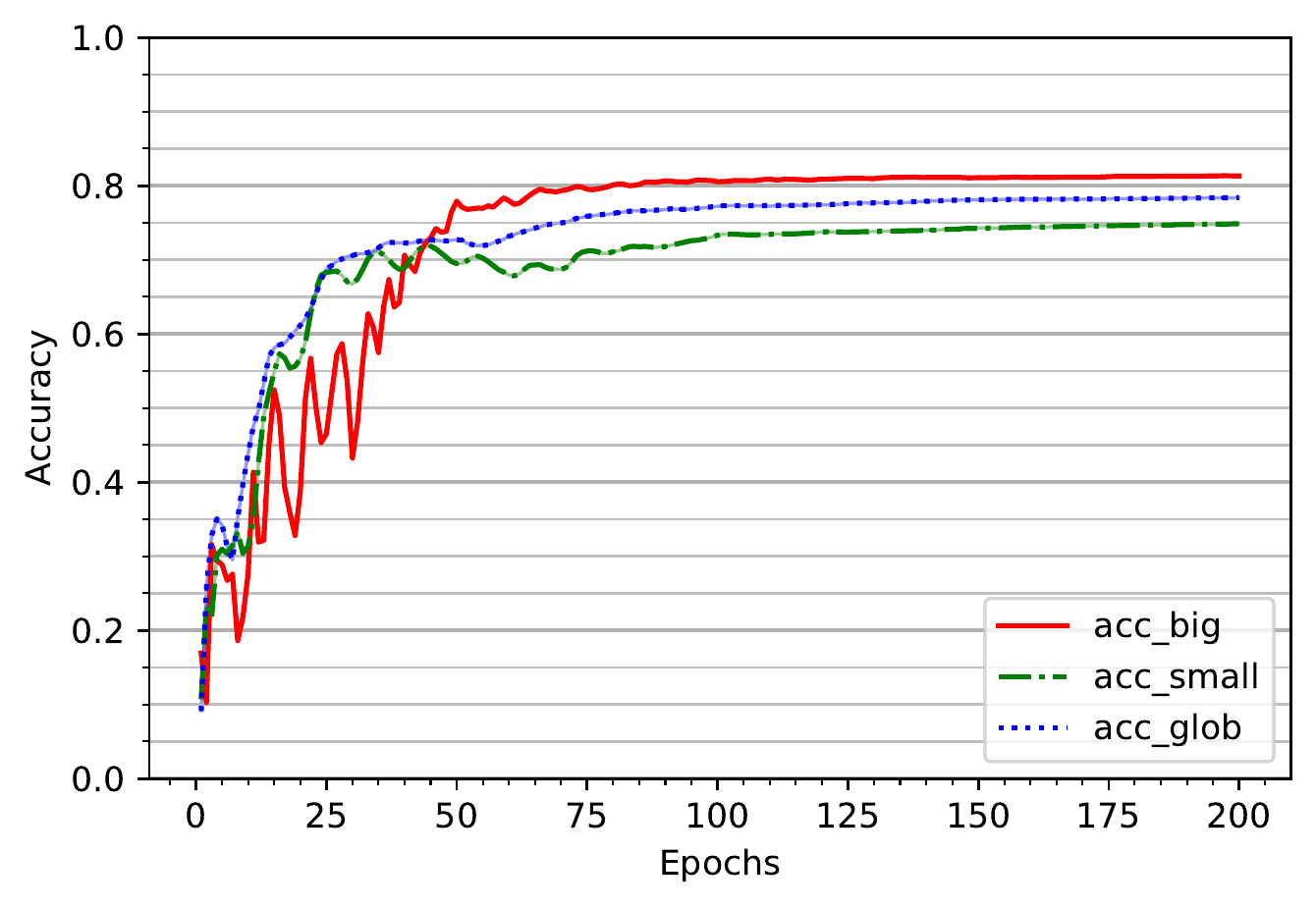}}
      \qquad
     \caption{Linear model on noisy FashionMNIST, for $\lambda = 0.01$.}
\end{figure}

\begin{figure}
    \centering
    \subfloat[Using the average] 
    {\includegraphics[scale=0.4]{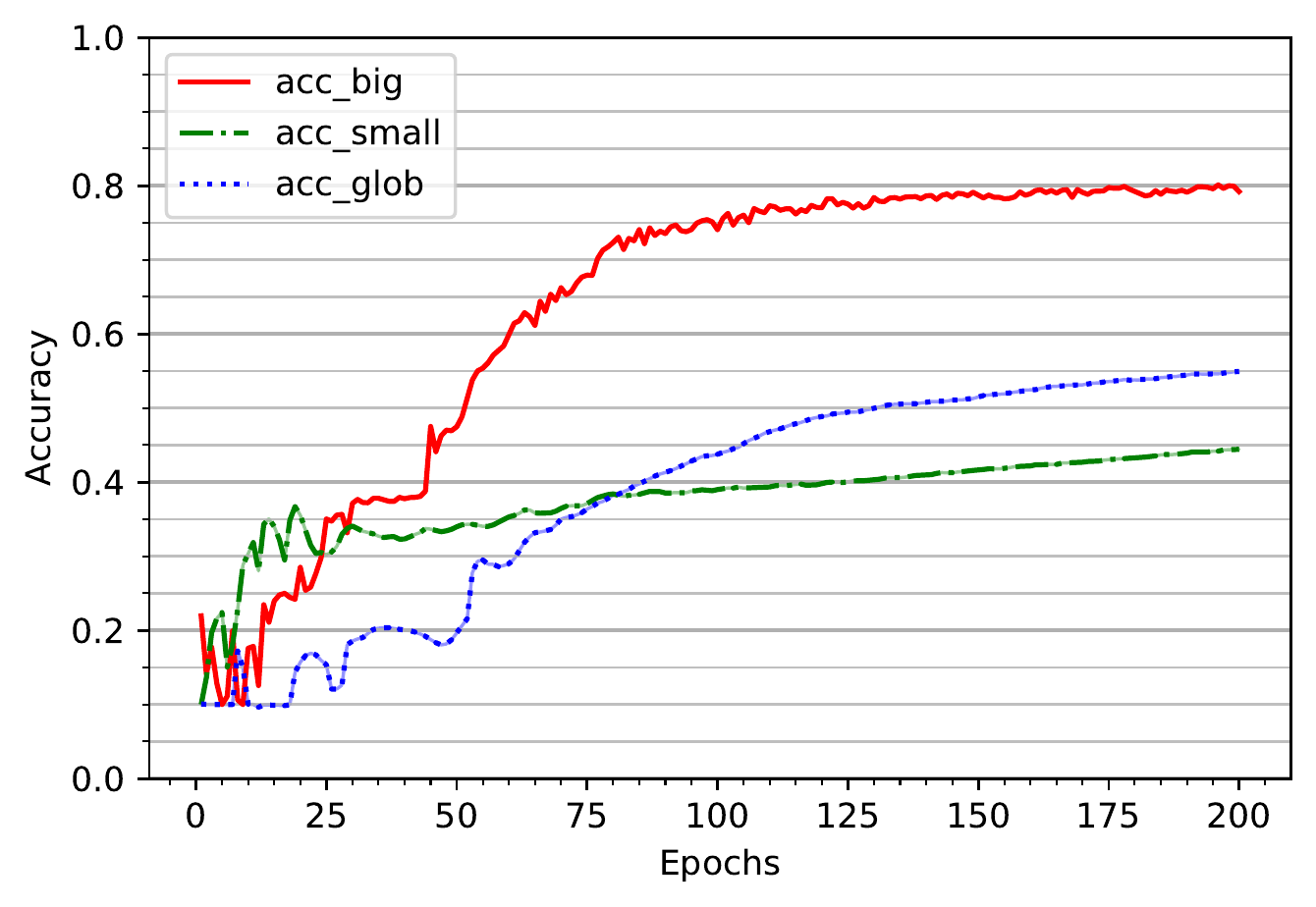}}
      \qquad
    \subfloat[Using the sum] {\includegraphics[scale=0.4]{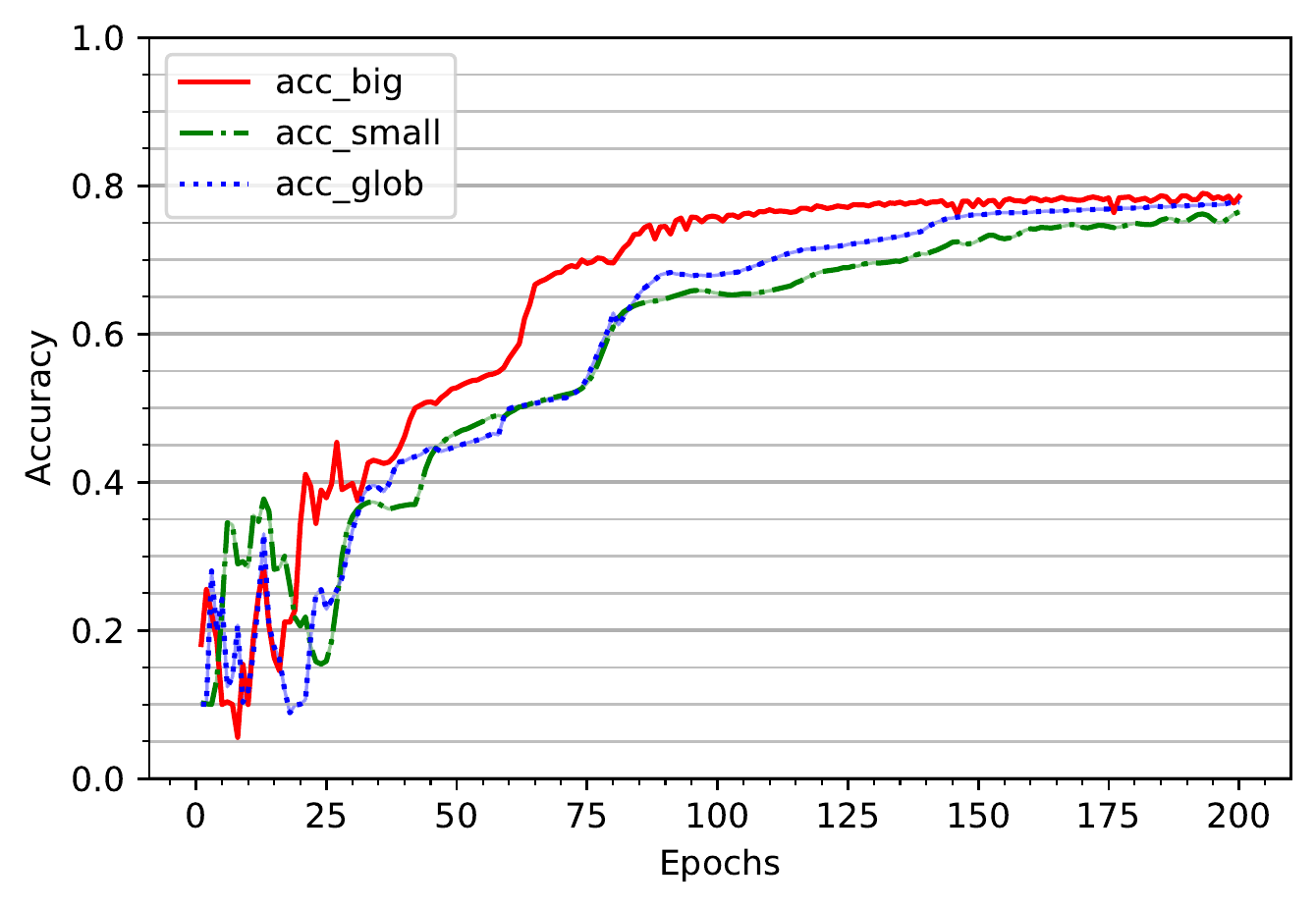}}
      \qquad
     \caption{2-layer neural network on noisy FashionMNIST, for $\lambda = 0.01$.}
\end{figure}

\begin{figure}
    \centering
    \subfloat[Using the average] 
    {\includegraphics[scale=0.4]{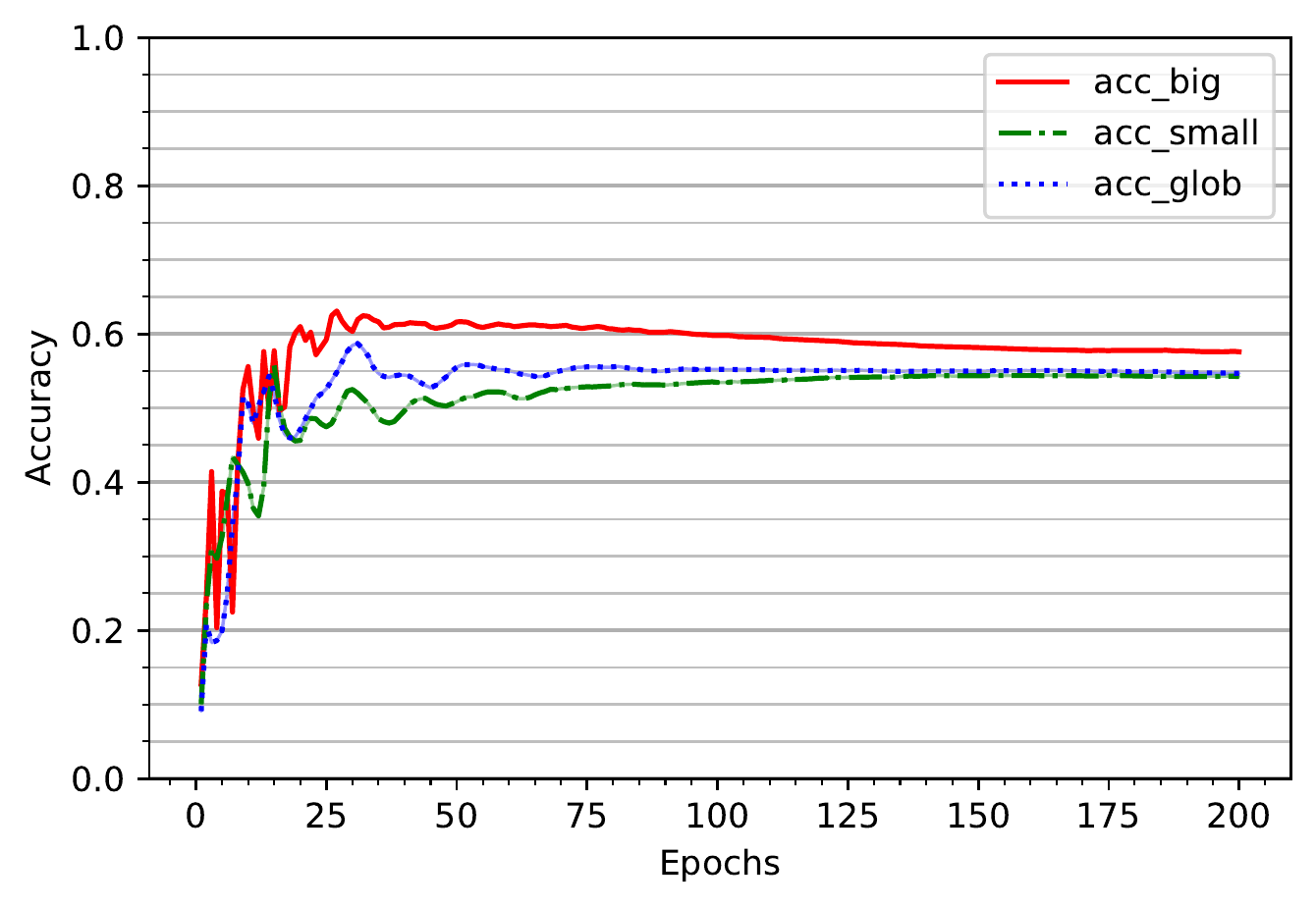}}
      \qquad
    \subfloat[Using the sum] {\includegraphics[scale=0.4]{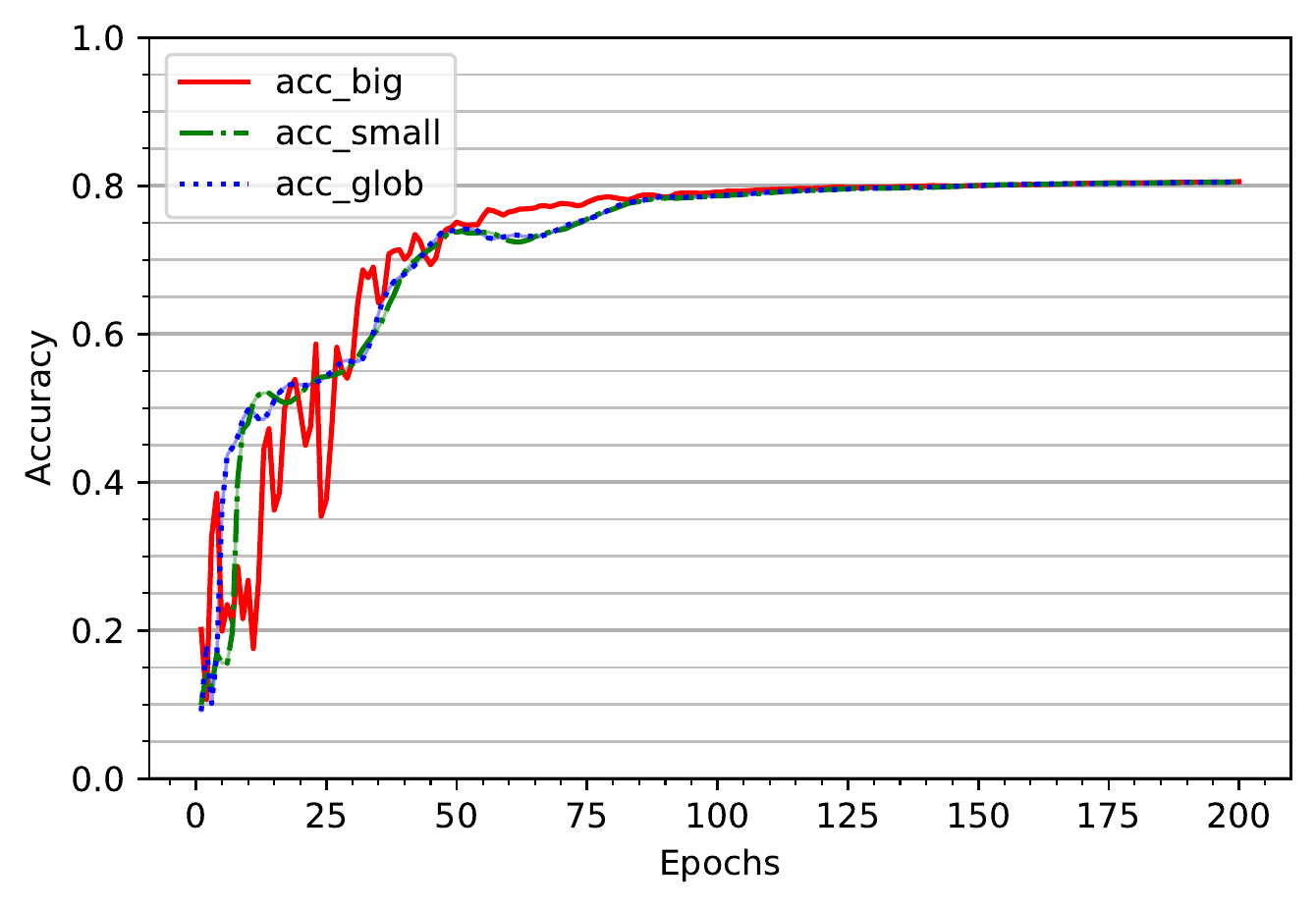}}
      \qquad
     \caption{Linear model on noisy FashionMNIST, for $\lambda = 0.1$.}
\end{figure}

\begin{figure}
    \centering
    \subfloat[Using the average] 
    {\includegraphics[scale=0.4]{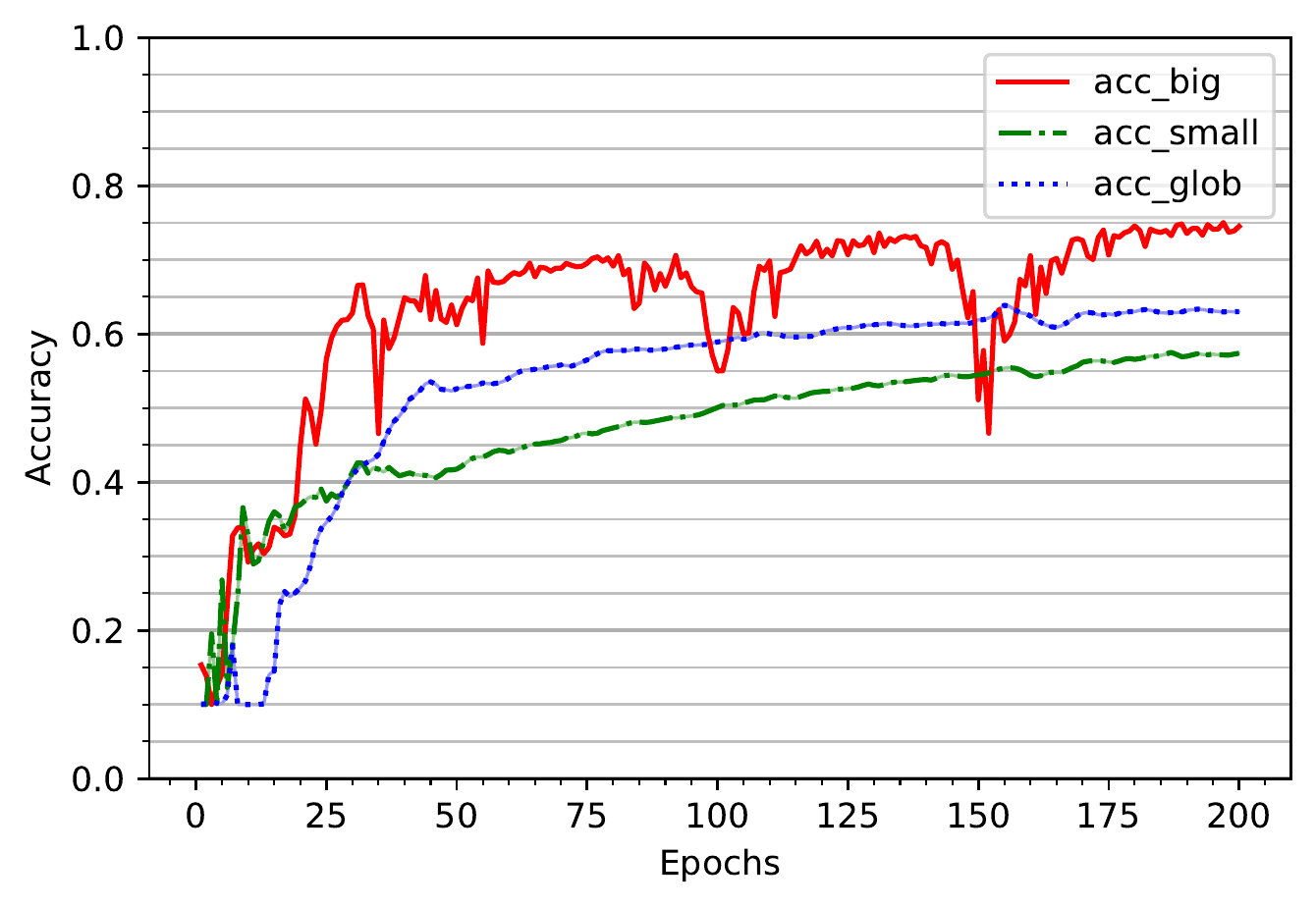}}
      \qquad
    \subfloat[Using the sum] {\includegraphics[scale=0.4]{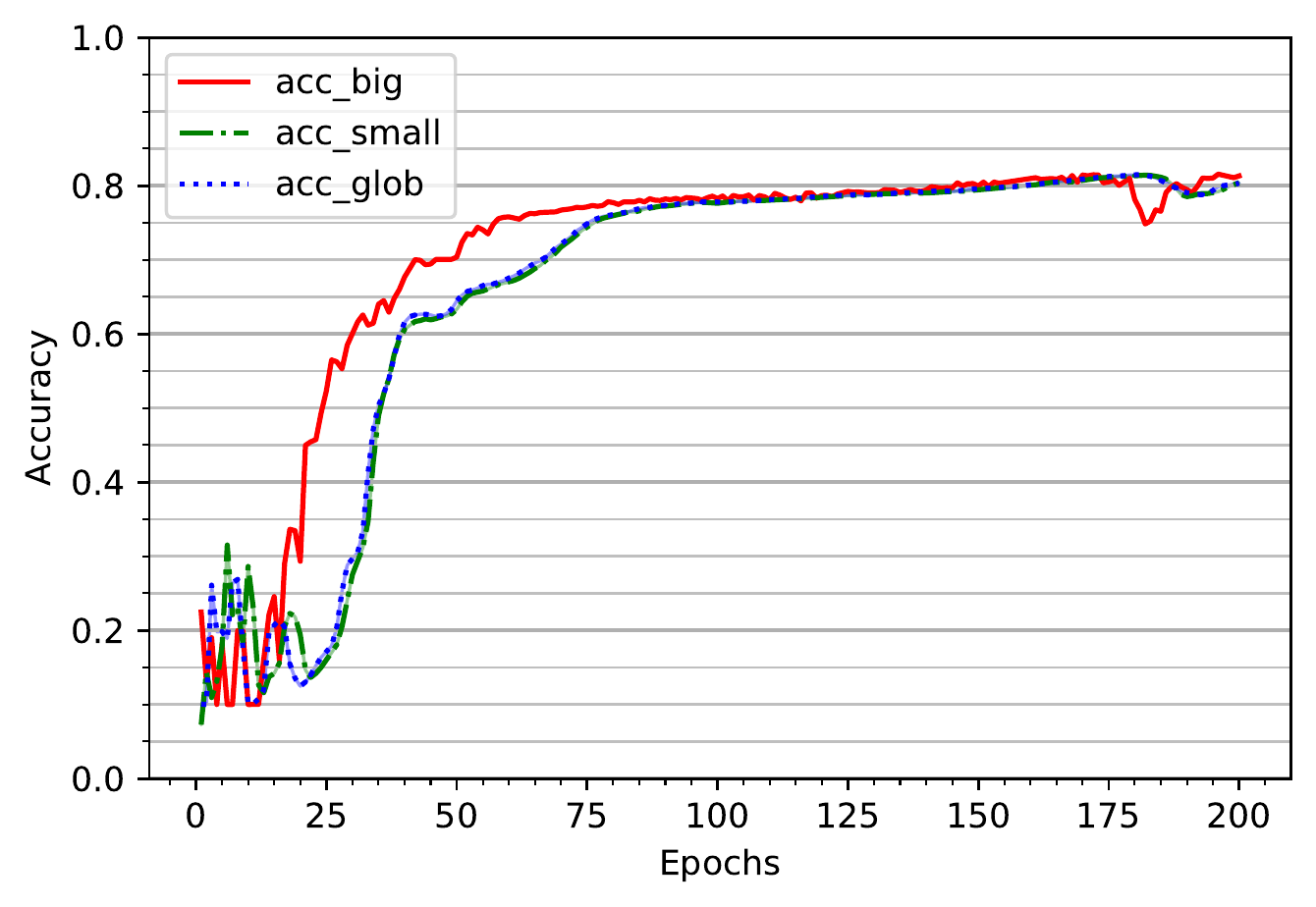}}
      \qquad
     \caption{2-layer neural network on noisy FashionMNIST, for $\lambda = 0.1$.}
\end{figure}

\begin{figure}
    \centering
    \subfloat[Using the average] 
    {\includegraphics[scale=0.4]{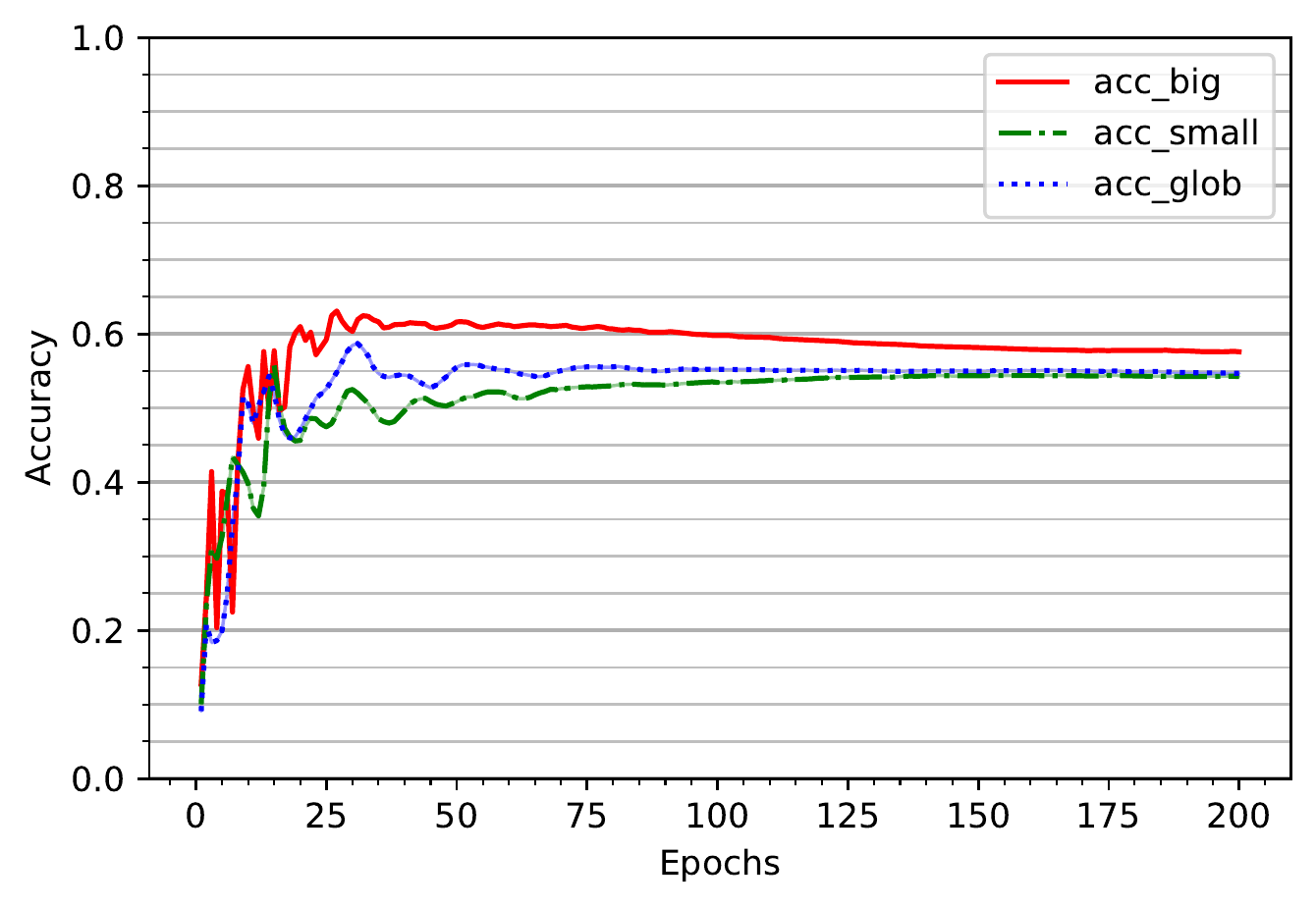}}
      \qquad
    \subfloat[Using the sum] {\includegraphics[scale=0.4]{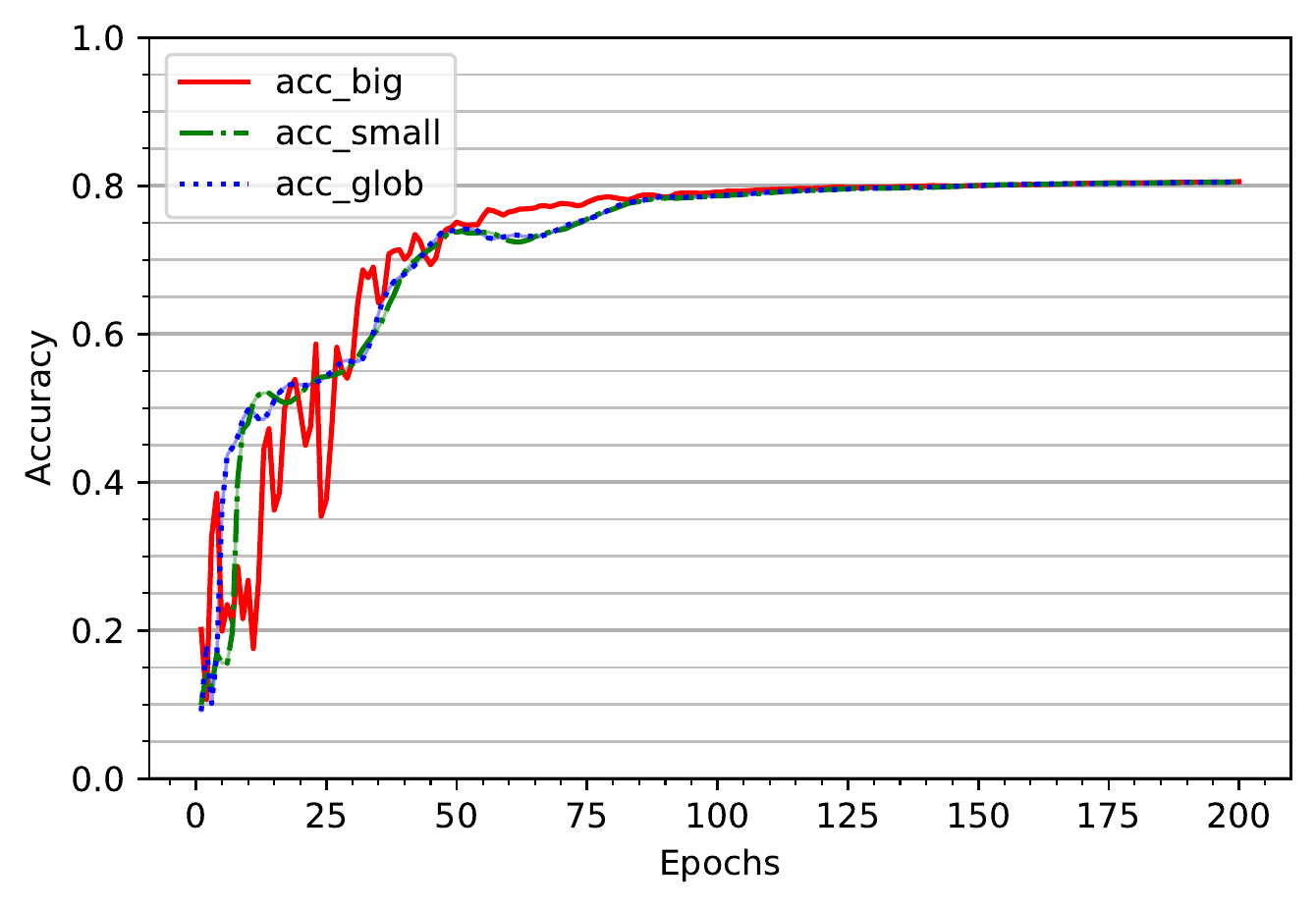}}
      \qquad
     \caption{Linear model on noisy FashionMNIST, for $\lambda = 1$.}
\end{figure}

\begin{figure}
    \centering
    \subfloat[Using the average] 
    {\includegraphics[scale=0.4]{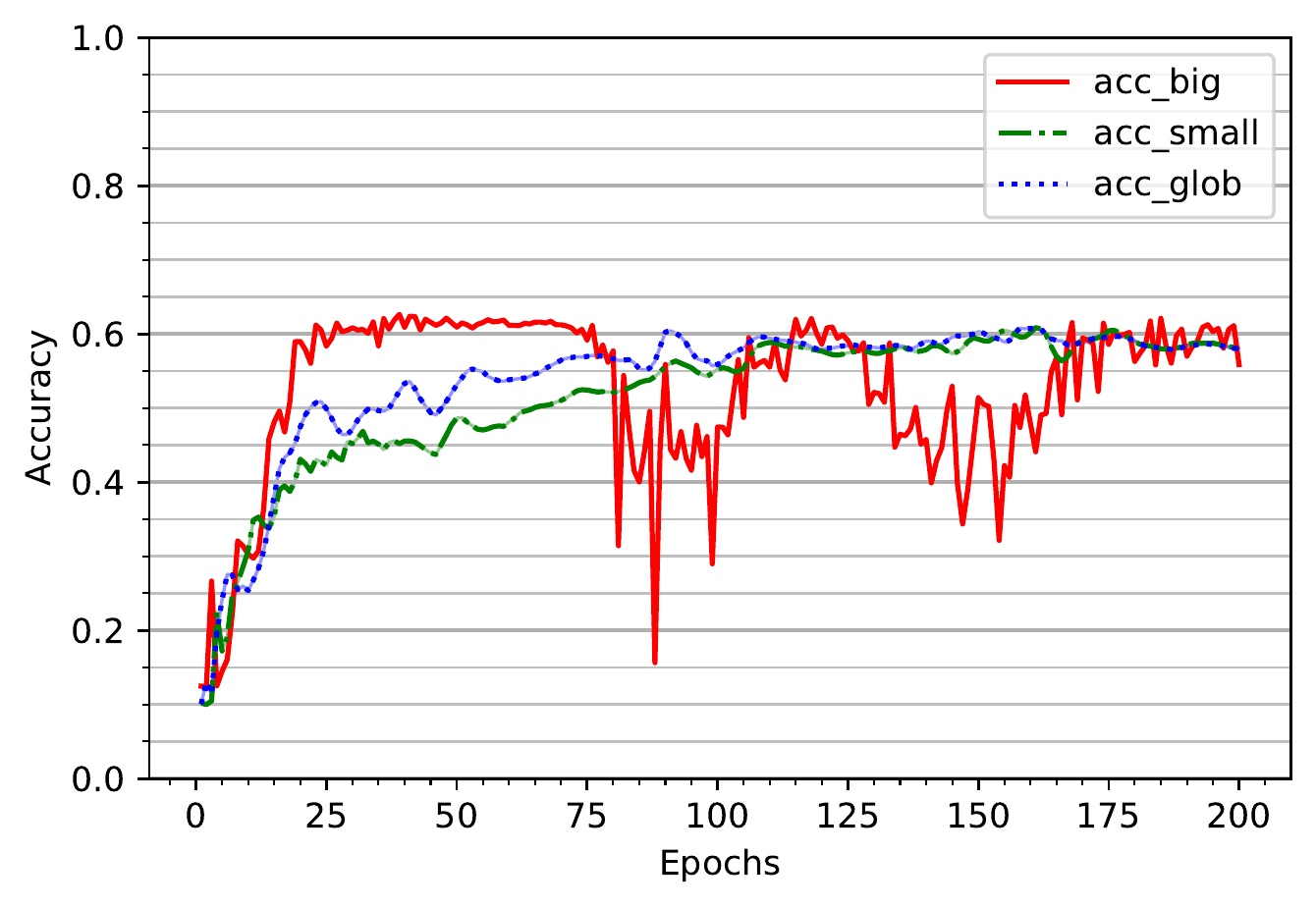}}
      \qquad
    \subfloat[Using the sum] {\includegraphics[scale=0.4]{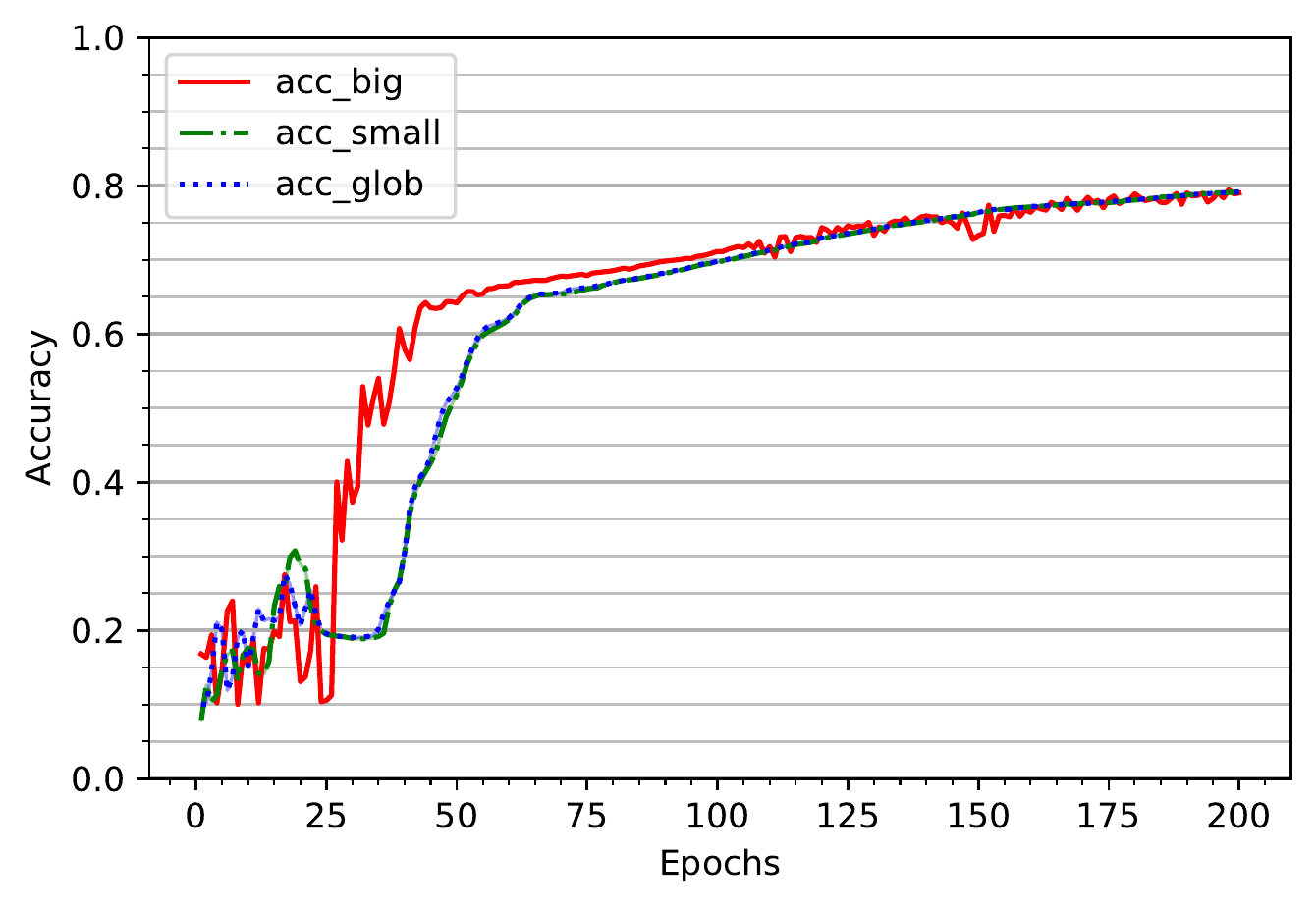}}
      \qquad
     \caption{2-layer neural network on noisy FashionMNIST, for $\lambda = 1$.}
\end{figure}

\begin{figure}
    \centering
    \subfloat[Using the average] 
    {\includegraphics[scale=0.4]{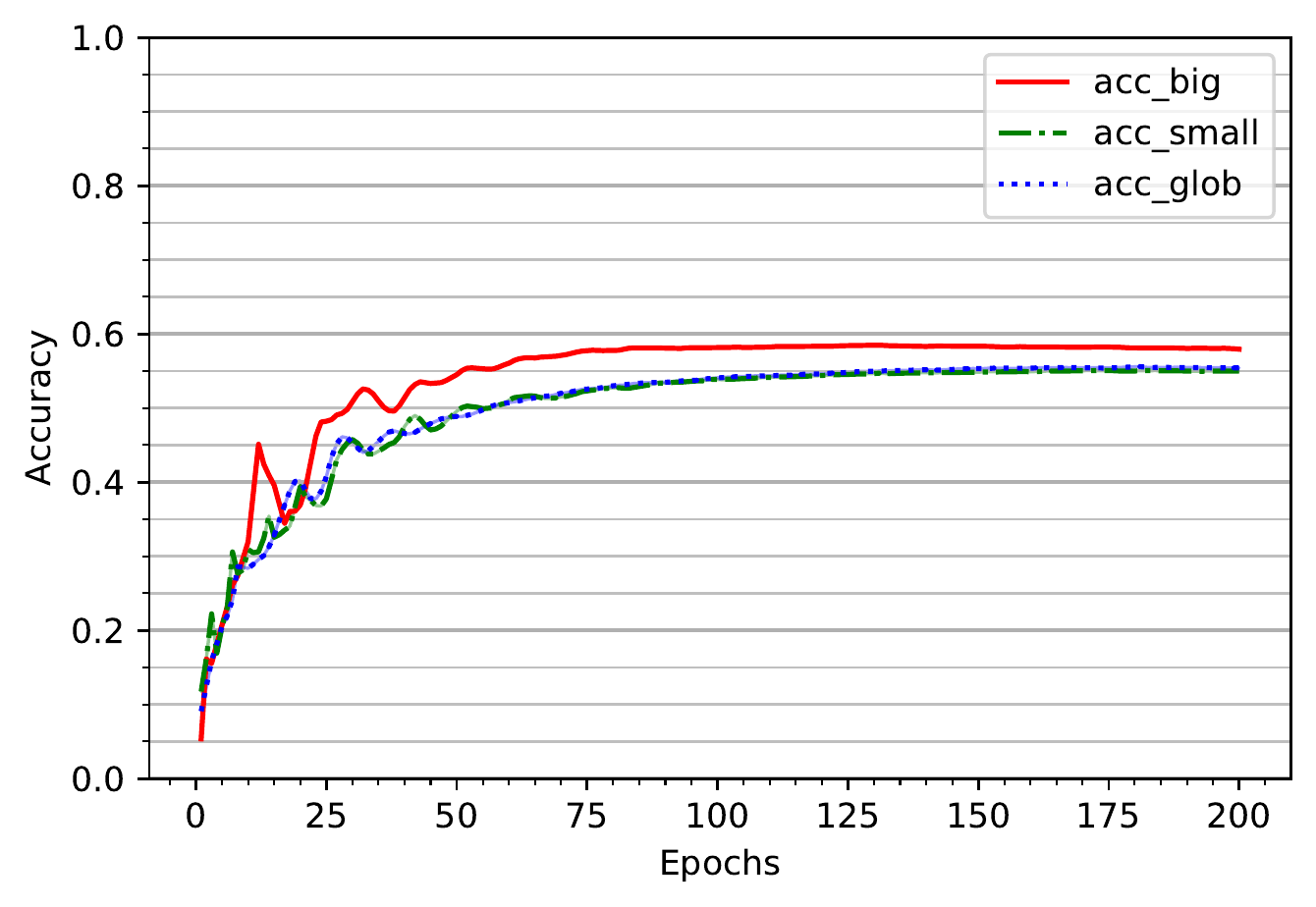}}
      \qquad
    \subfloat[Using the sum] {\includegraphics[scale=0.4]{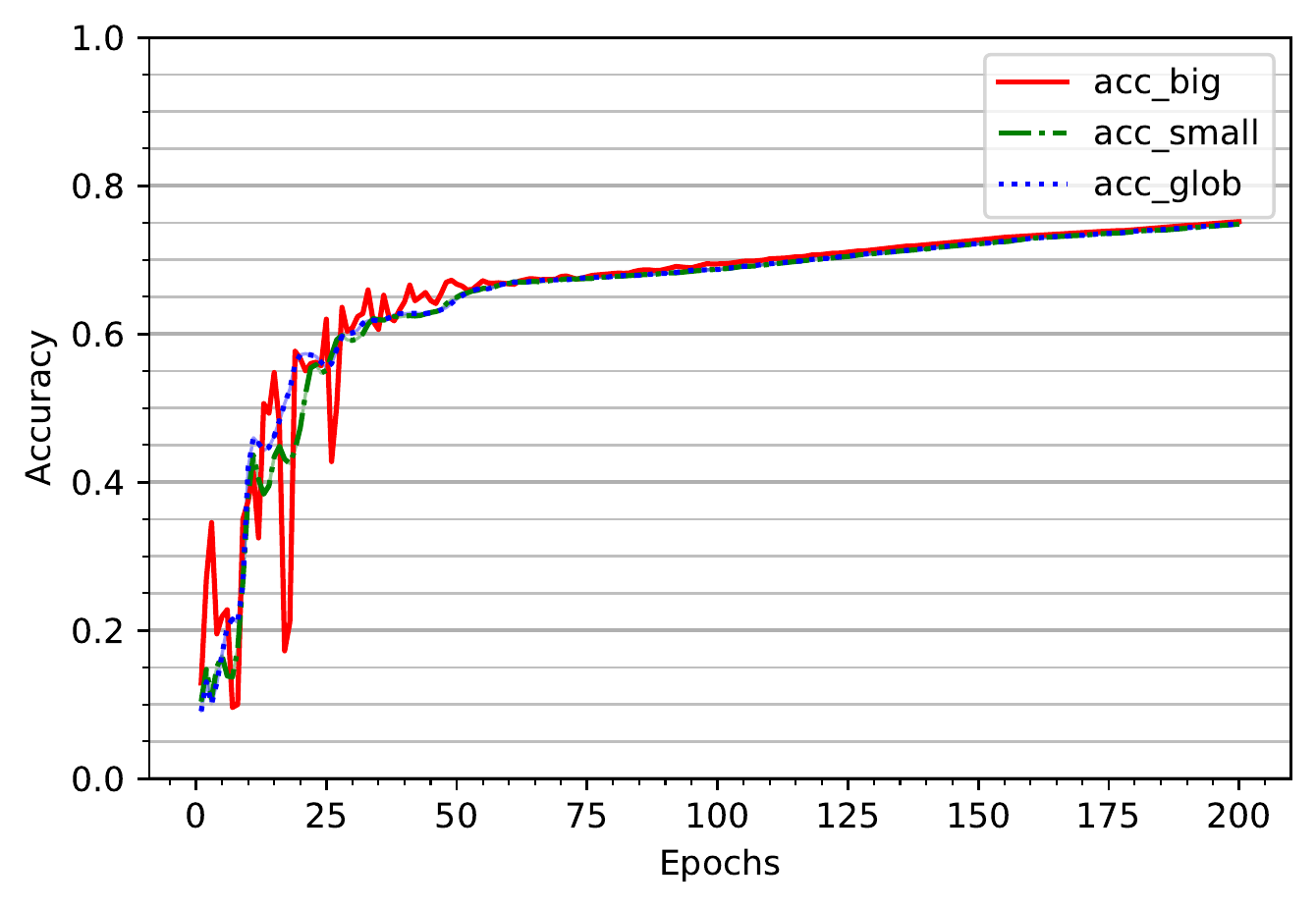}}
      \qquad
     \caption{Linear model on noisy FashionMNIST, for $\lambda = 10$.}
\end{figure}

\begin{figure}
    \centering
    \subfloat[Using the average] 
    {\includegraphics[scale=0.4]{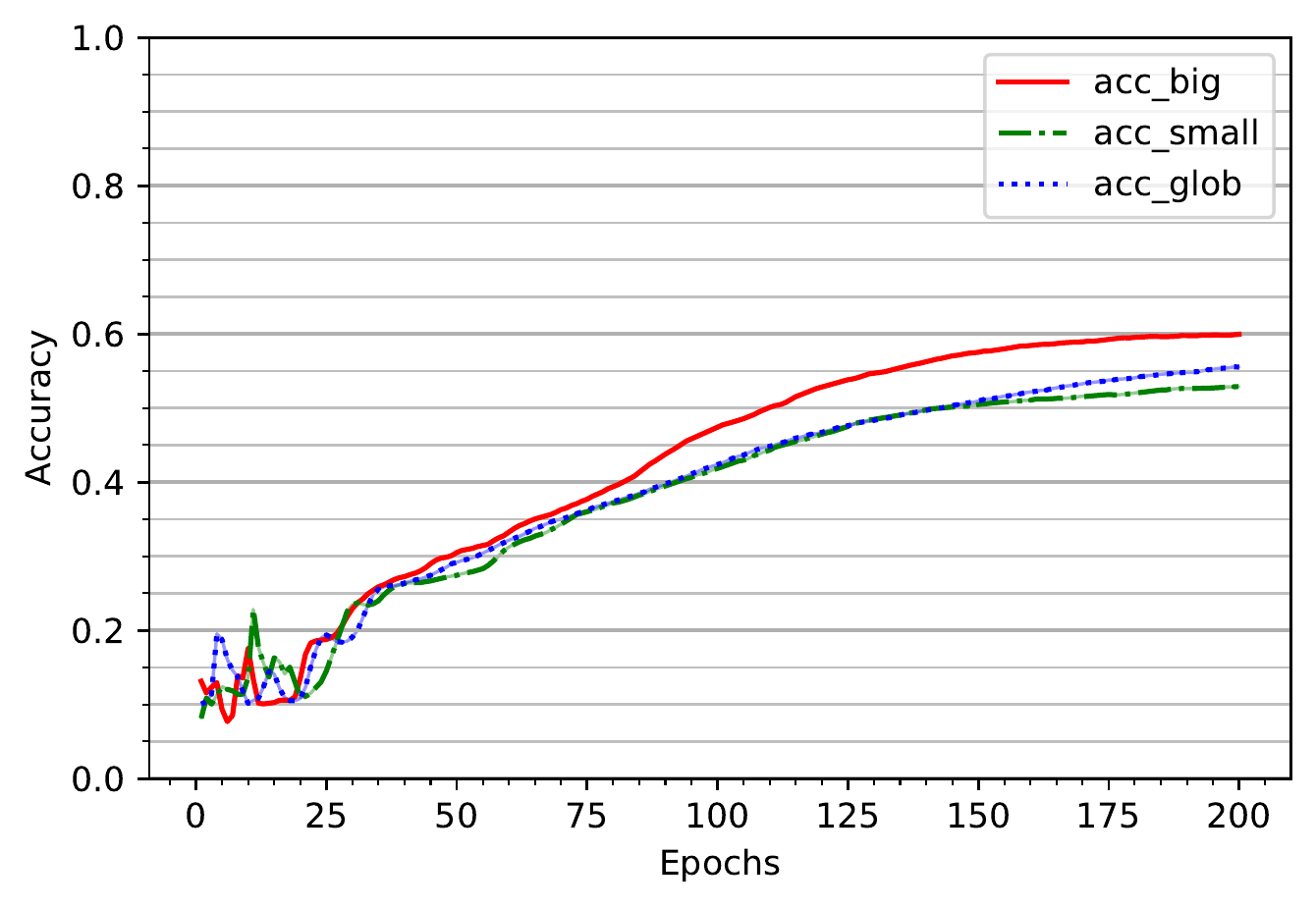}}
      \qquad
    \subfloat[Using the sum] {\includegraphics[scale=0.4]{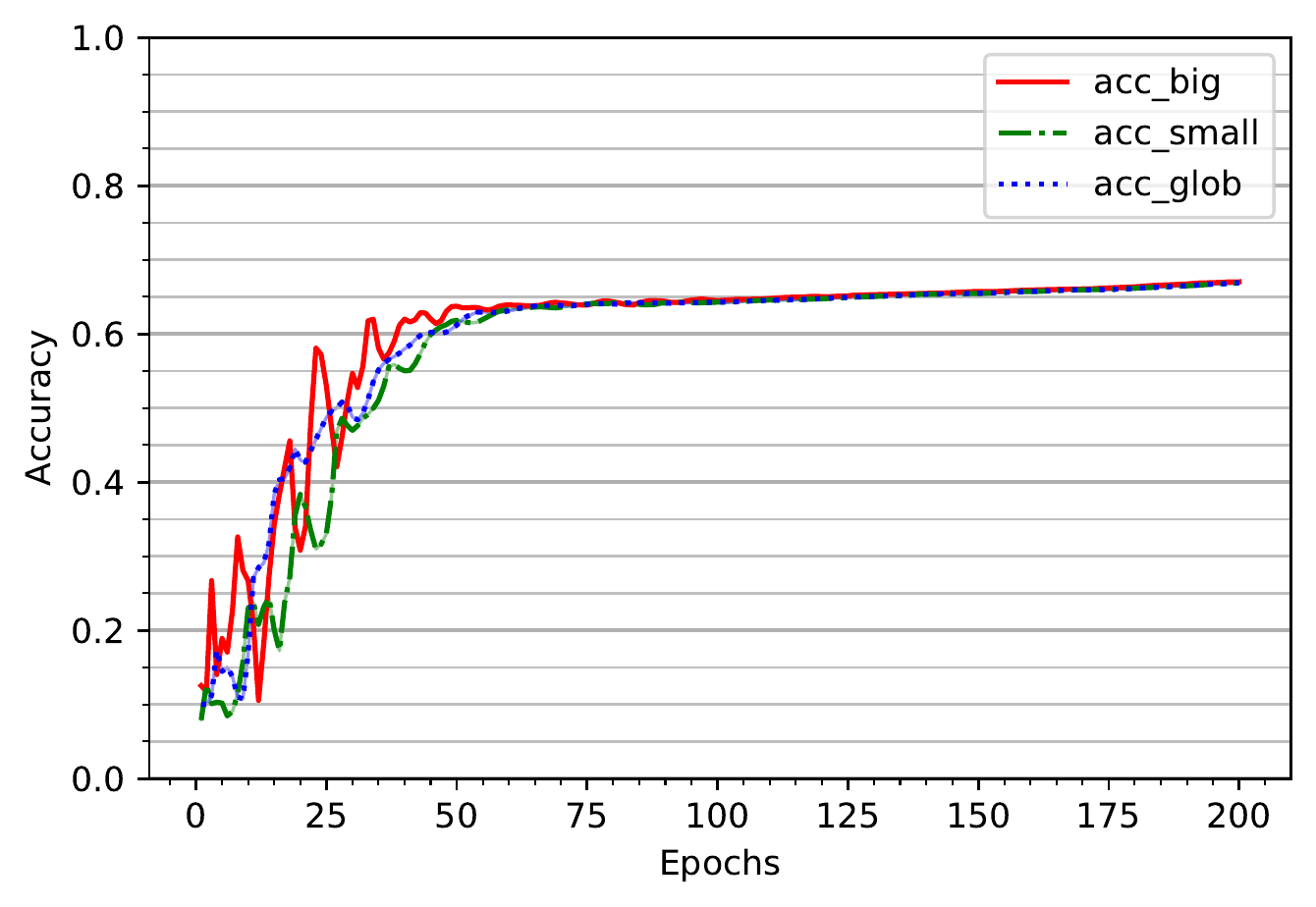}}
      \qquad
     \caption{2-layer neural network on noisy FashionMNIST, for $\lambda = 10$.}
\end{figure}

\begin{figure}
    \centering
    \subfloat[Using the average] 
    {\includegraphics[scale=0.4]{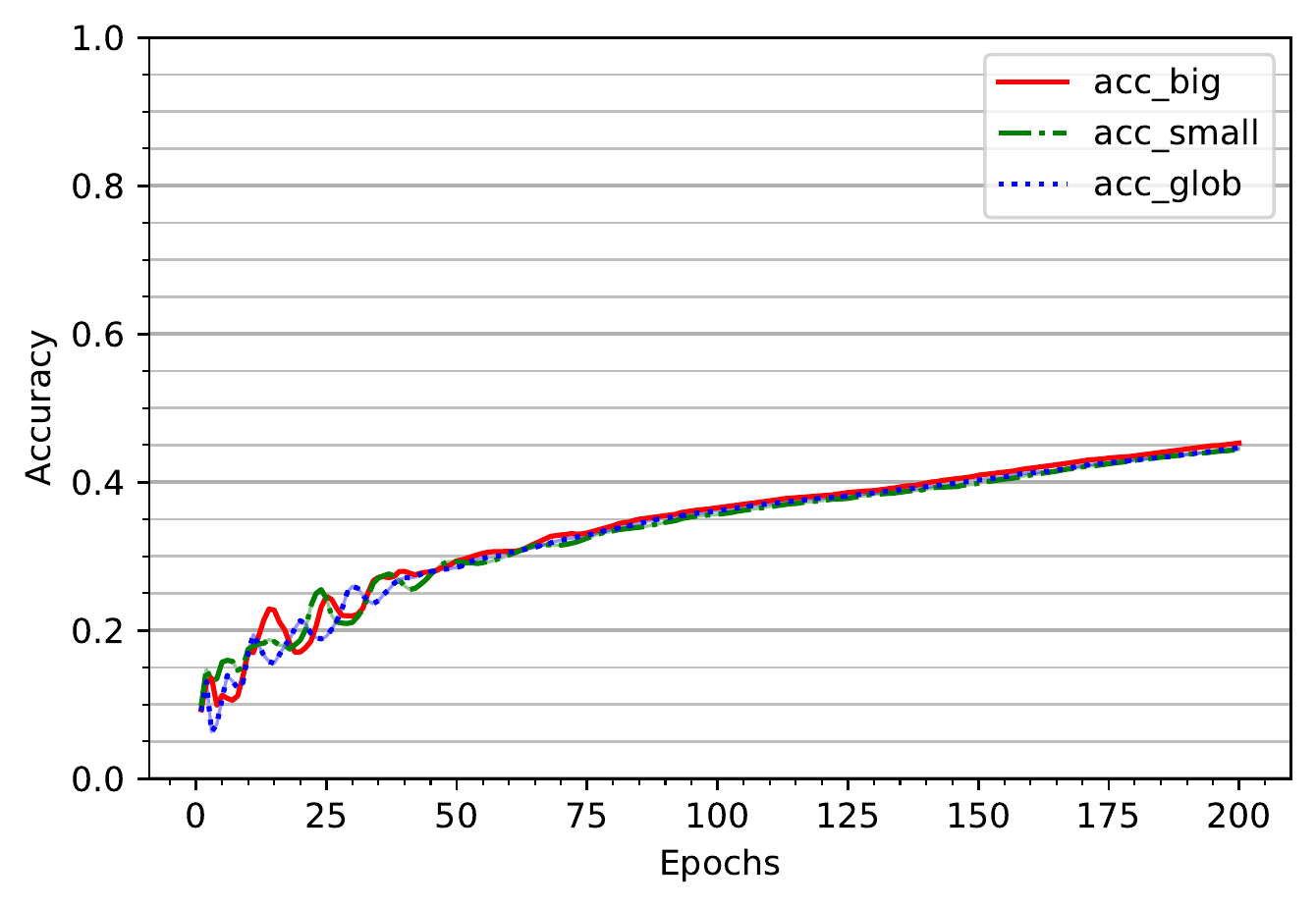}}
      \qquad
    \subfloat[Using the sum] {\includegraphics[scale=0.4]{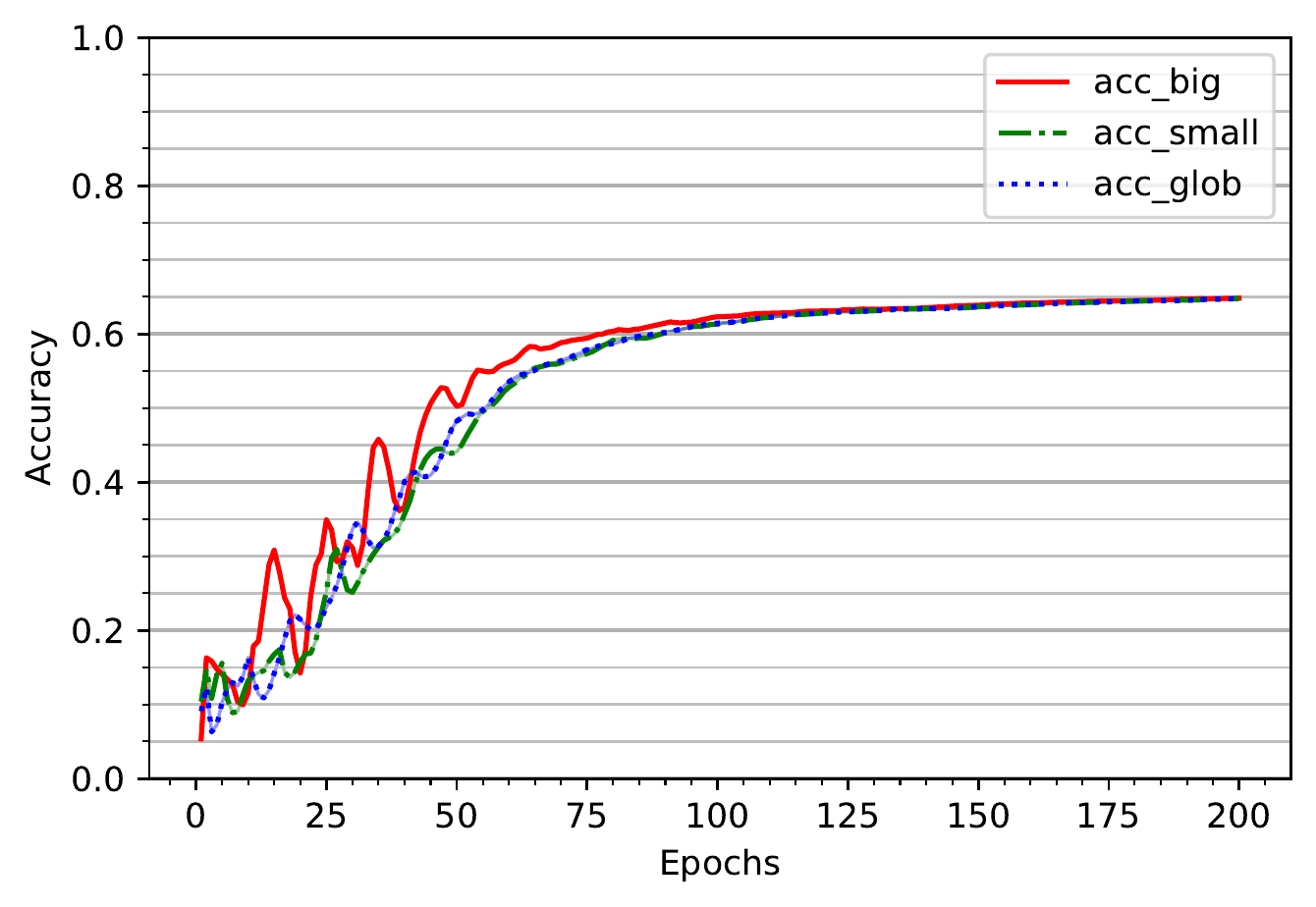}}
      \qquad
     \caption{Linear model on noisy FashionMNIST, for $\lambda = 100$.}
\end{figure}

\begin{figure}
    \centering
    \subfloat[Using the average] 
    {\includegraphics[scale=0.4]{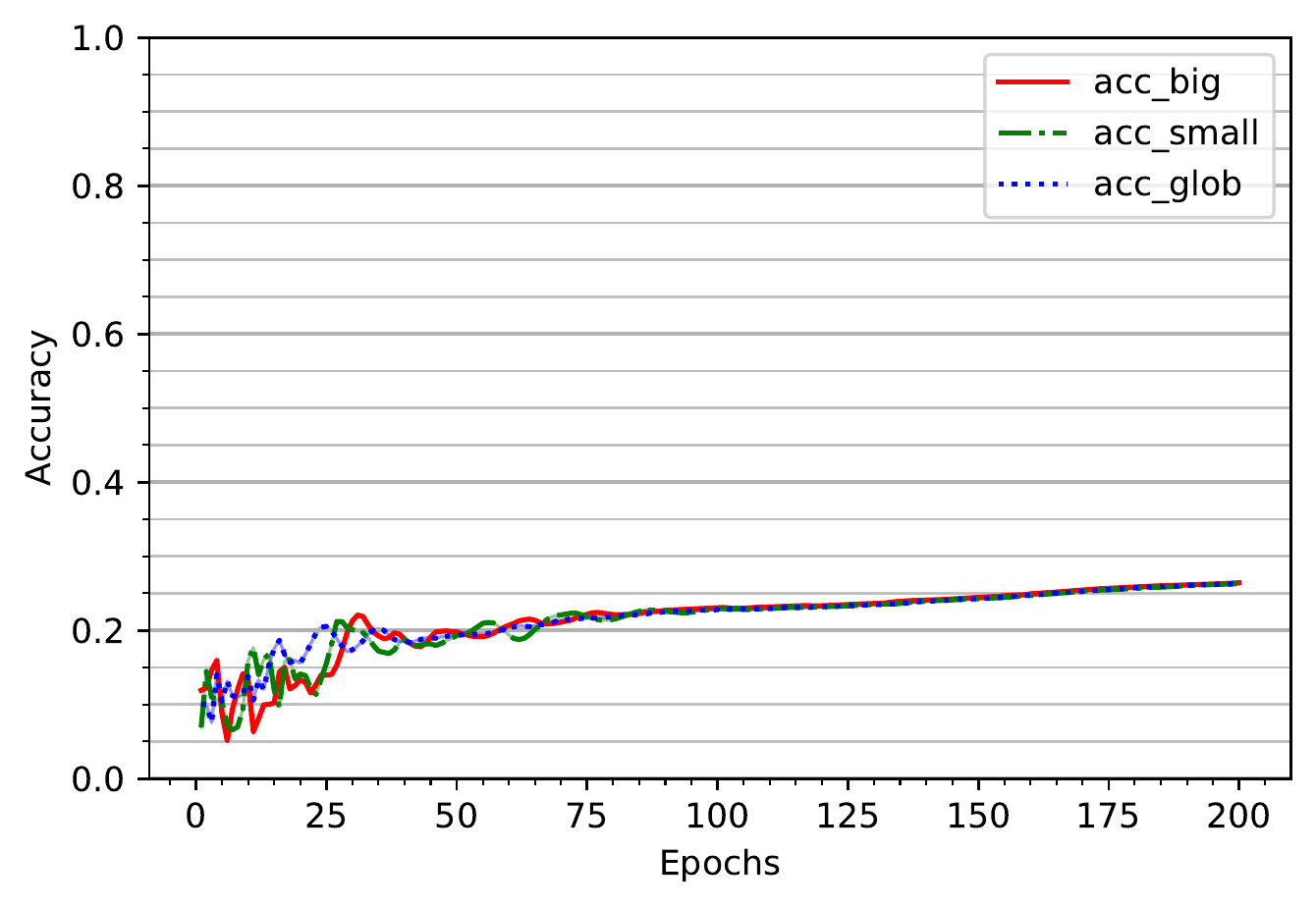}}
      \qquad
    \subfloat[Using the sum] {\includegraphics[scale=0.4]{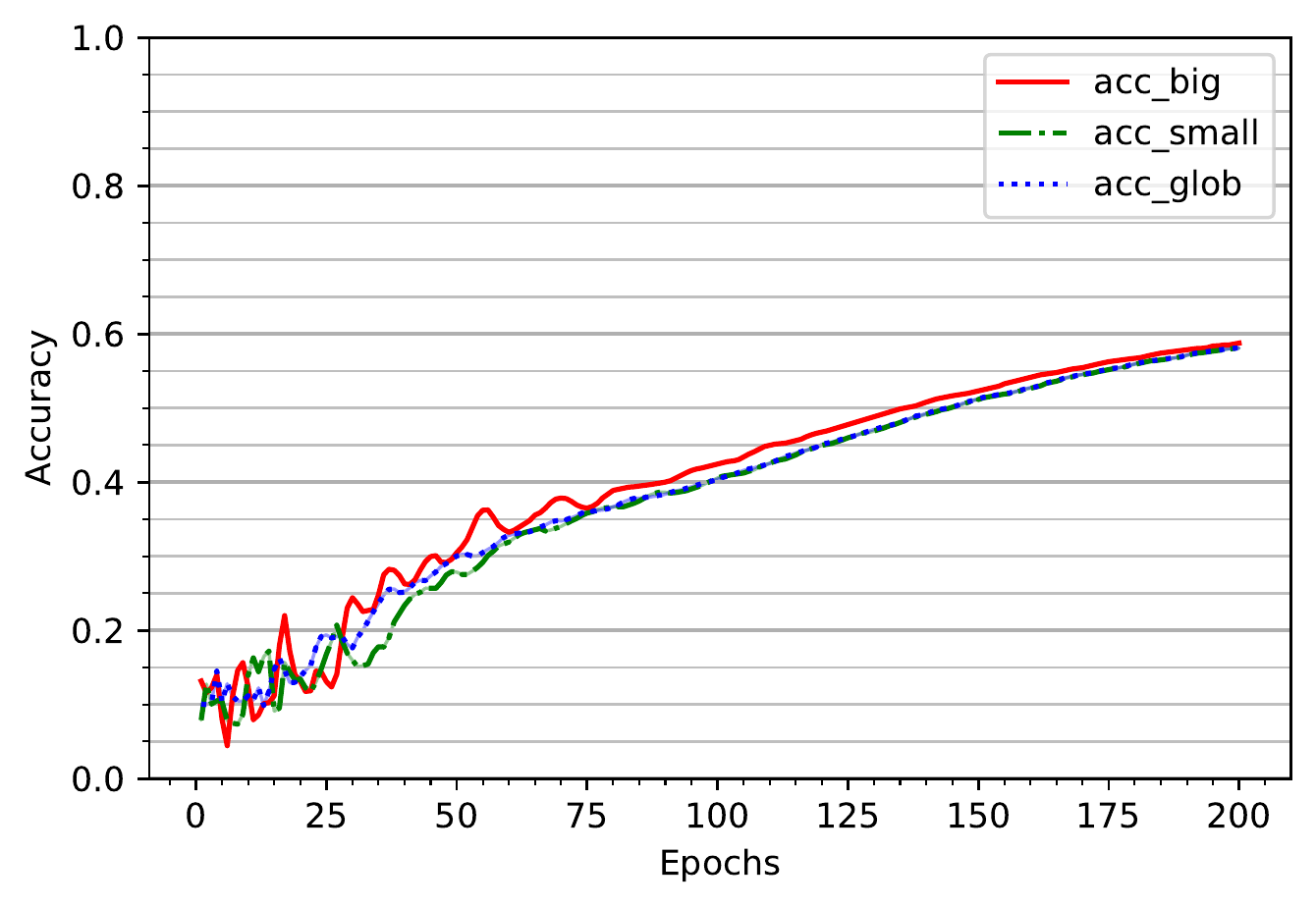}}
      \qquad
     \caption{2-layer neural network on noisy FashionMNIST, for $\lambda = 100$.}
\end{figure}

\subsubsection{FashionMNIST without noise}

Recall that we introduced noise into FashionMNIST to make the problem harder to learn and observe a clear difference between the average and the sum.
In this section, we present results of our experiments when the noise is removed.

\begin{figure}
    \centering
    \subfloat[Using the average] 
    {\includegraphics[scale=0.4]{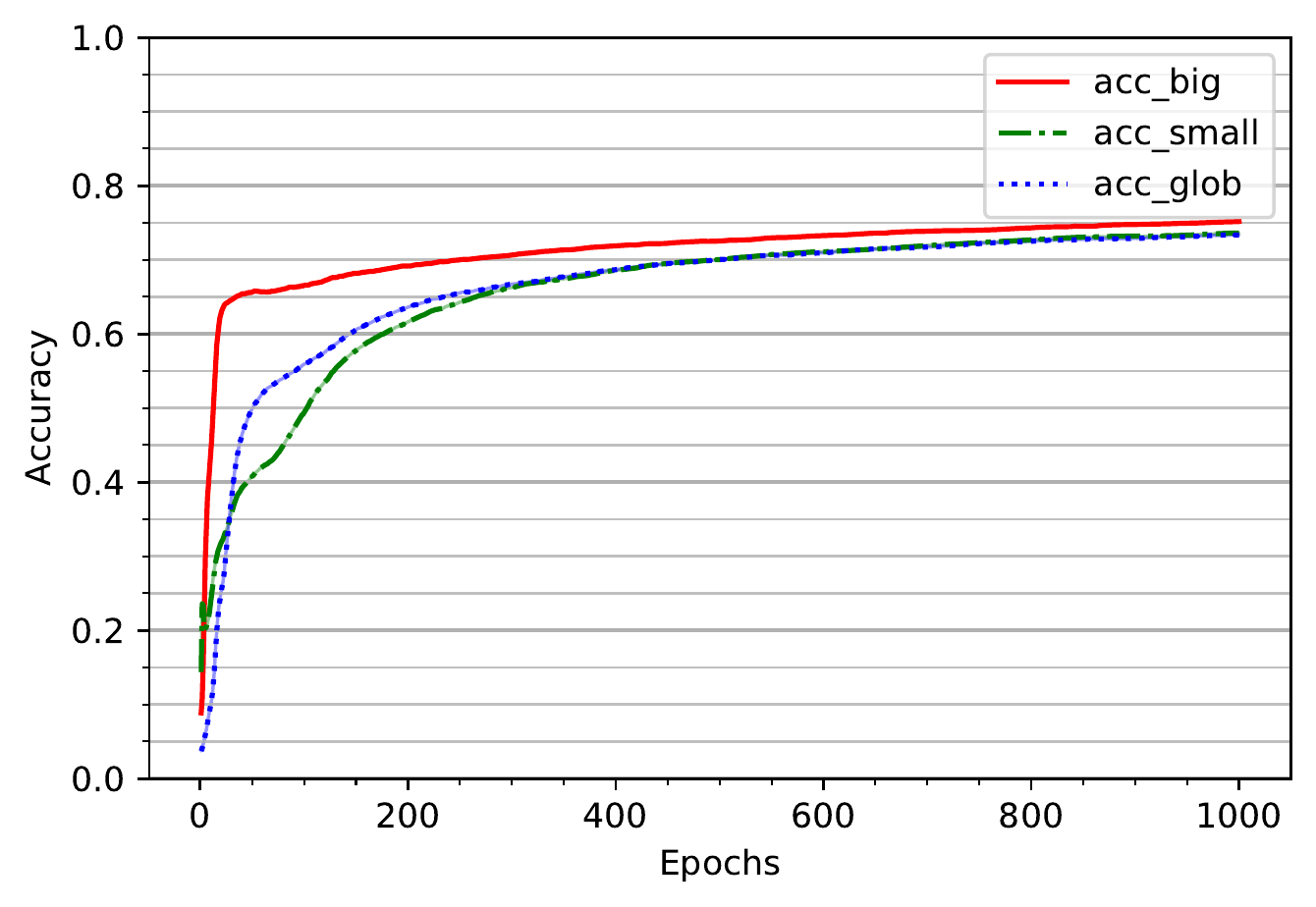}}
      \qquad
    \subfloat[Using the sum] {\includegraphics[scale=0.4]{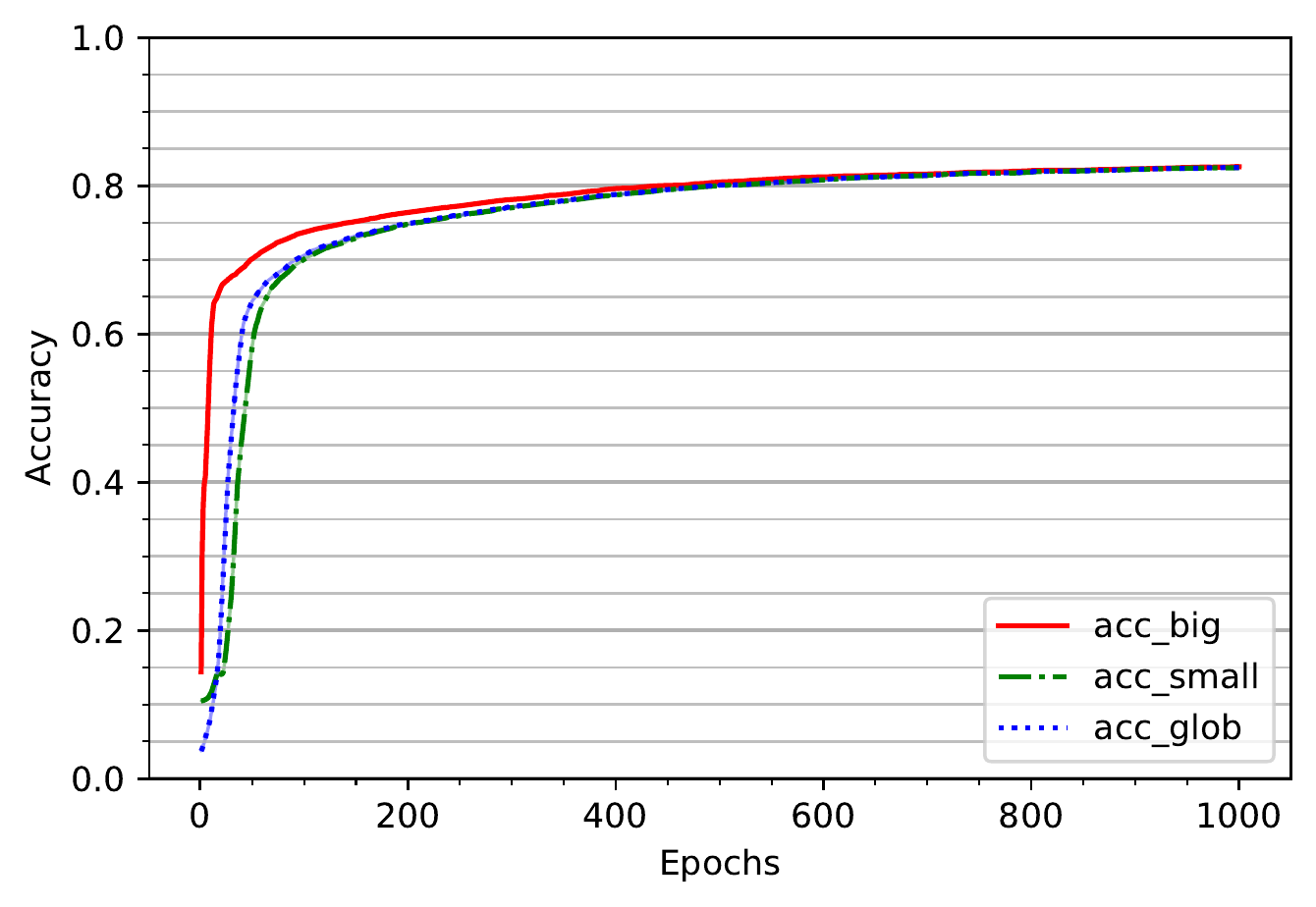}}
      \qquad
     \caption{Linear model on FashionMNIST (without noise), for $\lambda = 1$.}
\end{figure}

\begin{figure}
    \centering
    \subfloat[Using the average] 
    {\includegraphics[scale=0.4]{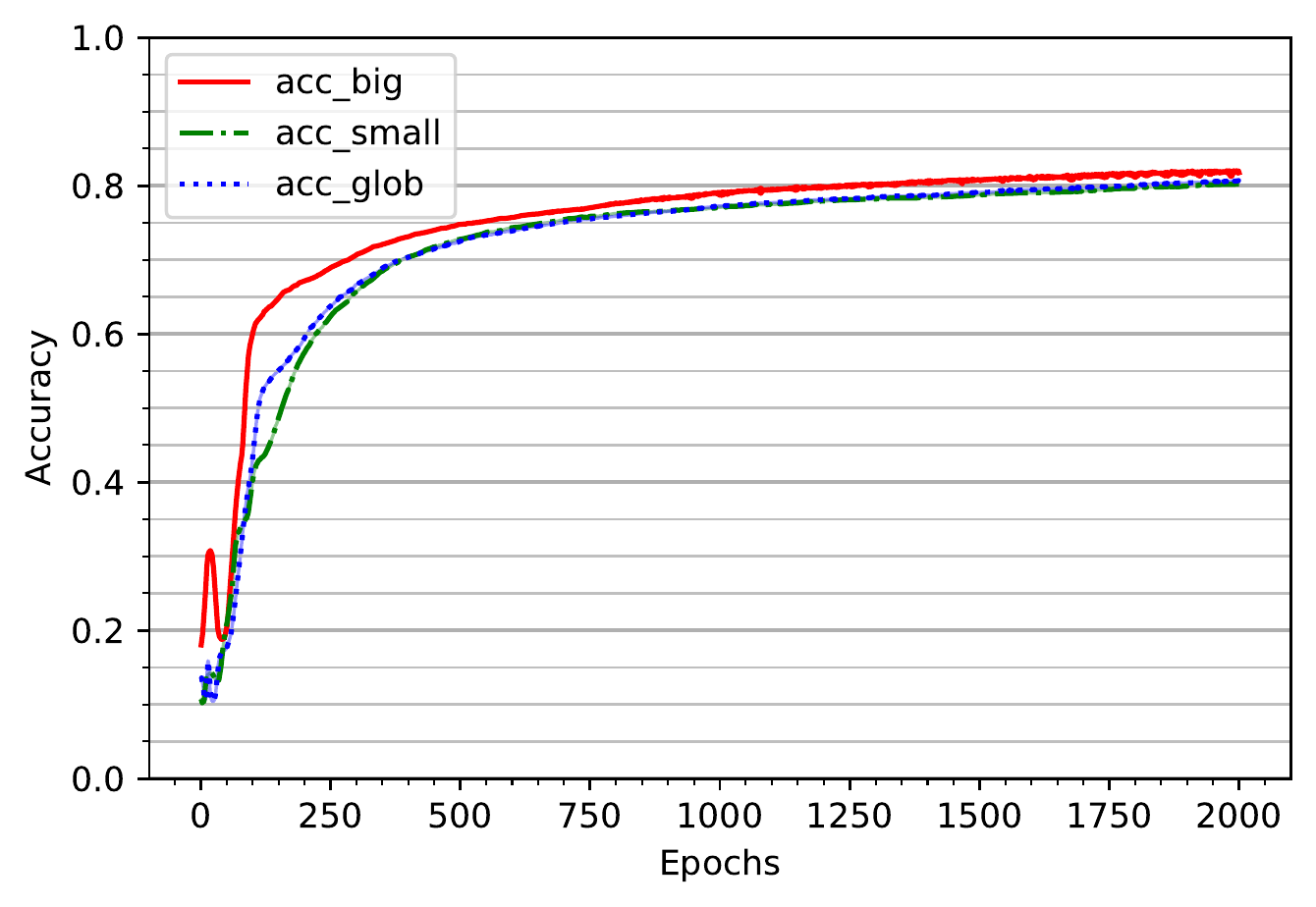}}
      \qquad
    \subfloat[Using the sum] {\includegraphics[scale=0.4]{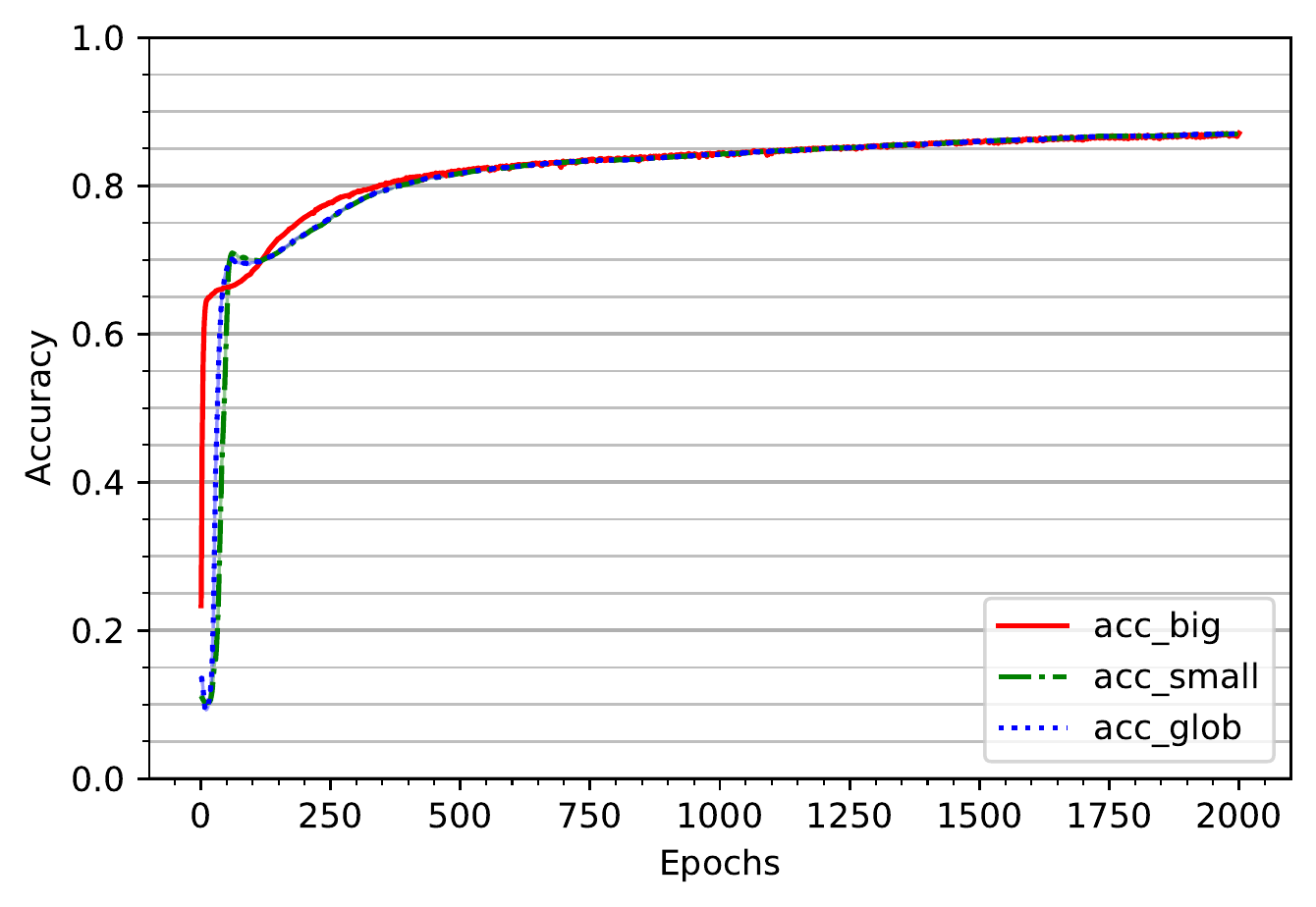}}
      \qquad
     \caption{2-layer neural network on FashionMNIST (without noise), for $\lambda = 1$.}
\end{figure}

Even without noise, the difference between using the sum and using the expectation still seems important.
We acknowledge, however, that the plots suggest that even though we ran this experiment for 10 times more (and 5 times more for the linear model) than other experiments, we might not have reached convergence yet, and that the use of the expectation might still eventually gets closer to the case of sum.
We believe that the fact that the difference between sum and expectation in the absence of noise is weak
is due to the fact that the FashionMNIST dataset is sufficiently linearly separable.
Thus, we achieve a near-zero loss in both cases, which make the sum and the expectation 
close at optimum.

Even in this case, however, we observed that the sum clearly outperforms the expectation especially, in the first epochs. 
We argue that the reason for this is the following.
By taking the average in local losses, the weights of the data of idle users are essentially blown out of proportion.
As a result, the optimizer will very quickly fit these data.
However, the signal from the data of the active user will then be too weak, so that the optimizer has to first almost perfectly fit the idle nodes' data before it can catch the signal of the active user's data and hence the average achieves weaker convergence performances than the sum.

%% file: specific_models_proofs.tex
\section{Linear Regression and Classification are Gradient PAC*}
\label{sec:models_proof}

Throughout this section, we use the following terminology.

\begin{definition}
Consider a parameterized event $\mathcal E(\NODEINPUT{})$. 
We say that the event $\mathcal E$ occurs with high probability if $\probability{\mathcal E(\NODEINPUT{})} \rightarrow 1$ as $\NODEINPUT{} \rightarrow \infty$.
\end{definition}

\subsection{Preliminaries}

Define $\norm{\Sigma}{2} \triangleq \max_{\norm{x}{2} \neq 0} (\norm{\Sigma x}{2}/\norm{x}{2})$ the $\ell_2$ operator norm of the matrix $\Sigma$. 
For symmetric matrices $\Sigma$, this is also the largest eigenvalue in absolute value.

\begin{theorem}[Covariance concentration, Theorem 6.5 in \cite{wainwright19}]
\label{th:covariance_concentration}
Denote $\Sigma = \expect{\query{\nodeinput} \query{\nodeinput}^T}$, where $\query{\nodeinput}\in \setR^d$ is from a $\sigma_{\query{}}$-sub-Gaussian random distribution $\Tilde{\query{}}$.
Then, there are universal constants $c_1$, $c_2$ and $c_3$ such that, for any set $\set{\query{\nodeinput}}_{\nodeinput \in [\NODEINPUT{}]}$ of i.i.d. samples from $\Tilde{\query{}}$, and any $\delta > 0$, the sample covariance $\widehat \Sigma = \frac{1}{\NODEINPUT{}} \sum \query{\nodeinput} \query{\nodeinput}^T$ satisfies the bound
\begin{equation}
    \probability{ \frac{1}{\sigma_{\query{}}^2} \norm{\empiricalcovariance - \Sigma}{2} \geq c_1 \left( \sqrt{\frac{d}{\NODEINPUT{}}} + \frac{d}{\NODEINPUT{}} \right) + \delta }
    \leq c_2 \exp \left( - c_3 \NODEINPUT{} \min(\delta, \delta^2) \right).
\end{equation}
\end{theorem}

\begin{theorem}[Weyl's Theorem, Theorem 4.3.1 in \cite{horn2012}]
\label{th:weyl}
  Let A and B be Hermitian\footnote{For real matrices, Hermitian is the same as symmetric.} and let the respective eigenvalues of $A$ and $B$ and $A+B$ be $\{\lambda_i(A)\}_{i=1}^d$, $\{\lambda_i(B)\}_{i=1}^d$, and $\{\lambda_i(A+B)\}_{i=1}^d$, each increasingly ordered. Then
  \begin{equation}
     \lambda_i(A+B) \leq \lambda_{i+j}(A)+\lambda_{d-j}(B),\quad j = 0,1, ..., d-i,
  \end{equation}
  and
  \begin{equation}
      \lambda_{i+j}(A) + \lambda_{j+1}(B) \leq \lambda_i(A+B),\quad j=0,...,i-1,
  \end{equation}
  for each $i = 1,...,d$.
\end{theorem}

\begin{lemma}
\label{lemma:discrepancy_min_eigenvalues}
Consider two symmetric definite positive matrices $S$ and $\Sigma$.
Denote $\rho_{min}$ and $\lambda_{min}$ their minimal eigenvalues.
Then $\absv{\rho_{min} - \lambda_{min}} \leq \norm{S - \Sigma}{2}$.
\end{lemma}

\begin{proof}
  This is a direct consequence of Theorem \ref{th:weyl}, for $A=S$, $B=\Sigma - S$, $i = 1$, and $j = 0$.
\end{proof}

\begin{corollary}
\label{cor:covariance_concentration}
There are universal constants $c_1$, $c_2$ and $c_3$ such that, for any $\sigma_{\query{}}$-sub-Gaussian vector distribution $\Tilde{\query{}} \in \setR^d$ and any $\delta > 0$, the sample covariance $\widehat \Sigma = \frac{1}{\NODEINPUT{}} \sum \query{\nodeinput} \query{\nodeinput}^T$ satisfies the bound
\begin{equation}
    \probability{ \frac{1}{\sigma_{\query{}}^2} \absv{\min \spectrum(\empiricalcovariance) - \min \spectrum(\Sigma)} \geq c_1 \left( \sqrt{\frac{d}{\NODEINPUT{}}} + \frac{d}{\NODEINPUT{}} \right) + \delta }
    \leq c_2 \exp \left( - c_3 \NODEINPUT{} \min(\delta, \delta^2) \right),
\end{equation}
where $\min \spectrum(\empiricalcovariance) \text{ and } \min \spectrum(\Sigma)$ are the minimal eigenvalues of $\empiricalcovariance$ and $\Sigma$.
\end{corollary}

\begin{proof}
  This follows from Theorem \ref{th:covariance_concentration} and Lemma \ref{lemma:discrepancy_min_eigenvalues}.
\end{proof}

\begin{lemma}
\label{lemma:min_empirical_eigenvalue}
With high probability, $\min \spectrum(\empiricalcovariance) \geq \min \spectrum (\Sigma)/2$.
\end{lemma}

\begin{proof}
  Denote $\lambda_{min} \triangleq \min \spectrum(\Sigma)$ and $\widehat\lambda_{min} \triangleq \min \spectrum(\empiricalcovariance)$.
  Since each  $\query{\nodeinput}$ is drawn i.i.d. from a $\sigma_{\query{}}$-sub-Gaussian, 
  we can apply Corollary \ref{cor:covariance_concentration}. 
  Namely, there are constants $c_1$, $c_2$ and $c_3$, such that for any $\delta>0$, we have
  \begin{equation}
     \probability{\absv{\widehat\lambda_{min}-\lambda_{min}}\geq c_1 \sigma_{\query{}}^2 \left(\sqrt{\frac{d}{\NODEINPUT{}}}+\frac{d}{\NODEINPUT{}} \right) + \delta \sigma_{\query{}}^2 } \leq c_2 \exp{(-c_3 \NODEINPUT{} \min \set{\delta,\delta^2})}.
  \end{equation}
  We now set $\delta \triangleq \lambda_{min}/(4\sigma_{\query{}}^2)$ and we consider $\NODEINPUT{}$ large enough so that $c_1 \left(\sqrt{\frac{d}{\NODEINPUT{}}}+\frac{d}{\NODEINPUT{}} \right) \leq \lambda_{min}/(4\sigma_{\query{}}^2)$.
  With high probability, we then have $\widehat\lambda_{min}\geq \lambda_{min}/2$.
\end{proof}

\subsection{Linear Regression is Gradient-PAC*}
\label{app:least_square}

In this section, we prove the first part of Lemma~\ref{lemma:gradient-pac}.
Namely, we prove that linear regression is gradient-PAC* learning.

\subsubsection{Lemmas for linear regression}

Before moving to the main proof that linear regression is gradient-PAC*, we first prove a few useful lemmas.
These lemmas will rest on the following well-known theorems.

\begin{theorem}[Lemma 2.7.7 in \cite{vershynin18}]
\label{th:product-subgaussian}
If $X$ and $Y$ are sub-Gaussian, then $XY$ is sub-exponential.
\end{theorem}

\begin{theorem}[Equation 2.18 in \cite{wainwright19}]
\label{th:sub-exponential}
If $X_1, \ldots, X_{\NODEINPUT{}}$ are iid sub-exponential variables, then there exist constants $c_4$, $c_5$ such that, for all $\NODEINPUT{}$, we have
\begin{equation}
    \forall t \in [0, c_4] \mathsep
    \probability{\absv{X-\expect{X}} \geq t \NODEINPUT{}} \leq 2 \exp \left( - c_5 \NODEINPUT{} t^2 \right).
\end{equation}
\end{theorem}

\begin{lemma}
\label{lemma:noise_query_sub-exponential}
For all $j \in [d]$, the random variables $X_{\nodeinput} \triangleq \noise{\nodeinput} \query{\nodeinput}[j]$ are iid, sub-exponential and have zero mean.
\end{lemma}

\begin{proof}
  The fact that these variables are iid follows straightforwardly from the fact that the noises $\noise{\nodeinput}$ are iid, and the queries $\query{\nodeinput}$ are also iid.
  Moreover, both are sub-Gaussian, and by Theorem \ref{th:product-subgaussian}, the product of sub-Gaussian variables is sub-exponential.
  Finally, we have $\expect{X} = \expect{\noise{} \query{}[j]} = \expect{\noise{}} \expect{\query{}[j]} = 0$, using the independence of the noise and the query, and the fact that noises have zero mean ($\expect{\noise{}} = 0$).
\end{proof}

\begin{lemma}
\label{lemma:noise_query_concentration_bound}
  There exists $B$ such that $\norm{\sum_{\nodeinput \in \NODEINPUT{}} \noise{\nodeinput} \query{\nodeinput}}{2} \leq B \NODEINPUT{}^{3/4}$ with high probability.
\end{lemma}

\begin{proof}
  By Lemma \ref{lemma:noise_query_sub-exponential}, the terms $\noise{\nodeinput} \query{\nodeinput}[j]$ are iid, sub-exponential and have zero mean. 
  Therefore, by Theorem \ref{th:sub-exponential}, there exist constants $c_4$ and $c_5$ such that for any coordinate $j \in [d]$ of $\noise{\nodeinput} \query{\nodeinput}$ and for all $0\leq u \leq c_4$, we have
  \begin{equation}
    \probability{\absv{\sum_{\nodeinput \in \NODEINPUT{}} \noise{\nodeinput} \query{\nodeinput}[j]}\geq \NODEINPUT{} u} \leq 2\exp{(-c_5\NODEINPUT{}u^2)}.  
  \end{equation}
   Plugging $u =  v \NODEINPUT{}^{(-1/4)}$ into the inequality for some small enough constant $v$, and using union bound then yields
  \begin{equation}
     \probability{\norm{\sum_{\nodeinput \in \NODEINPUT{}} \noise{\nodeinput} \query{\nodeinput}}{2} \geq \NODEINPUT{}^{(3/4)}v\sqrt{d}}\leq \probability{\norm{\sum_{\nodeinput \in \NODEINPUT{}} \noise{\nodeinput} \query{\nodeinput}}{\infty} \geq \NODEINPUT{}^{(3/4)}v} \leq 2d\exp{(-c_5 \sqrt{\NODEINPUT{}}v^2)}.
  \end{equation}
  Defining $B \triangleq v\sqrt{d}$ yields the lemma.
\end{proof}

\subsubsection{Proof that linear regression is gradient-PAC*}

We now move on to proving that least square linear regression is gradient-PAC*.

\begin{proof}[Proof of Theorem \ref{th:linear_regression}]
  Note that $\nabla_\param \lossperinput (\param, \query{}, \answer{}) = (\param^T \query{} - \answer{}) \query{}$.
  Thus, on input $\nodeinput \in [\NODEINPUT{}]$, we have
  \begin{equation}
    \nabla_\param \lossperinput (\param, \query{\nodeinput}, \answer{}(\query{\nodeinput}, \trueparam)) = \left( (\param-\trueparam)^T \query{\nodeinput} \right) \query{\nodeinput} - \noise{\nodeinput} \query{\nodeinput}.
  \end{equation}
  Moreover, we have 
  \begin{equation}
    (\param - \trueparam)^T \nabla_\param \left(\reglocalweight{} \norm{\param}{2}^2\right) = 2\reglocalweight{} (\param - \trueparam)^T \param = 2\reglocalweight{} \norm{\param-\trueparam}{2}^2 +2\reglocalweight{} (\param - \trueparam)^T \trueparam .
  \end{equation}
  As a result, we have
  \begin{align}
  \label{eq:lin_reg_pac_grad}
    &(\param - \trueparam)^T \nabla_\param \independentloss{}(\param, \data{})
    = \\ &\NODEINPUT{}(\param - \trueparam)^T \widehat\Sigma (\param - \trueparam)
    - (\param - \trueparam)^T \left(\sum_{\nodeinput \in \NODEINPUT{}} \noise{\nodeinput} \query{\nodeinput} \right) +  2\reglocalweight{} \norm{\param-\trueparam}{2}^2 +2\reglocalweight{} (\param - \trueparam)^T \trueparam.
  \end{align}
  But now, with high probability, we have 
  $(\param - \trueparam)^T \widehat\Sigma (\param - \trueparam) 
    \geq (\lambda_{min}/2) \norm{\param - \trueparam}{2}^2$ 
  (Lemma \ref{lemma:min_empirical_eigenvalue})
  and $\norm{\sum_{\nodeinput \in \NODEINPUT{}} \noise{\nodeinput} \query{\nodeinput}}{2} \leq B \NODEINPUT{}^{(3/4)}$ (Lemma \ref{lemma:noise_query_concentration_bound}).
   Using the fact that $\norm{\trueparam}{2}\leq \parambound$ and the Cauchy-Schwarz inequality, we have 
  \begin{align}
     (\param - \trueparam)^T \nabla_\param \independentloss{}(\param, \data{}) &\geq (\frac{\lambda_{min}}{2}\NODEINPUT{}+\reglocalweight{}) \norm{\param - \trueparam}{2}^2 -  (B \NODEINPUT{}^{(3/4)} + 2\reglocalweight{}\parambound) \norm{\param - \trueparam}{2}.
  \end{align}
  Denoting $A_\parambound \triangleq \frac{\lambda_{min}}{2}$ and $B_\parambound \triangleq  B + 2\reglocalweight{}\parambound$ and using the fact that $\NODEINPUT{}\geq1$, we then have
  \begin{align}
    (\param - \trueparam)^T \nabla_\param \independentloss{}(\param, \data{}) &\geq A_\parambound \NODEINPUT{} \norm{\param - \trueparam}{2}^2-  B_\parambound \NODEINPUT{}^{(3/4)}  \norm{\param - \trueparam}{2}\\
    &\geq A_\parambound \NODEINPUT{} \min \set{\norm{\param - \trueparam}{2},\norm{\param - \trueparam}{2}^2} -  B_\parambound \NODEINPUT{}^{(3/4)}  \norm{\param - \trueparam}{2},
  \end{align}
  with high probability. 
  This corresponds to saying  Assumption \ref{ass:unbiased} is satisfied for $\alpha = 3/4$.
\end{proof}

\subsection{Logistic Regression}
\label{app:logistic_regression}

In this section, we now prove the second part of Lemma~\ref{lemma:gradient-pac}.
Namely, we prove that logistic regression is gradient-PAC* learning.

\subsubsection{Lemmas about the sigmoid function}

We first prove two useful lemmas about the following logistic distance function.

\begin{definition}
We define the logistic distance function by $\logisticdistance(a,b) \triangleq (a-b)\left(\sigma(a)-\sigma(b)\right)$.
\end{definition}

\begin{lemma}
\label{lemma:sigmiod_bounded_inequality}
  If $a,b\in \setR$ such that for some $k>0$, $\absv{a} \leq k$ and $\absv{b} \leq k$, then there exists some constant $c_k>0$ such that 
  \begin{equation}
    \logisticdistance(a,b) \geq c_k \absv{a-b}^2.
  \end{equation}
\end{lemma}
  
\begin{proof}
  Note that the derivative of $\sigma(z)$ is strictly positive, symmetric ($\sigma'(z) = \sigma'(-z)$) and monotonically decreasing for $z\geq0$. Therefore, for any $z\in[-k,k]$, we know $\sigma'(z)\geq c_k \triangleq \sigma'(k)$. Thus, by the mean value theorem, we have
  \begin{equation}
    \frac{\sigma(a)-\sigma(b)}{a-b} \geq c_k.
  \end{equation}
  Multiplying both sides by $(a-b)^2$ then yields the lemma.
\end{proof}

\begin{lemma}
\label{lemma:sigmiod_general_inequality}
  If $b \in \setR$, and $\absv{b}\leq k$, for some $k > 0$, then there exists a constant $d_k$, such that for any $a \in \setR$, we have 
  \begin{equation}
    \logisticdistance(a,b) \geq d_k \absv{a-b} - d_k
  \end{equation}
\end{lemma}
\begin{proof}
  Assume $\absv{a-b}\geq 1$ and define $d_k \triangleq \sigma(k+1)-\sigma(k)$. If $b\geq 0$, since $\sigma'(z)$ is decreasing for $z\geq0$, we have $\sigma(b)-\sigma(b-1)\geq\sigma(b+1)-\sigma(b) \geq d_k$, and by symmetry, a similar argument holds for $b\leq0$. Thus, we have
  \begin{equation}
    \absv{\sigma(a)-\sigma(b)} \geq \min \set{\sigma(b)-\sigma(b-1),\sigma(b+1)-\sigma(b)} \geq d_k.
  \end{equation}
  Therefore, 
  \begin{equation}
      (a-b)\left(\sigma(a)-\sigma(b)\right) \geq d_k \absv{a-b}\geq  d_k \absv{a-b}-d_k.
  \end{equation}
  For the case of $\absv{a-b}\leq 1$, we also have 
  $(a-b)\left(\sigma(a)-\sigma(b)\right) \geq 0 \geq  d_k \absv{a-b}-d_k$.
\end{proof}

\subsubsection{A uniform lower bound}

\begin{definition}
Denote $\sphere{d-1} \triangleq \set{\unitvector{} \in \setR^d \st \norm{\unitvector{}}{2} = 1}$ the hypersphere in $\setR^d$.
\end{definition}

\begin{lemma}
\label{lemma:positive_query_for_all_unitvectors}
Assume $\support (\querydistribution{})$ spans $\setR^d$. Then, for all $\unitvector{} \in \sphere{d-1}$, 
$\expect{ \absv{\query{}^T \unitvector{}}} > 0$.
\end{lemma}

\begin{proof}
  Let $\unitvector{} \in \sphere{d-1}$.
  We know that there exists $\query{1}, \ldots, \query{d} \in \support(\querydistribution{})$ and $\alpha_1, \ldots, \alpha_d \in \setR$ such that $\unitvector{}$ is colinear with $\sum \alpha_j \query{j}$.
  In particular, we then have $\unitvector{}^T \sum \alpha_j \query{j} = \sum \alpha_j (\query{j}^T \unitvector{}) \neq 0$.
  Therefore, there must be a query $\query{*} \in \support(\querydistribution{})$ such that $\query{*}^T \unitvector{} \neq 0$,
  which implies $a \triangleq \absv{\query{*}^T \unitvector{}} > 0$
  By continuity of the scalar product, there must then also exist $\varepsilon > 0$ such that, for any $\query{} \in \ball(\query{*}, \varepsilon)$, we have $\absv{\query{}^T \unitvector{}} \geq a/2$, where  $\ball(\query{*}, \varepsilon)$ is an Euclidean ball centered on $\query{*}$ and of radius $\varepsilon$.
  
  But now, by definition of the support, we know that $p \triangleq \probability{\query{} \in \ball(\query{*}, \varepsilon)} > 0$.
  By the law of total expectation, we then have
  \begin{align}
      \expect{\absv{\query{}^T \unitvector{}}}
      &= \expect{\absv{\query{}^T \unitvector{}} \st \query{} \in \ball(\query{*}, \varepsilon)} \probability{\query{} \in \ball(\query{*}, \varepsilon)} \nonumber \\
      &\qquad \qquad + \expect{\absv{\query{}^T \unitvector{}} \st \query{} \notin \ball(\query{*}, \varepsilon)} \probability{\query{} \notin \ball(\query{*}, \varepsilon)} \\
      &\geq ap/2 + 0 > 0,
  \end{align}
  which is the lemma.
\end{proof}

\begin{lemma}
\label{lemma:logistic_uniform_lower_bound}
Assume that, for all unit vectors $\unitvector{} \in \sphere{d-1}$, we have $\expect{ \absv{\query{}^T \unitvector{}}} > 0$, and that $\support(\querydistribution{})$ is bounded by $\distbound$. 
Then there exists $C>0$ such that, with high probability, 
  \begin{equation}
     \forall \unitvector{} \in \sphere{d-1} \mathsep
     \sum_{\nodeinput \in \NODEINPUT{}} \absv{\query{\nodeinput}^T \unitvector{}} \geq C \NODEINPUT{}.
  \end{equation}
\end{lemma}

\begin{proof}
  By continuity of the scalar product and the expectation operator, and by compactness of $\sphere{d-1}$, we know that 
  \begin{equation}
      C_0 \triangleq \inf_{\unitvector{} \in \setR^d} \expect{ \absv{\query{}^T \unitvector{}}} > 0.
  \end{equation}
  Now define $\varepsilon \triangleq C_0 / 4\distbound$. 
  Note that $\sphere{d-1} \subset \bigcup_{\unitvector{} \in \sphere{d-1}} \ball(\unitvector{}, \varepsilon)$.
  Thus we have a covering of the hypersphere by open sets.
  But since $\sphere{d-1}$ is compact, we know that we can extract a finite covering.
  In other words, there exists a finite subset $S \subset \sphere{d-1}$ such that $\sphere{d-1} \subset \bigcup_{\unitvector{} \in S} \ball(\unitvector{}, \varepsilon)$.
  Put differently, for any $\unitvectorbis{} \in \sphere{d-1}$, there exists $\unitvector{} \in S$ such that $\norm{\unitvector{} - \unitvectorbis{}}{2} \leq \varepsilon$.
  
  Now consider $\unitvector{} \in S$. 
  Given that $\support(\querydistribution{})$ is bounded, we know that $\absv{\query{\nodeinput}^T \unitvector{}} \in [0, \distbound]$.
  Moreover, such variables $\absv{\query{\nodeinput}^T \unitvector{}}$ are iid.
  By Hoeffding's inequality, for any $t > 0$, we have
  \begin{equation}
     \probability{\absv{\sum_{\nodeinput \in \NODEINPUT{}} \absv{\query{\nodeinput}^T \unitvector{}}-\NODEINPUT{}\expect{\absv{\query{}^T\ \unitvector{}}}} \geq \NODEINPUT{} t} 
     \leq 2 \exp \left( \frac{-2\NODEINPUT{}t^2}{\distbound} \right).
  \end{equation}
  Choosing $t = C_0/2$ then yields
  \begin{align}
    \probability{\sum_{\nodeinput \in \NODEINPUT{}} \absv{\query{\nodeinput}^T \unitvector{}}
    \leq \frac{C_0\NODEINPUT{}}{2}} &\leq \probability{\absv{\sum_{\nodeinput \in \NODEINPUT{}} \absv{\query{\nodeinput}^T \unitvector{\param-\trueparam}}-\NODEINPUT{}\expect{\absv{\query{}^T\ \unitvector{\param-\trueparam}}}} \geq \frac{\NODEINPUT{}C_0}{2}} \\
    &\leq 2 \exp \left(\frac{-\NODEINPUT{}{C_0}^2}{2\distbound}\right).
  \end{align}
  Taking a union bound for $\unitvector{} \in S$ then guarantees 
  \begin{equation}
      \probability{\forall \unitvector{} \in S \mathsep \sum_{\nodeinput \in \NODEINPUT{}} \absv{\query{\nodeinput}^T \unitvector{}} \geq \frac{C_0\NODEINPUT{}}{2}}
      \geq 1 - 2 \card{S} \exp \left(\frac{-\NODEINPUT{}{C_0}^2}{2\distbound}\right),
  \end{equation}
  which clearly goes to 1 as $\NODEINPUT{} \rightarrow \infty$.
  Thus $\forall \unitvector{} \in S \mathsep \sum_{\nodeinput \in \NODEINPUT{}} \absv{\query{\nodeinput}^T \unitvector{}} \geq \frac{C_0\NODEINPUT{}}{2}$ holds with high probability.
  
  Now consider $\unitvectorbis{} \in \sphere{d-1}$.
  We know that there exists $\unitvector{} \in S$ such that $\norm{\unitvector{} - \unitvectorbis{}}{2} \leq \varepsilon$.
  Then, we have
  \begin{align}
      \sum_{\nodeinput \in [\NODEINPUT{}]} \absv{\query{\nodeinput}^T \unitvectorbis{}}
      &= \sum_{\nodeinput \in [\NODEINPUT{}]} \absv{\query{\nodeinput}^T \unitvector{} + \query{\nodeinput}^T (\unitvectorbis{} - \unitvector{})} \\
      &\geq \sum_{\nodeinput \in [\NODEINPUT{}]} \absv{\query{\nodeinput}^T \unitvector{}} - \NODEINPUT{} \distbound \norm{\unitvectorbis{} - \unitvector{}}{2} \\
      &\geq \frac{C_0 \NODEINPUT{}}{2} - \NODEINPUT{} \distbound \frac{C_0}{4 \distbound} 
      = \frac{C_0 \NODEINPUT{}}{4},
  \end{align}
  which proves the lemma.
\end{proof}

\subsubsection{Lower bound on the discrepancy between preferred and reported answers}

\begin{lemma}
\label{lemma:logistic_discrepancy_true_reported}
Assume that $\querydistribution{}$ has a bounded support, whose interior contains the origin.
Suppose also that $\norm{\trueparam}{2} \leq \parambound$.
Then there exists $A_\parambound$ such that, with high probability, we have
\begin{equation}
    \sum_{\nodeinput \in [\NODEINPUT{}]} \logisticdistance (\query{\nodeinput{}}^T \param, \query{\nodeinput{}}^T \trueparam)
    \geq A_\parambound \NODEINPUT{} \min \set{\norm{\param - \trueparam}{2}, \norm{\param - \trueparam}{2}^2}.
\end{equation}
\end{lemma}

\begin{proof}
  Note that by Cauchy-Schwarz inequality we have
  \begin{equation}
    \absv{\query{\nodeinput}^T\trueparam} \leq \norm{\query{\nodeinput}}{2}\norm{\trueparam}{2}\leq \distbound \parambound.
  \end{equation}
  Thus, Lemma \ref{lemma:sigmiod_general_inequality} implies the existence of a positive constant $d_\parambound$, such that for all $\param \in \setR^d$, we have
  \begin{align}
          \sum_{\nodeinput \in \NODEINPUT{}} 
          \logisticdistance \left( \query{\nodeinput}^T \param, \query{\nodeinput}^T \trueparam \right)
          &\geq \sum_{\nodeinput \in \NODEINPUT{}} \left(d_\parambound \absv{\query{\nodeinput}^T\param - \query{\nodeinput}^T\trueparam} - d_\parambound \right)\\
          &= -d_\parambound \NODEINPUT{} +  d_\parambound\norm{\param-\trueparam}{2} \sum_{\nodeinput \in \NODEINPUT{}} \absv{\query{\nodeinput}^T \unitvector{\param-\trueparam}},
  \end{align}
  where $\unitvector{\param-\trueparam} \triangleq (\param-\trueparam)/\norm{\param-\trueparam}{2}$ is the unit vector in the direction of $\param-\trueparam$. 
  
  Now, by Lemma \ref{lemma:logistic_uniform_lower_bound}, we know that, with high probability,
  for all unit vectors $\unitvector{} \in \sphere{d-1}$, we have $\sum \absv{\query{\nodeinput}^T \unitvector{}} \geq C \NODEINPUT{}$.
  Thus, for $\NODEINPUT{}$ sufficiently large, for any $\param \in \setR^d$, with high probability, we have 
  \begin{equation}
  \label{eq:log_reg_bound}
      \sum_{\nodeinput \in \NODEINPUT{}} 
      \logisticdistance (\query{\nodeinput{}}^T \param, \query{\nodeinput{}}^T \trueparam)
      \geq \frac{d_\parambound C_{min}}{2}\NODEINPUT{} \norm{\param-\trueparam}{2} - d_\parambound \NODEINPUT{}.
  \end{equation}
  Defining $e_\parambound\triangleq \frac{d_\parambound C_{min}}{4}$, and $f_\parambound \triangleq \frac{4}{C_{min}}$, for $\norm{\param-\trueparam}{2}>f_\parambound$, we then have
  \begin{equation}
  \label{eq:logreg_second_inequality}
     \sum_{\nodeinput \in \NODEINPUT{}} 
     \logisticdistance (\query{\nodeinput{}}^T \param, \query{\nodeinput{}}^T \trueparam)
     \geq e_\parambound \NODEINPUT{}\norm{\param-\trueparam}{2}.
  \end{equation}
  We now focus on the case of $\norm{\param-\trueparam}{2}\leq f_\parambound$. The triangle inequality yields $\norm{\param}{2}\leq \norm{\param-\trueparam}{2} + \norm{\trueparam}{2} \leq f_\parambound + \parambound$. By Cauchy-Schwarz inequality, we then have $\absv{\query{\nodeinput}^T \param}\leq (f_\parambound + \parambound)\distbound \triangleq g_\parambound$ and $\absv{\query{\nodeinput}^T\trueparam} \leq \parambound \distbound \leq g_\parambound$. 
  Thus, by Lemma \ref{lemma:sigmiod_bounded_inequality}, we know there exists some constant $c_\parambound$ such that
  \begin{align}
    \sum_{\nodeinput \in \NODEINPUT{}} \left(\sigma(\query{\nodeinput}^T\param) - \sigma(\query{\nodeinput}^T\trueparam)\right) (\query{\nodeinput}^T\param - \query{\nodeinput}^T\trueparam) &\geq \sum_{\nodeinput \in \NODEINPUT{}} c_\parambound \absv{\query{\nodeinput}^T\param - \query{\nodeinput}^T\trueparam}^2 \\
    &= \sum_{\nodeinput \in \NODEINPUT{}} c_\parambound (\param-\trueparam)^T\query{i}\query{i}^T(\param-\trueparam)\\
    &= c_\parambound (\param-\trueparam)^T \left( \sum_{\nodeinput \in \NODEINPUT{}} \query{i}\query{i}^T \right) (\param-\trueparam).
  \end{align}
  Since distribution $\Tilde{\query{}}$ is bounded (and thus sub-Gaussian), by Theorem \ref{th:covariance_concentration}, with high probability, we have
  \begin{equation}
    (\param-\trueparam)^T \left( \sum_{\nodeinput \in \NODEINPUT{}} \query{i}\query{i}^T \right) (\param-\trueparam) \geq \frac{\lambda_{min}}{2} \NODEINPUT{} \norm{\param-\trueparam}{2}^2,
  \end{equation}
  where $\lambda_{min}$ is the smallest eigenvalue of $\expect{\query{\nodeinput}\query{\nodeinput}^T}$. Thus, for $\norm{\param-\trueparam}{2}\leq f_\parambound$, we have
  \begin{equation}
  \label{eq:logreg_first_inequality}
    \sum_{\nodeinput \in \NODEINPUT{}} \left(\sigma(\query{\nodeinput}^T\param) - \sigma(\query{\nodeinput}^T\trueparam)\right) (\query{\nodeinput}^T\param - \query{\nodeinput}^T\trueparam) \geq \frac{\lambda_{min}c_\parambound}{2} \NODEINPUT{} \norm{\param-\trueparam}{2}^2.
  \end{equation}
  Combining this with (\ref{eq:logreg_second_inequality}), and defining $A_\parambound \triangleq \min \set{\frac{\lambda_{min}c_\parambound}{2}, e_\parambound}$, we then obtain the lemma.
\end{proof}

\subsubsection{Proof that logistic regression is gradient-PAC*}

Now we proceed with the proof that logistic regression is gradient-PAC*.

\begin{proof}[Proof of Theorem \ref{th:logistic_regression}]
  Note that $\sigma(-z) = e^{-z} \sigma(z) = 1-\sigma(z)$ and $\sigma'(z) = e^{-z} \sigma^2(z)$. We then have
  \begin{align}
    \nabla_\param \lossperinput (\param, \query{}, \answer{})
    &= - \frac{\sigma'(\answer{} \query{}^T \param) \answer{} \query{}}{\sigma(\answer{} \query{}^T \param)} 
    = - e^{-\answer{} \query{}^T \param} \sigma(\answer{} \query{}^T \param) \answer{} \query{} \\
    &= - \sigma(- \answer{} \query{}^T \param) \answer{} \query{} 
    = \left(\sigma(\query{}^T\param) - \indicator{\answer{}=1}\right)\query{},
  \end{align}
  where $\indicator{\answer{}=1}$ is the indicator function that outputs $1$ if $\answer{}=1$, and $0$ otherwise.
  As a result,
  \begin{align}
    &(\param-\trueparam)^T\nabla_\param \independentloss{}(\param, \data{}) =\\ &(\param-\trueparam)^T \left( \sum_{\nodeinput \in \NODEINPUT{}} \left(\sigma(\query{\nodeinput}^T\param) - \indicator{\answer{\nodeinput}=1}\right)\query{\nodeinput}\right) + 2 \reglocalweight{} (\param-\trueparam)^T\param \\ &= (\param-\trueparam)^T \left( \sum_{\nodeinput \in \NODEINPUT{}} \left(\sigma(\query{\nodeinput}^T\param) - \sigma(\query{\nodeinput}^T\trueparam) + \sigma(\query{\nodeinput}^T\trueparam) - \indicator{\answer{\nodeinput}=1}\right)\query{\nodeinput}\right)\\
    &+2\reglocalweight{} \norm{\param-\trueparam}{2}^2 +2\reglocalweight{} (\param - \trueparam)^T \trueparam\\ 
    &= \label{eq:log_error} 
    \sum_{\nodeinput \in [\NODEINPUT{}]}
    \logisticdistance\left( \query{\nodeinput}^T \param, \query{\nodeinput}^T \trueparam \right)
    + (\param-\trueparam)^T \left( \sum_{\nodeinput \in \NODEINPUT{}} \left(\sigma(\query{\nodeinput}^T\trueparam) - \indicator{\answer{\nodeinput}=1}\right)\query{\nodeinput}\right)\\
    &+2\reglocalweight{} \norm{\param-\trueparam}{2}^2 +2\reglocalweight{} (\param - \trueparam)^T \trueparam.
  \end{align}
  By Lemma \ref{lemma:logistic_discrepancy_true_reported}, with high probability, we have
  \begin{equation}
      \sum_{\nodeinput \in [\NODEINPUT{}]}
    \logisticdistance\left( \query{\nodeinput}^T \param, \query{\nodeinput}^T \trueparam \right)
    \geq A_\parambound \NODEINPUT{} \min \set{ \norm{\param - \trueparam}{2}, \norm{\param - \trueparam}{2}^2 }.
  \end{equation}
  To control the second term of (\ref{eq:log_error}), note that the random vectors $Z_\nodeinput \triangleq \left(\sigma(\query{\nodeinput}^T\trueparam) - \indicator{\answer{\nodeinput}=1}\right)\query{\nodeinput}$ are iid with norm at most $\distbound$. 
  Moreover, since $\expect{\indicator{\answer{\nodeinput}=1}|\query{\nodeinput}} = \sigma(\query{\nodeinput}^T\trueparam)$, by the tower rule, we have $\expect{Z_\nodeinput} = \expect{\expect{Z_\nodeinput|\query{\nodeinput}}} = 0$. Therefore, by applying Hoeffding’s bound to every coordinate of $Z_i$, and then taking a union bound, for any $B>0$, we have 
  \begin{equation}
    \probability{\norm{\sum_{\nodeinput \in \NODEINPUT{}} Z_i}{2}\geq B \NODEINPUT{}^{3/4}} 
    \leq 2d\exp \left(-\frac{B^2\sqrt{\NODEINPUT{}}}{2d\distbound^2} \right).
  \end{equation}
  Applying now Cauchy-Schwarz inequality, with high probability, we have
  \begin{equation}
    \absv{\nonumber (\param-\trueparam)^T \left( \sum_{\nodeinput \in \NODEINPUT{}} \left(\sigma(\query{\nodeinput}^T\trueparam) - \indicator{\answer{\nodeinput}=1}\right)\query{\nodeinput}\right) }\leq B \NODEINPUT{}^{3/4} \norm{\param-\trueparam}{2}.
  \end{equation}
  Combining this with (\ref{eq:logreg_first_inequality}) and using $\norm{\trueparam}{2}^2\leq \parambound$, we then have
  \begin{align}
     &(\param - \trueparam)^T \nabla_\param \independentloss{}(\param, \data{})\\ &\geq (A_\parambound\NODEINPUT{}+\reglocalweight{}) \set{ \norm{\param - \trueparam}{2}, \norm{\param - \trueparam}{2}^2 } -  (B \NODEINPUT{}^{(3/4)} + 2\reglocalweight{}\parambound) \norm{\param - \trueparam}{2}\\
     &\geq A_\parambound\NODEINPUT{} \set{ \norm{\param - \trueparam}{2}, \norm{\param - \trueparam}{2}^2 } -  B_\parambound \NODEINPUT{}^{(3/4)} \norm{\param - \trueparam}{2},
  \end{align}
  where $B_\parambound = B + 2\reglocalweight{}\parambound$.
  This shows that Assumption $\ref{ass:unbiased}$ is satisfied for logistic loss for $\alpha = 3/4$, and $A_\parambound$ and $B_\parambound$ as previously defined.
  
\end{proof}

%% file: PAC_proof.tex
\section{Proofs of Local PAC*-Learnability}
\label{sec:pac_proof}

Let us now prove Lemma~\ref{th:pac}.
To do so, consider the preferred models $\trueparamfamily$ and a subset $\HONEST \subset [\NODE]$ of honest users.
Denote $\datafamily{-\HONEST}$ the datasets provided by users $\node \in [\NODE] - \HONEST$.
Each honest user $\honest \in \HONEST$ provides an honest dataset $\data{\honest}$ of cardinality at least $\NODEINPUT{} \geq 1$.
Consider the bound $K_\HONEST \triangleq \max_{\honest \in \HONEST} \norm{\trueparamsub{\honest}}{2}$ on the parameter norm of honest active users $\honest \in \HONEST$.

\subsection{Bounds on the Optima}

Before proving the theorem, we prove a useful lemma that bounds the set of possible values for the global model and honest local models.

\begin{lemma}
\label{lemma:bounded_optimum}
  Assume that $\regularization$ and $\lossperinput$ are nonnegative.
  For $\NODEINPUT{}$ large enough, if all honest active users $\honest \in \HONEST$ provide at least $\NODEINPUT{}$ data, then, with high probability, $\optimumfamily_{\HONEST}$ must lie in a compact subset of $\setR^{d \times \HONEST}$ that does not depend on $\NODEINPUT{}$.
\end{lemma}

\begin{proof}
  Denote $L^0  \triangleq \globalloss{} (0, (\trueparamfamily_{H},0_{-\HONEST}), (\emptyset,\datafamily{-\HONEST}))$.
  Essentially, we will show that, if $\optimumfamily_{\HONEST}$ is too far from $\trueparamfamily_{\HONEST}$, then the loss will take values strictly larger than $L^0$.

Assumption \ref{ass:unbiased} implies the existence of an event $\mathcal E$ that occurs with probability at least $\Probability_0 \triangleq \Probability(K_\HONEST, \NODEINPUT{})^{\card{\HONEST}}$, under which, for any $\paramsub{\honest} \in \setR^d$, we have
\begin{equation}
\label{eq:week_gradient-pac}
    \left( \paramsub{\honest} - \trueparamsub{\honest} \right)^T \nabla \independentloss{\honest} \left(\paramsub{\honest} \right)
    \geq A_{K_\HONEST} \NODEINPUT{} \min \left\lbrace \norm{\paramsub{\honest}-\trueparamsub{\honest}}{2}, \norm{\paramsub{\honest}-\trueparamsub{\honest}}{2}^2 \right\rbrace - B_{K_\HONEST} \NODEINPUT{}^\alpha \norm{\paramsub{\honest} - \trueparamsub{\honest}}{2},
\end{equation}
which implies
\begin{equation}
    \unitvector{\left( \paramsub{\honest} - \trueparamsub{\honest} \right)}^T \nabla \independentloss{\honest} \left(\paramsub{\honest} \right)
    \geq A_{K_\HONEST} \NODEINPUT{} \min \left\lbrace 1, \norm{\paramsub{\honest}-\trueparamsub{\honest}}{2} \right\rbrace - B_{K_\HONEST} \NODEINPUT{}^\alpha.
\end{equation}
Note also that $P_0 \rightarrow 1$ as $\NODEINPUT{} \rightarrow \infty$.
We now integrate both sides over the line segment from $\trueparamsub{\honest}$ to $\paramsub{\honest}$. The fundamental theorem of calculus for line integrals then yields

\begin{align}
&\independentloss{\honest} \left(\paramsub{\honest} \right) - \independentloss{\honest} \left(\trueparamsub{\honest} \right)=
  \norm{\paramsub{\honest}-\trueparamsub{\honest}}{2} \int_{t=0}^1 \unitvector{\left( \paramsub{\honest} - \trueparamsub{\honest} \right)}^T \nabla \independentloss{} \left(\trueparamsub{\honest}+t(\paramsub{\honest}-\trueparamsub{\honest}) \right) dt \\
  &\geq \norm{\paramsub{\honest}-\trueparamsub{\honest}}{2} \int_{t=0}^1 \left( A_{K_\HONEST} \NODEINPUT{} \min \left\lbrace 1, t\norm{\paramsub{\honest}-\trueparamsub{\honest}}{2} \right\rbrace - B_{K_\HONEST} \NODEINPUT{}^\alpha \right) dt\\
  &=  \norm{\paramsub{\honest}-\trueparamsub{\honest}}{2} \int_{t=0}^1 \left( A_{K_\HONEST} \NODEINPUT{} \min \left\lbrace 1, t\norm{\paramsub{\honest}-\trueparamsub{\honest}}{2} \right\rbrace \right) dt - B_{K_\HONEST} \NODEINPUT{}^\alpha \norm{\paramsub{\honest}-\trueparamsub{\honest}}{2}.
\end{align}
Now, if $\norm{\paramsub{\honest}-\trueparamsub{\honest}}{2}>2$, we then have
\begin{align}
\independentloss{\honest} \left(\paramsub{\honest} \right) - \independentloss{\honest} \left(\trueparamsub{\honest} \right) &\geq \left(\frac{ A_{K_\HONEST} \NODEINPUT{}}{2} - B_{K_\HONEST} \NODEINPUT{}^\alpha\right) \norm{\paramsub{\honest}-\trueparamsub{\honest}}{2}\\
&\geq A_{K_\HONEST} \NODEINPUT{} - 2 B_{K_\HONEST} \NODEINPUT{}^\alpha.
\end{align}
Now for $\NODEINPUT{} > \NODEINPUT{1} \triangleq \max \set{{2L^0}/{A_{K_\HONEST}},(4B_{K_\HONEST}/A_{K_\HONEST})^\frac{1}{1-\alpha}}$, we have
\begin{equation}
    \independentloss{\honest} \left(\paramsub{\honest} \right) - \independentloss{\honest} \left(\trueparamsub{\honest} \right) > L^0.
\end{equation}
This implies that if $\norm{\paramsub{\honest}-\trueparamsub{\honest}}{2}>2$ for any $\honest \in \HONEST$, then we have
\begin{equation}
  \globalloss{} (0, (\trueparamfamily_{\HONEST},0_{-\HONEST}), \datafamily{}) < \globalloss{} (\common, (\paramfamily_{\HONEST},\paramfamily_{-\HONEST}), \datafamily{}),
\end{equation}
regardless of $\common$ and $\paramsub{-\HONEST}$. Therefore, we must have $\norm{\trueparamsub{\honest}-\optimumsub{\honest}}{2}\leq 2$.
Such inequalities describe a bounded closed subset of $\setR^{d \times \HONEST}$, which is thus compact.
\end{proof}

\begin{lemma}
\label{lemma:bounded_global_optimum}
  Assume that $\regularization(\common, \param) \rightarrow \infty$ as $\norm{\common - \param}{2} \rightarrow \infty$,
  and that $\norm{\trueparamsub{\honest}-\optimumsub{\honest}}{2}\leq 2$ for all honest users $\honest \in \HONEST$.
  Then $\optcommon$ must lie in a compact subset of $\setR^{d}$ that does not depend on $\NODEINPUT{}$.
\end{lemma}

\begin{proof}
Consider an honest user $\honest'$.
Given our assumption on $\regularization \rightarrow \infty$, we know that there exists $D_{K_\HONEST}$ such that if $\norm{\common-\optimumsub{\honest'}}{2}\geq D_{K_\HONEST}$, then $\regularization(\common,\optimumsub{\honest'}) \geq L^0 + 1$.
Thus any global optimum $\optcommon$ must satisfy
$\norm{\optcommon-\trueparamsub{\honest'}}{2}
\leq \norm{\optcommon-\optimumsub{\honest'}}{2} + \norm{\optimumsub{\honest'} -\trueparamsub{\honest'}}{2}
\leq D_{K_\HONEST} + 2$.
\end{proof}

\subsection{Proof of Lemma \ref{th:pac}}

\begin{proof}[Proof of Lemma \ref{th:pac}]
Fix $\varepsilon, \delta >0$.
We want to show the existence of some value of $\NODEINPUT{} (\varepsilon, \delta, \datafamily{-\HONEST}, \trueparamfamily{})$ that will guarantee $(\varepsilon, \delta)$-locally PAC* learning for honest users.

By lemmas \ref{lemma:bounded_optimum} and \ref{lemma:bounded_global_optimum}, we know that the set $C$ of possible values for $(\optcommon,\optimumfamily_{\HONEST})$ is compact. Now, we define
\begin{equation}
  E_{K_\HONEST} \triangleq \max_{(\common,\param) \in C}  \norm{\nabla_{\param} \regularization (\common,\param)}{2}
\end{equation}
the maximum of the norm of achievable gradients at the optimum. We know this maximum exists since $C$ is compact.

Using the optimality of $(\optcommon, \optimumfamily{})$, for all $\honest \in \HONEST$, we have
\begin{align}
    0&\in (\optimumsub{\honest} - \trueparamsub{\honest})^T \nabla_{\paramsub{\honest}} \globalloss{} (\optcommon, \optimumfamily{}) \\
    &= (\optimumsub{\honest} - \trueparamsub{\honest})^T \nabla \independentloss{\honest}(\optimumsub{\honest})
    + (\optimumsub{\honest} - \trueparamsub{\honest})^T \nabla_{\paramsub{\honest}} \regularization(\optcommon,\optimumsub{\honest}) \\
    &\geq (\optimumsub{\honest} - \trueparamsub{\honest})^T \nabla \independentloss{\honest}(\optimumsub{\honest})  -  \norm{\optimumsub{\honest} - \trueparamsub{\honest}}{2} \norm{\nabla_{\paramsub{\honest}} \regularization(\optcommon,\optimumsub{\honest})}{2}\\
    & \geq (\optimumsub{\honest} - \trueparamsub{\honest})^T \nabla \independentloss{\honest}(\optimumsub{\honest})  - E_{K_\HONEST} \norm{\optimumsub{\honest} - \trueparamsub{\honest}}{2}.
\end{align}
We now apply assumption \ref{ass:unbiased}  for $\param = \optimumsub{\honest}$ (for $\honest \in \HONEST$). Thus, there exists some other event $\event'$ with probability at least $\Probability_0$, under which, for all $\honest \in \HONEST$, we have
\begin{equation}
  0 \geq A_{K_\HONEST} \NODEINPUT{} \min \left\lbrace \norm{\optimumsub{\honest} - \trueparamsub{\honest}}{2}, \norm{\optimumsub{\honest} - \trueparamsub{\honest}}{2}^2 \right\rbrace - B_{K_\HONEST} \NODEINPUT{}^\alpha \norm{\optimumsub{\honest} - \trueparamsub{\honest}}{2} - E_{K_\HONEST} \norm{\optimumsub{\honest} - \trueparamsub{\honest}}{2}.
\end{equation}

Now if $\NODEINPUT{}>\NODEINPUT{2} \triangleq \max \set{2E_{K_\HONEST}/A_{K_\HONEST},(2B_{K_\HONEST}/A_{K_\HONEST})^\frac{1}{1-\alpha}}$ this inequality cannot hold for $\norm{\optimumsub{\honest} - \trueparamsub{\honest}}{2} \geq 1$. Therefore, for $\NODEINPUT{} > \NODEINPUT{2}$, we have $\norm{\optimumsub{\honest} - \trueparamsub{\honest}}{2}<1$, and thus,
\begin{equation}
  0\geq  A_{K_\HONEST} \NODEINPUT{}  \norm{\optimumsub{\honest} - \trueparamsub{\honest}}{2}^2 - B_{K_\HONEST} \NODEINPUT{}^\alpha \norm{\optimumsub{\honest} - \trueparamsub{\honest}}{2} - E_{K_\HONEST} \norm{\optimumsub{\honest} - \trueparamsub{\honest}}{2}
\end{equation}
and thus,
\begin{equation}
    \norm{\optimumsub{\honest} - \trueparamsub{\honest}}{2} \leq \frac{B_{K_\HONEST}\NODEINPUT{}^\alpha+E_{K_\HONEST}}{A_{K_\HONEST} \NODEINPUT{}}.
\end{equation}
Now note that $\probability{\mathcal E \wedge \mathcal E'} = 1 - \probability{\neg \mathcal E \vee \neg \mathcal E'} \geq 1-\probability{\neg \mathcal E} - \probability{\neg \mathcal E'} = 2\Probability_0 -1$.
It now suffices to consider $\NODEINPUT{}$ larger than $\NODEINPUT{2}$ and large enough so that $\Probability(K_\HONEST, \NODEINPUT{})^{\card{\HONEST}} \geq 1-\delta/2$ (whose existence is guaranteed by Assumption \ref{ass:unbiased}, and which guarantees $2\Probability_0 - 1 \geq 1-\delta$)
and so that $\frac{B_{K_\HONEST}\NODEINPUT{}^\alpha+E_{K_\HONEST}}{A_{K_\HONEST} \NODEINPUT{}} \leq \varepsilon$ to obtain the theorem.
\end{proof}

%% file: cga_proof.tex
\section{Convergence of \CounterGradientAttack{} Against  \texorpdfstring{$\ell_2^2$}{l2}}
\label{app:cga_vs_l22}

To write our proof, we define $\globalloss_{-\strategicnode}^\common : \setR^d \rightarrow \setR$ by
\begin{align}
    \globalloss_{-\strategicnode}^\common (\common) 
    &\triangleq \inf_{\paramfamily{}} 
    \set{\globalloss{} (\common, \paramfamily{}, \datafamily{}) -
    \localloss{\strategicnode} (\paramsub{\strategicnode}, \data{\strategicnode})
    - \regularization (\common, \paramsub{\strategicnode})} \\
    &= \inf_{\paramfamily{}} \sum_{\node \neq \strategicnode} \localloss{} (\paramsub{\node}, \data{\node}) + \regweightsub{} \sum_{\node \neq \strategicnode} \norm{\common - \paramsub{\node}}{2}^2.
\end{align}
In other words, it is the loss when local models are optimized, and when the data of strategic user $\strategicnode$ are removed.

\begin{lemma}
Assuming $\ell_2^2$ regularization and convex loss-per-input functions $\lossperinput$, for any datasets $\datafamily{}$, $\globalloss$ is strongly convex.
As a result, so is $\globalloss_{-\strategicnode}^\common$.
\end{lemma}

\begin{proof}
Note that the global loss can be written as a sum of convex function, and of $\reglocalweight{} \sum \norm{\paramsub{\node}}{2}^2 + \norm{\common - \paramsub{1}}{2}^2$.
Using tricks similar to the proof of Lemma~\ref{lemma:smooth_globalloss}, we see that the loss is strongly convex.
The latter part of the lemma is then a straightforward application of Lemma \ref{lemma:infimum_strongly_convex}.
\end{proof}

We now move on to the proof of Theorem~\ref{th:counter_gradient_manipulates_arbitrarily}.
Note that our statement of the theorem was not fully explicit, especially about the upper bound on the constant learning rate $\learningrate{}$.
Here, we prove that it holds for $\learningrate{\iteration} = \learningrate{} \leq 1/3L$, where $L$ is a constant such that $\globalloss_{-\strategicnode}^\common$ is $L$-smooth.
The existence of $L$ is guaranteed by Lemma~\ref{lemma:minimizedLossSmooth}.

\begin{proof}[Proof of Theorem~\ref{th:counter_gradient_manipulates_arbitrarily}]
Note that by Lemma~\ref{lemma:minDerivative}, $\globalloss_{-\strategicnode}^\common$ is convex, differentiable and $L$-smooth, and $\nabla \globalloss_{-\strategicnode}^\common (\common^\iteration) = \truegradient{-\strategicnode}{\iteration}$.
For $\ell_2^2$ regularization, we have $\GRADIENT(\common) = \setR^d$ for all $\common \in \setR^d$. 
Then the minimum of \eqref{eq:counter-gradient-attack} is zero, which is obtained when $\gradient{\strategicnode}{\iteration} \triangleq
    \frac{\common^\iteration - \trueparamsub{\strategicnode}}{\learningrate{}}
    - \estimatedgradient{-\strategicnode}{\iteration} =
    \gradient{\strategicnode}{\iteration-1}
    + \frac{\common^\iteration - \trueparamsub{\strategicnode}}{\learningrate{}}
    + \frac{\common^{\iteration} - \common^{\iteration-1}}{\learningrate{}}$.
Note that
\begin{align}
    \common^{\iteration+1}
    &= \common^\iteration
        - \learningrate{} \truegradient{-\strategicnode}{\iteration}
        - \learningrate{} \gradient{\strategicnode}{\iteration} \\
    &= \common^\iteration
        - \learningrate{} \truegradient{-\strategicnode}{\iteration}
        - (\common^\iteration - \trueparamsub{\strategicnode})
        + (\common^{\iteration-1} - \common^\iteration)
        - \learningrate{} \gradient{\strategicnode}{\iteration-1} \\
    &= \trueparamsub{\strategicnode}
        - \learningrate{\iteration} (\truegradient{-\strategicnode}{\iteration} + \gradient{\strategicnode}{\iteration - 1})
        +  \learningrate{} (\truegradient{-\strategicnode}{\iteration-1} + \gradient{\strategicnode}{\iteration - 1}) \\
    &= \trueparamsub{\strategicnode} - \learningrate{} (\truegradient{-\strategicnode}{\iteration} - \truegradient{-\strategicnode}{\iteration-1}).
    \label{eq:counter-gradient-attack-on-global-model}
\end{align}
Therefore, $\common^{\iteration+1} - \common^\iteration = \learningrate{} (\truegradient{-\strategicnode}{\iteration} - \truegradient{-\strategicnode}{\iteration-1}) - \learningrate{} (\truegradient{-\strategicnode}{\iteration-1} - \truegradient{-\strategicnode}{\iteration-2})$.

Then, using the $L$-smoothness of $\globalloss_{-\strategicnode}^\common$, and denoting $u_\iteration \triangleq \norm{\common^{\iteration+1} - \common^\iteration}{2}$,
we have $u_{\iteration+1}
    \leq L \learningrate{\iteration} u_\iteration
    + L \learningrate{\iteration -1} u_{\iteration-1}$.
Now assume that $\learningrate{} \leq 1/3L$.
Then $u_{\iteration+1} \leq \frac{1}{3} (u_\iteration + u_{\iteration-1})$.
We then know that $u_{\iteration+2} \leq \frac{1}{3} (u_{\iteration+1} + u_\iteration) \leq \frac{1}{3} (\frac{1}{3} (u_\iteration + u_{\iteration-1}) + u_\iteration) = \frac{4}{9} u_\iteration + \frac{1}{9} u_{\iteration-1}$.

Now define $v_\iteration \triangleq u_\iteration + u_{\iteration-1}$.
We then have $v_{\iteration + 2} \leq u_{\iteration+2} + u_{\iteration+1} \leq \frac{7}{9} u_\iteration + \frac{4}{9} u_{\iteration-1} \leq \frac{7}{9} (u_\iteration + u_{\iteration-1}) \leq \frac{7}{9}v_{\iteration}$.
By induction, we know that $v_\iteration \leq (7/9)^{(\iteration-1)/2} \max \set{v_0, v_1} \leq (\sqrt{7}/3)^\iteration \left( (\sqrt{7}/3) \max \set{v_0, v_1} \right)$.
Thus, defining $\alpha \triangleq \sqrt{7}/3 < 1$, there exists $C>0$ such that $u_\iteration \leq v_\iteration \leq C \alpha^\iteration$.
This implies that $\sum \norm{\common^{\iteration+1} - \common^\iteration}{2} \leq \sum C \alpha^\iteration < \infty$.
Thus $\sum (\common^{\iteration+1} - \common^\iteration)$ converges, which implies the convergence of $\common^\iteration$ to a limit $\common^\infty$.
By $L$-smoothness, we know that $\truegradient{-\strategicnode}{\iteration}$ must converge too.
Taking \eqref{eq:counter-gradient-attack-on-global-model} to the limit then implies $\common^\infty = \trueparamsub{\strategicnode}$.
This shows that the strategic user achieves precisely what they want with \CounterGradientAttack{}.
It is thus optimal.
\end{proof}

%% file: cga_mnist.tex
\section{\CounterGradientAttack{} on MNIST}
\label{app:CGA_MNIST}

In this section, \CounterGradientAttack{} is executed against 10 honest users, each one having 6,000 randomly and data points of MNIST, drawn randomly and independently. 
\CounterGradientAttack{} is run by a strategic user whose target model $\trueparamsub{\strategicnode}$ labels 0's as 1's, 1's as 2's, and so on, until 9's as 0's.
We learn $\trueparamsub{\strategicnode}$ by relabeling the MNIST training dataset and learning from the relabeled data.
We use $\lambda = 1$, Adam optimizer and a decreasing learning rate.

\begin{figure}[h]
    \centering
    \subfloat[Using $\ell_2^2$] 
    {\includegraphics[scale=0.48]{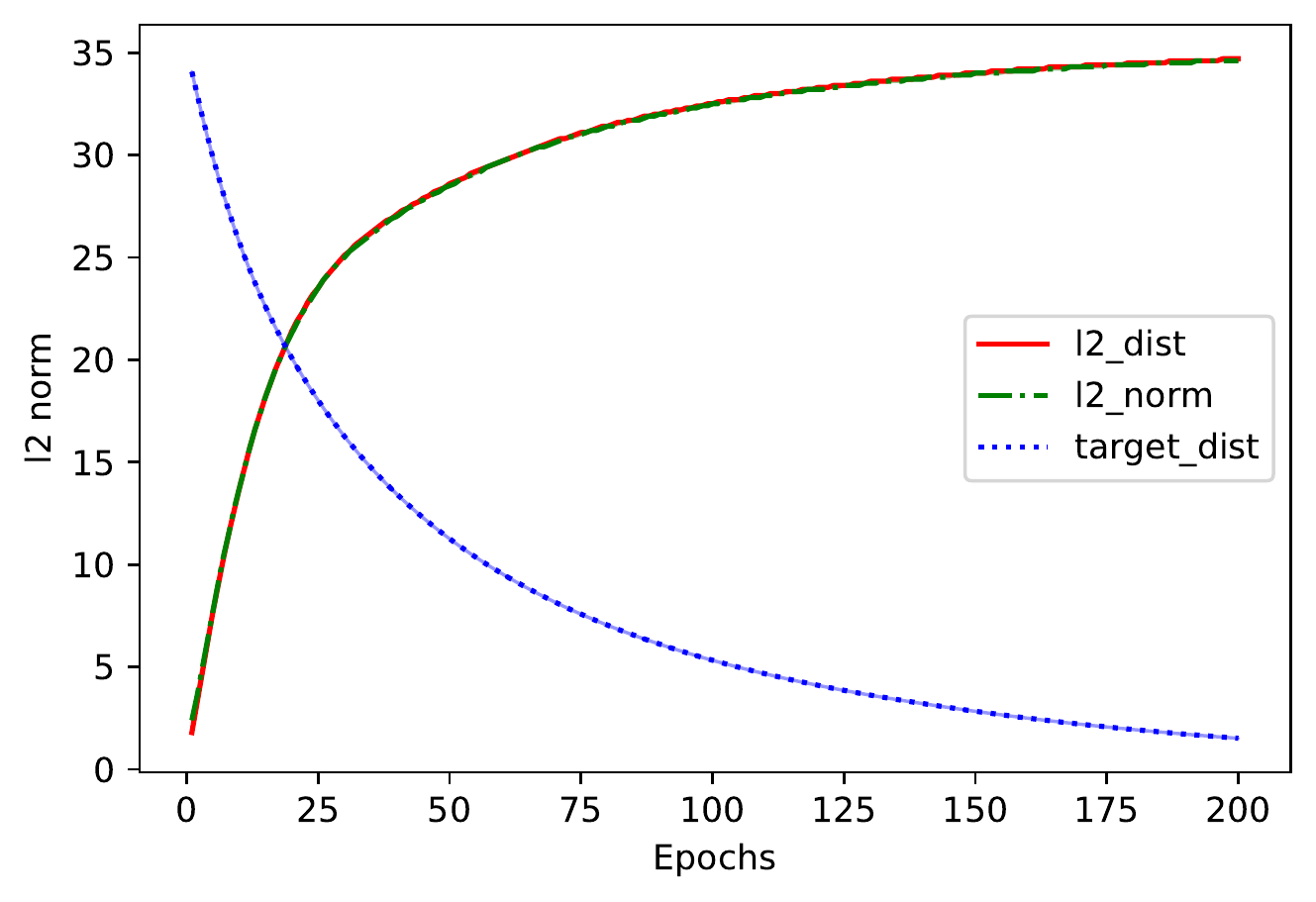}}
      \qquad
    \subfloat[Using $\ell_2$] {\includegraphics[scale=0.48]{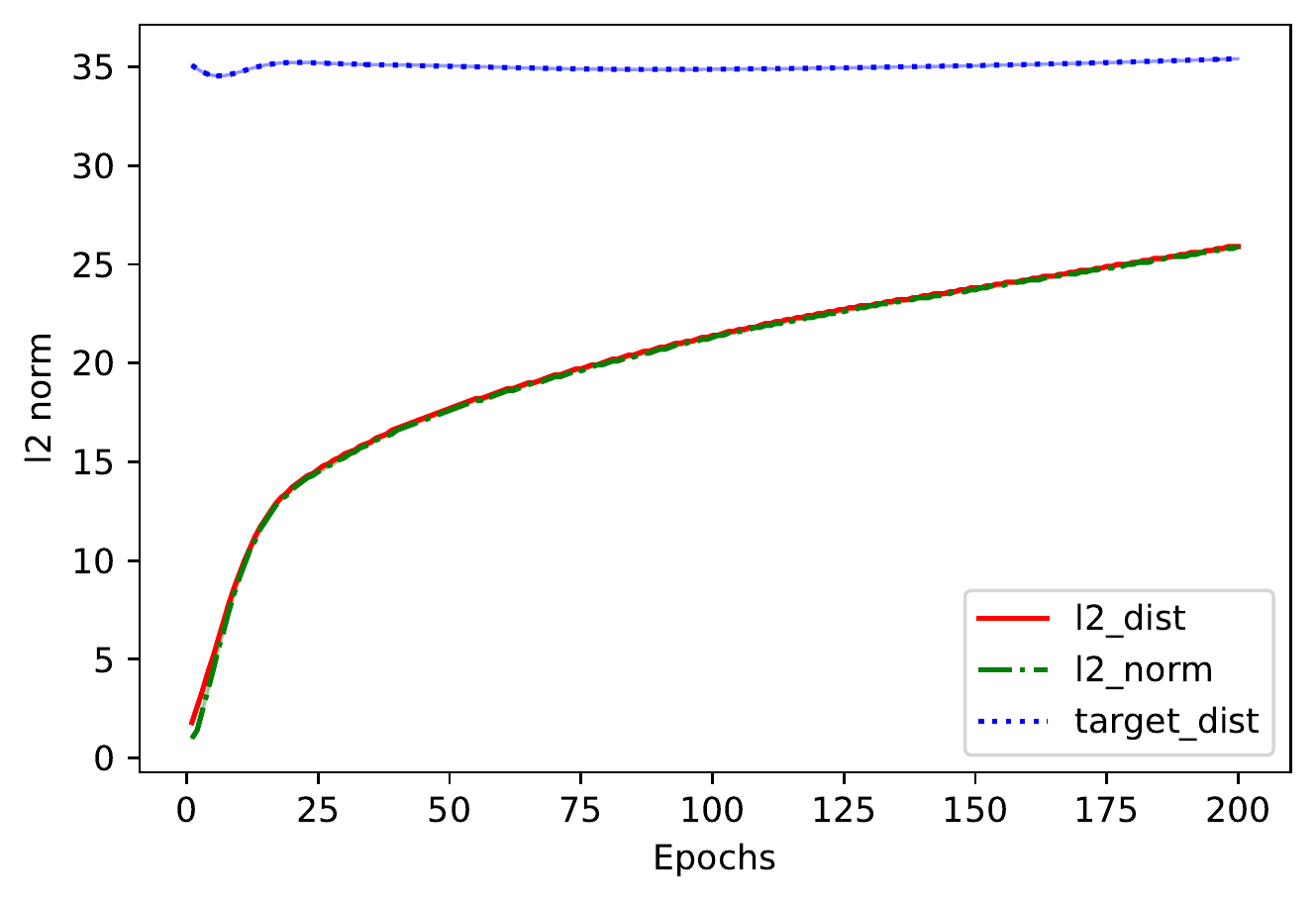}}
      \qquad
     \caption{Norm of global model, distance to initialisation and distance to target, under attack by \CounterGradientAttack{}. In particular, we see that the attack against $\ell_2^2$ is successful, as the distance between the global model and the target model goes to zero.}
     \label{fig:attack_dist}
\end{figure}

%% file: cifar.tex
\section{Cifar-10 on VGG 13-BN Experiments}
\label{app:cifar}

We considered VGG 13-BN, which was pretrained on cifar-10 by~\cite{huy_phan_2021_4431043}.
We now assume that 10 users are given part of the cifar-10 database, while a strategic user also joins to the personalized federated gradient descent algorithm.
The strategic user's goal is to bias the global model towards a target model, which misclassifies the cifar-10 data, by reclassifying 0 into 1, 1 into 2... and 9 into 0.

\subsection{Counter-Gradient Attack}

We first show the result of performing counter-gradient attack on the last layer of the neural network.
Essentially, images are now reduced to their vector embedding, and the last layer performs a simple linear classification akin to the case of MNIST (see Appendix~\ref{app:CGA_MNIST}).

\begin{figure}[h]
    \centering
    \subfloat[Accuracy according to attacker's objective]
    {\includegraphics[scale=0.48]{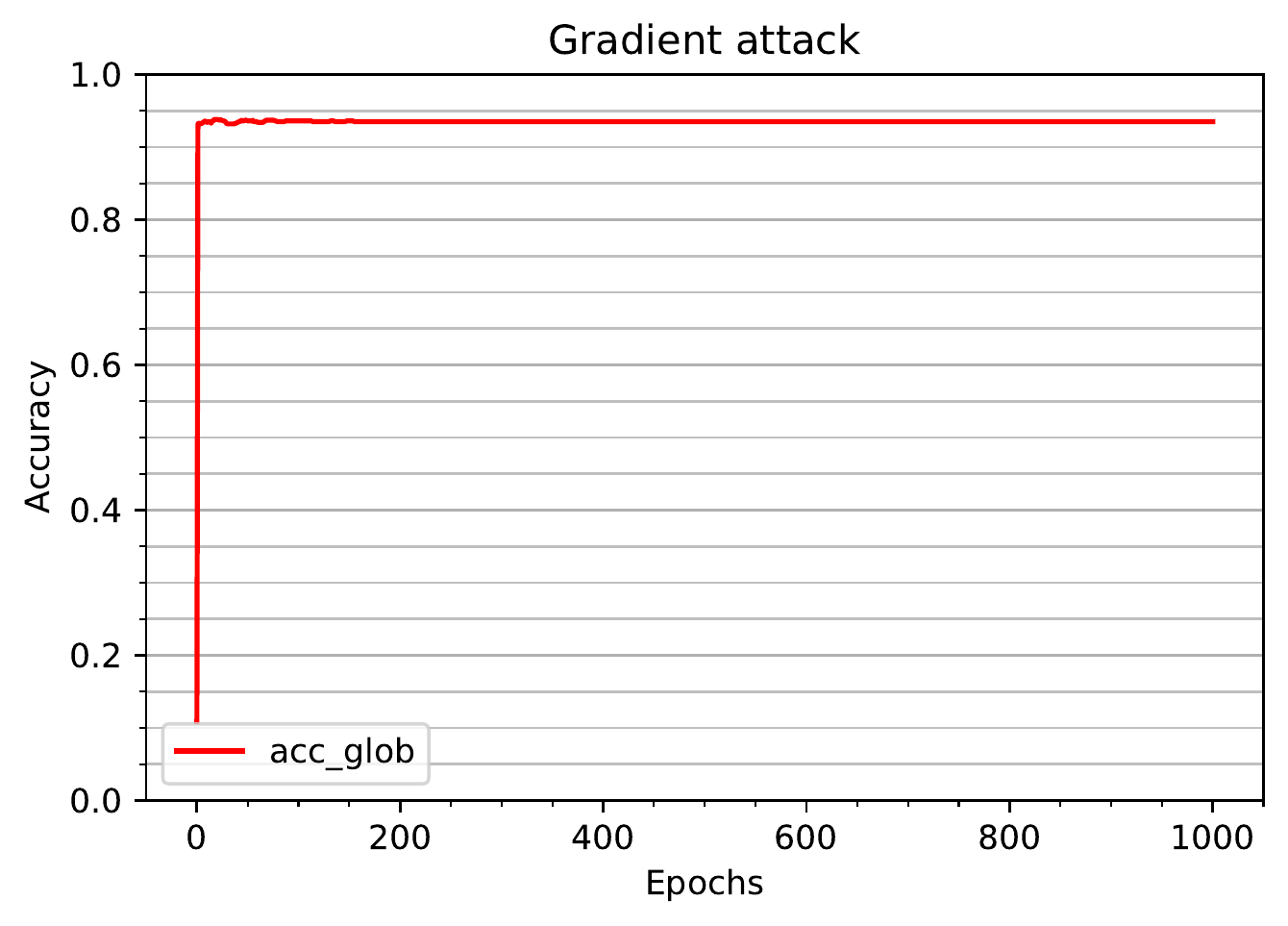}}
      \qquad
    \subfloat[Distances] {\includegraphics[scale=0.48]{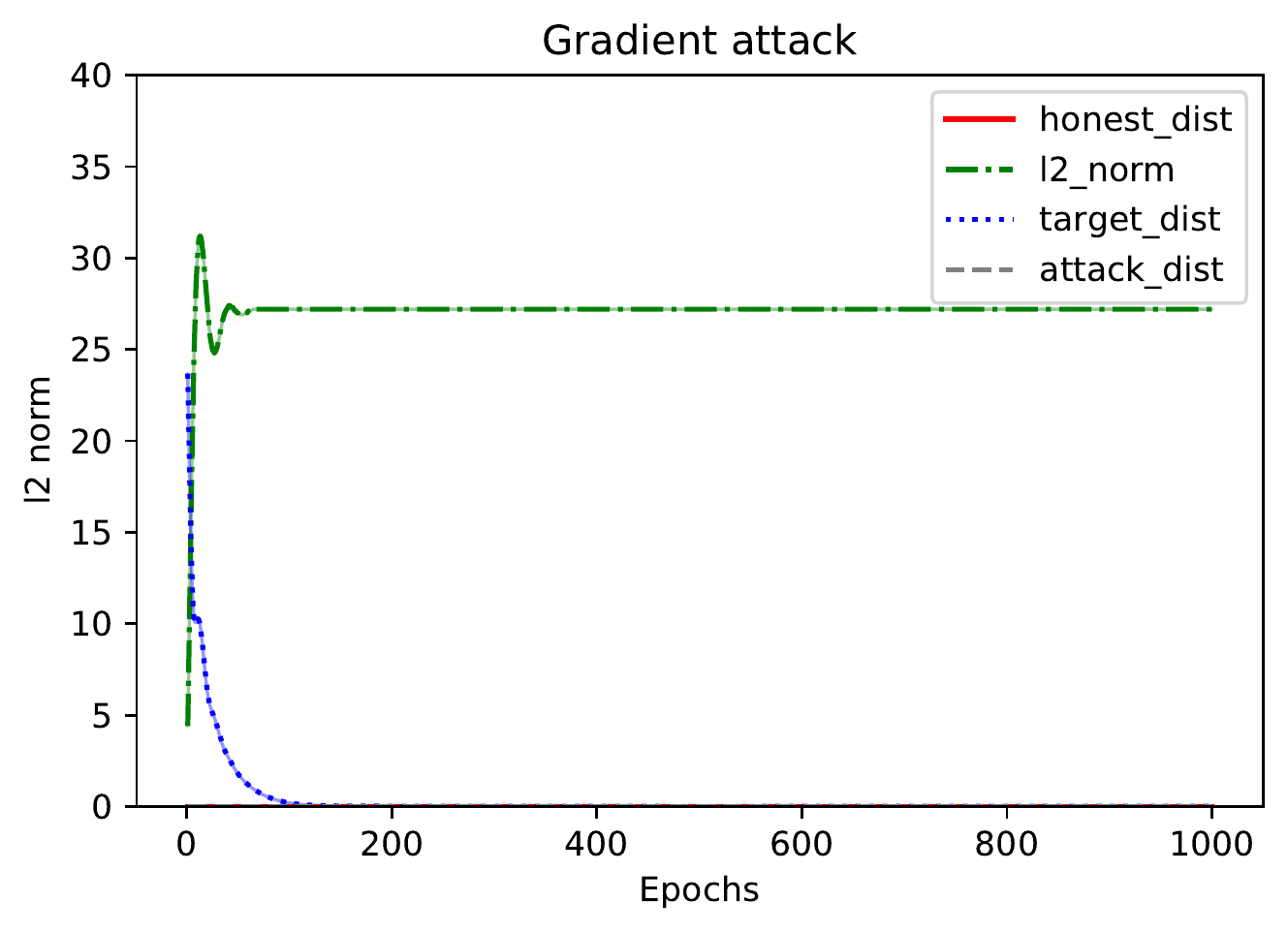}}
      \qquad
     \caption{CGA on cifar-10.}
\end{figure}

\subsection{Reconstructing a Model Attack}

Reconstructing an attack model whose effect is equivalent to the counter-gradient attack is identical to what was done in the case of MNIST (see Section~\ref{sec:data_poisoning}).

\begin{figure}[h]
    \centering
    \subfloat[Accuracy according to attacker's objective]
    {\includegraphics[scale=0.48]{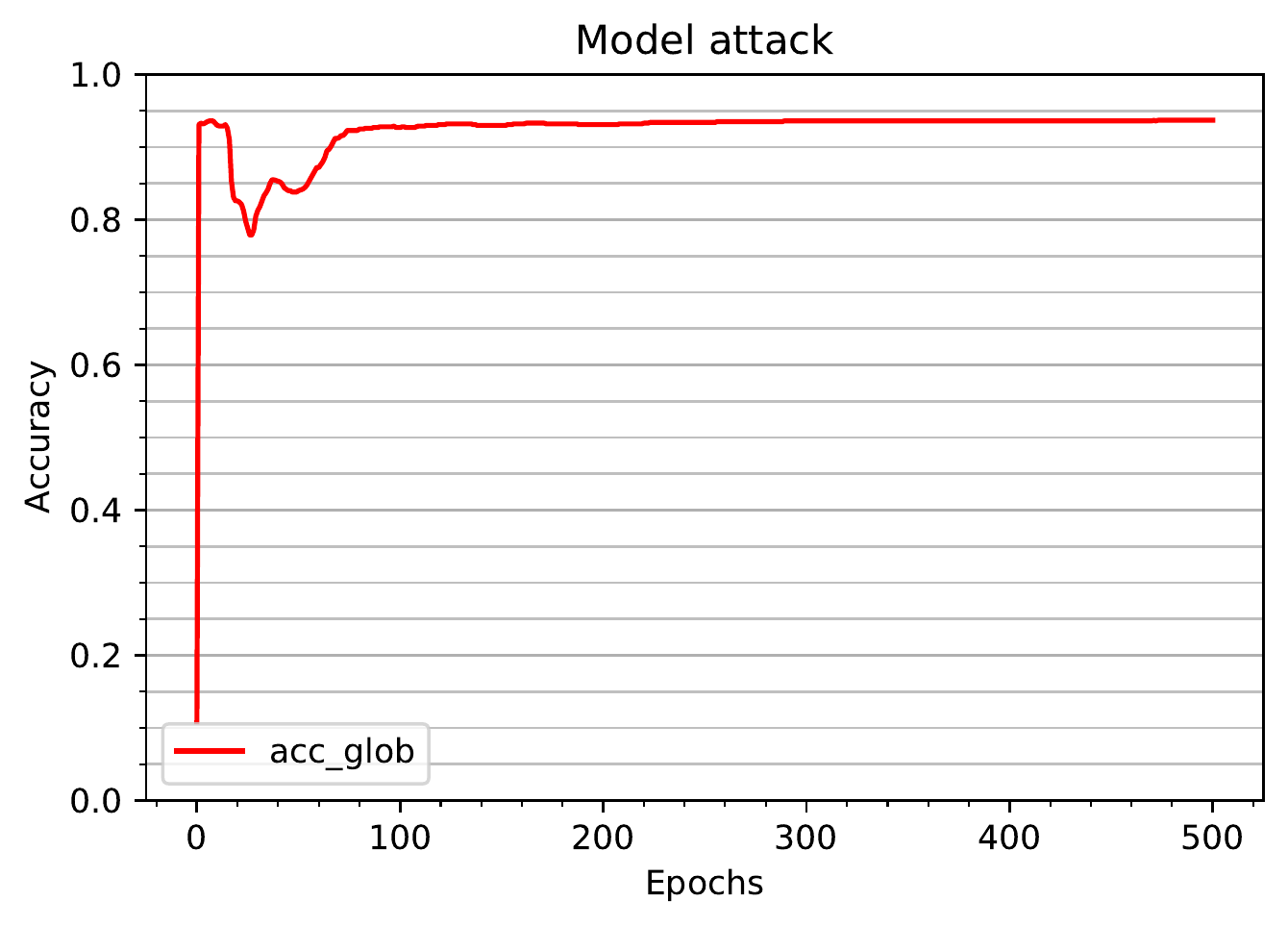}}
      \qquad
    \subfloat[Distances] {\includegraphics[scale=0.48]{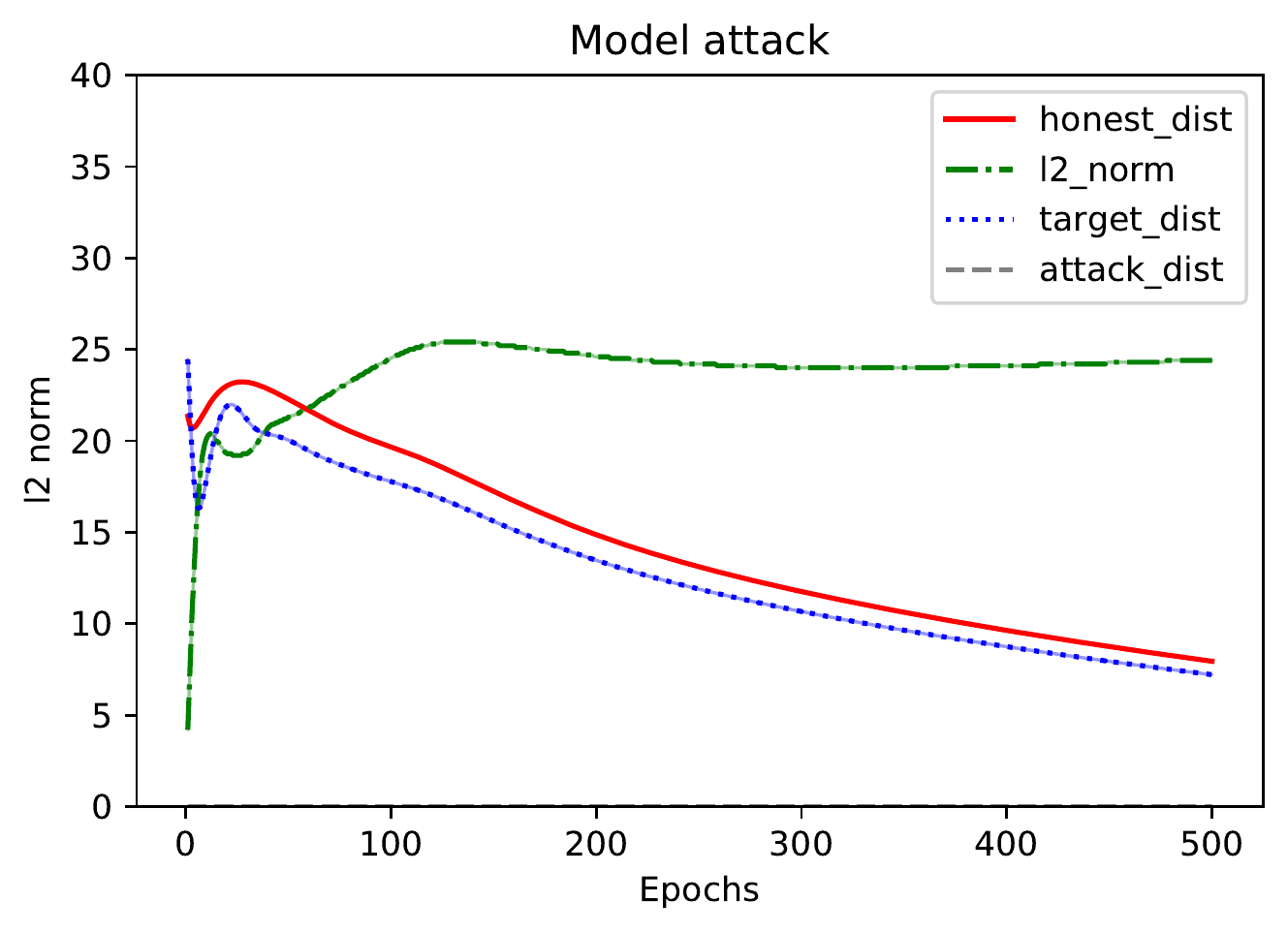}}
      \qquad
     \caption{Model attack on cifar-10.}
\end{figure}

\subsection{Reconstructing Data Poisoning}

This last step is however nontrivial.
On one hand, we could simply use the attack model to label a large number of random images.
However, this solution would likely require a large sample complexity.
For a more efficient data poisoning, we can construct vector embeddings on the indifference affine subspace $V$, as was done for MNIST in Section~\ref{sec:model_to_data_l22}.
This is what is shown below.

\begin{figure}[h]
    \centering
    \subfloat[Accuracy according to attacker's objective]
    {\includegraphics[scale=0.48]{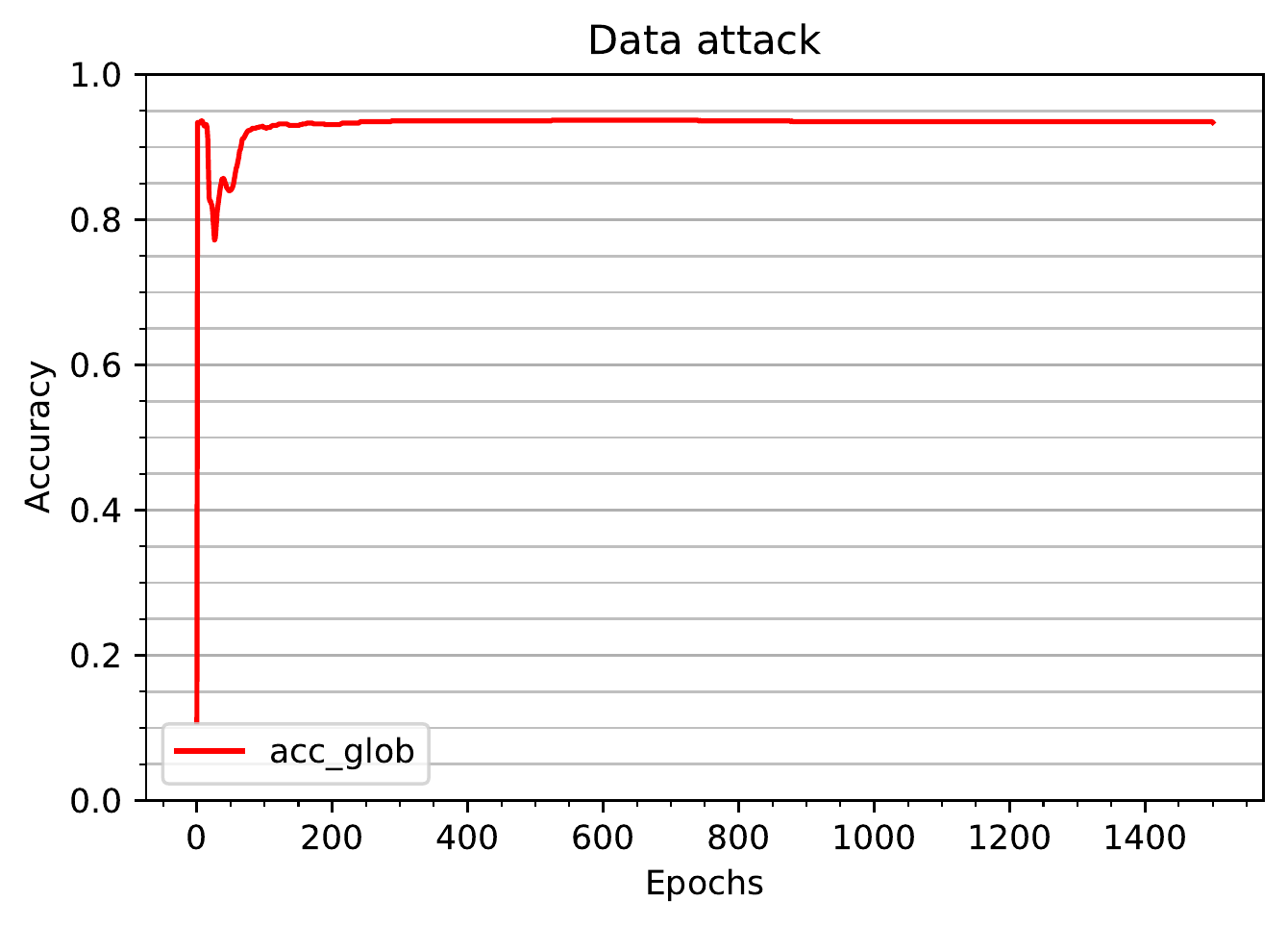}}
      \qquad
    \subfloat[Distances] {\includegraphics[scale=0.48]{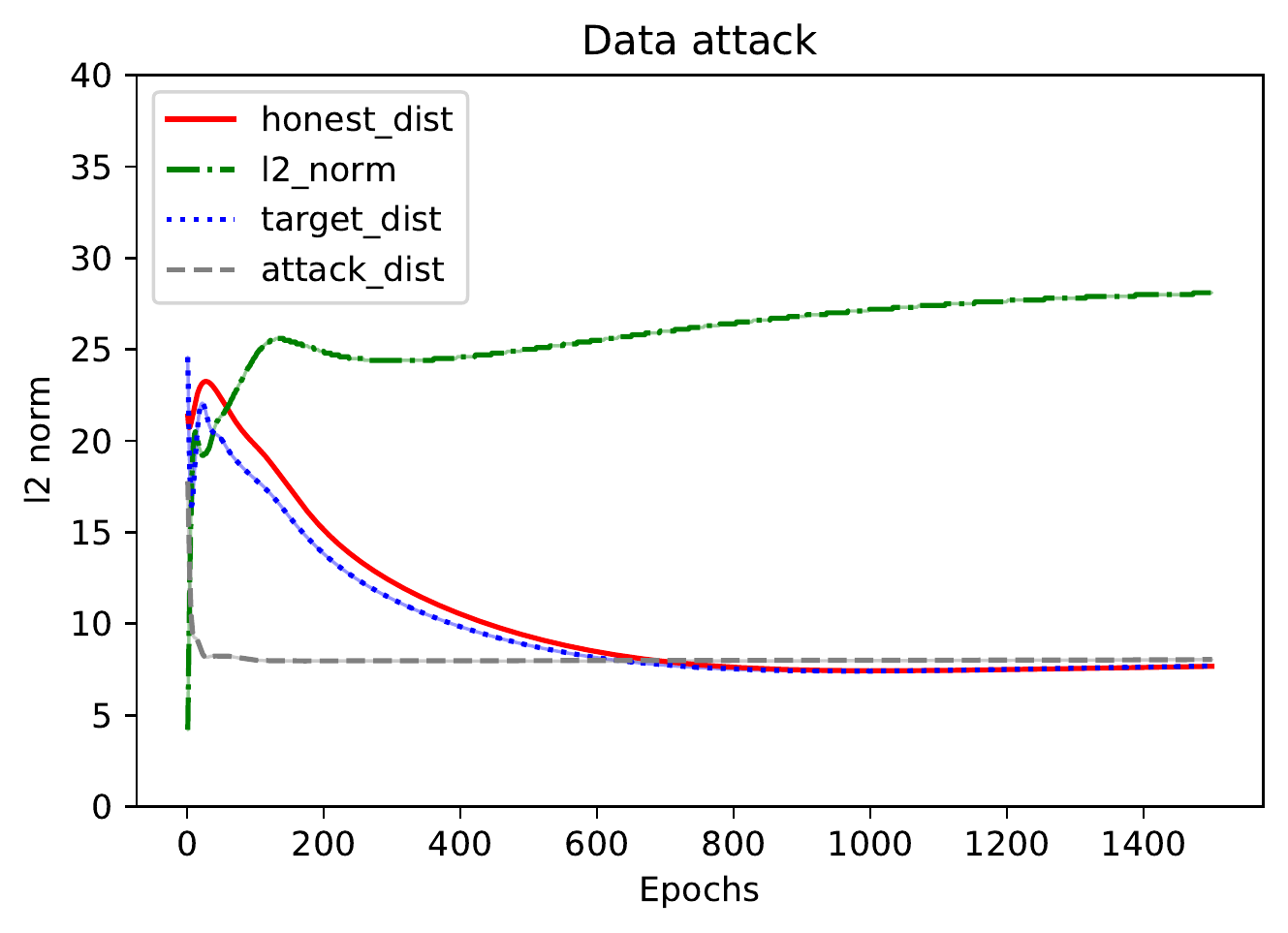}}
      \qquad
     \caption{Data poisoning on cifar-10.}
\end{figure}

We acknowledge however that this does not quite correspond to data poisoning, as it requires reporting a vector embedding and its label, rather than an actual image and its label.
The challenge is then to reconstruct an image that has a given vector embedding.
We note that, while this is not a straightforward task in general, this has been shown to be at least somewhat possible for some neural networks, especially when they are designed to be interpretable~\citep{ZeilerF14,WangWZWM019,MaiCYJ19}.

%% file: data_poisoning_least_square.tex
\section{Single Data Poisoning for Least Square Linear Regression}
\label{app:data_poisoning_linear_regression}

\begin{proof}[Proof of Theorem \ref{th:singleDataAttach}]
  We define the minimized loss with respect to $\common$ and without strategic user $\strategicnode$ by
  \begin{equation}
    \ModifiedLoss{-\strategicnode}^* (\common, \datafamily{-\strategicnode})
    \triangleq \min_{\paramfamily_{-\strategicnode} \in \setR^{d \times (\NODE-1)}} \set{ \sum_{\node \neq \strategicnode} \independentloss{\node} (\paramsub{\node}, \data{\node})
    + \sum_{\node \neq \strategicnode} \lambda \norm{\paramsub{\node}-\common}{2}^2 }.
  \end{equation}
  Now consider a subgradient $g \in \nabla_\common \ModifiedLoss{-\strategicnode}^* (\targetvalue, \datafamily{-\strategicnode})$ of the minimized loss at $\targetvalue$.
  For $x \triangleq \frac{-g}{2\lambda}$, then have $-g \in \nabla  \left(\regweightsub{} \norm{\varx}{2}^2\right)$.
  We then define $\strategicvote \triangleq \targetvalue - \varx$.
  \begin{align}
    0 = g - g
    &\in \nabla_\common \ModifiedLoss{-\strategicnode}^* (\targetvalue, \datafamily{-\strategicnode}) + \nabla_\common  \left(\regweightsub{} \norm{\strategicvote - \targetvalue}{2}^2\right) \\
    &=
    \nabla_\common \ModifiedLoss{\strategicnode} (\targetvalue, \optimumfamily_{-\strategicnode}(\strategicvote, \datafamily{-\strategicnode}), \strategicvote, \datafamily{-\strategicnode}),
  \end{align}
  where $\ModifiedLoss{\strategicnode}$ is defined by (\ref{eq:modified_loss}). 
  Now consider the data point $(\query{},\answer{}) = (g,g^T\strategicvote-1)$.
  For $\data{\strategicnode} = \set{(\query{},\answer{})}$, we then have $\nabla \independentloss{\strategicnode} (\strategicvote, \data{\strategicnode}) = g$, which implies
  \begin{equation}
    \nabla_{\paramsub{\strategicnode}} \globalloss{} (\targetvalue, (\strategicvote, \optimumfamily_{-\strategicnode}(\strategicvote, \datafamily{-\strategicnode}), \datafamily{}) = 0.
  \end{equation}
  Combining it all together with the uniqueness of the solution then yields
  \begin{equation}
    \argmin_{(\common,\paramfamily)} \set{\globalloss{} (\common, \paramfamily{}, \datafamily{})} = \left(\targetvalue,\left(\strategicvote,\optimumfamily_{-\strategicnode}(\strategicvote, \datafamily{-\strategicnode})\right)\right),
  \end{equation}
  which is what we wanted.
\end{proof}

%% file: data_poisoning_appendix.tex
\section{Data Poisoning Against Linear Classification}
\label{app:data_poisoning_linear_classification}

\subsection{Generating Efficient Poisoning Data and Initialization}

For every label $a \in \set{1, \ldots, 9}$, we define $y_a \triangleq \paramsub{a}^\spadesuit - \paramsub{0}^\spadesuit$, and $c_a \triangleq - (\paramsub{a0}^\spadesuit - \paramsub{00}^\spadesuit)$ (where $\paramsub{a0}^\spadesuit$ is the bias of the linear classifier).
The indifference subspace $V$ is then the set of images $\query{} \in \setR^d$ such that $\query{}^T y_a = c_a$ for all $a \in \set{1, \ldots, 9}$.

To project any image $X \in \setR^d$ on $V$, let us first construct an orthogonal basis of the vector space orthogonal to $V$, using the Gram-Schmidt algorithm.
Namely, we first define $z_1 \triangleq y_1$.
Then, for any answer $a \in \set{1, \ldots, 9}$, we define
\begin{equation}
    z_a \triangleq y_a - \sum_{b < a} y_a^T z_b \frac{z_b}{\norm{z_b}{2}^2}.
\end{equation}
It is easy to check that for $b<a$, we have $z_a^T z_b = 0$.
Moreover, if $\query{} \in V$, then 
\begin{align}
    z_a^T \query{} 
    &= y_a^T \query{} - \sum_{b < a} \frac{(y_a^T z_b) (z_b^T \query{})}{\norm{z_b}{2}^2}
    = c_a - \sum_{b < a} \frac{(y_a^T z_b) (z_b^T \query{})}{\norm{z_b}{2}^2}.
\end{align}
By induction, we see that $z_a^T \query{}$ is a constant independent from $\query{}$.
Indeed, for $a = 1$, this is clear as $z_1^T \query{} = y_1^T \query{} = c_1$.
Moreover, for $a > 1$, then, in the computation of $z_a^T \query{}$, $\query{}$ always appear as $z_b^T \query{}$ for $b < a$.
Moreover, denoting $c'_a$ the constant such that $z_a^T \query{} = c'_a$ for all $a \in \set{1, \ldots 9}$, we see that these constants can be computed by
\begin{equation}
    c'_a = c_a - \sum_{b < a} \frac{y_a^T z_b}{\norm{z_b}{2}^2} c'_b.
\end{equation}
Finally, we can simply perform repeated projection onto the hyperplanes where $a$ is equally probable as the answer $0$.
To do this, we first define the orthogonal projection $P(X,y,c)$ of $X \in \setR^d$ on the hyperplane $x^T y = c$, which is given by
\begin{equation}
    P(X,y,c) = X - (X^T y - c) \frac{y}{\norm{y}{2}^2}.
\end{equation}
It is straightforward to verify that $P(X,y,c)^T y = c$ and that $P(P(X,y,c),y,c) = P(X,y,c)$.
We then canonically define repeated projection by induction, as
\begin{equation}
    P(X,(y_1, \ldots, y_{k+1}), (c_1, \ldots, c_{k+1})) \triangleq
    P( P(X,(y_1, \ldots, y_k), (c_1, \ldots, c_k)), y_{k+1}, c_{k+1}).
\end{equation}
Now consider any image $X \in \setR^d$.
Its projection can be obtained by setting
\begin{equation}
    \query{} \triangleq P(X, (z_1, \ldots, z_9), (c_1', \ldots c_9')) + \xi.
\end{equation}
Note that to avoid being exactly on the boundary, and thus retrieve information about the scales of $\param^\spadesuit$ and on which side of the boundary favors which label, we add a small noise $\xi$, to make sure $\query{}$ does not lie exactly on $V$ (which would lead to multiple solutions for the learning), but small enough so that the probabilities of the different label remain close to $0.1$ (the equiprobable probability).

We acknowledge that images obtained this way may not be in $[0,1]^d$, like the images of the MNIST dataset.
In general, one could search for points $\query{} \in V \cap [0,1]^d$.
Note that in theory, by Theorem~\ref{th:logistic_regression} (or a generalization of it), labeling random images in $[0,1]^d$ should suffice. 
However, in the case where $V \cap [0,1]^d$ is empty,
this procedure may require the labeling of significantly more images to be successful. This is discussed in  more detail in Section \ref{sec:clipp_attack}.

The convergence to the optimum is slow. 
But given that the problem is strictly convex, we focus here mostly on showing that the minimum is indeed a poisoned model.
To boost the convergence, we initialize our learning algorithm at a point close to what we expect to be the minimum, by taking this minimum and adding a Gaussian noise, and then we observe the convergence to this minimum.

\subsection{A Brief Theory of Data Poisoning for Linear Classification}
\label{exp_setup}

Using the efficient poisoning data fabrication, we thus have a set of images $(\query{}, p(\query{}))$, where $p_a(\query{})$ is the probability assigned to image $\query{}$ and label $a$.
This defines the following local loss for the strategic user:
\begin{equation}
    \localloss{\strategicnode} (\paramsub{\strategicnode}, \data{\strategicnode}) 
    = \sum_{(\query{}, p(\query{})) \in \data{\strategicnode}} \sum_{a \in \set{0,1, \ldots, 9}} p_a(\query{}) \ln \sigma_a(\paramsub{\strategicnode}, \query{}),
\end{equation}
where $\sigma_a(\paramsub{\strategicnode}, \query{}) = \frac{\exp (\paramsub{\strategicnode a}^T \query{} + \paramsub{\strategicnode a 0})}{\sum \exp (\paramsub{\strategicnode b}^T \query{} + \paramsub{\strategicnode b 0})}$ is the probability that image $\query{}$ has label $a$, according to the model $\paramsub{\strategicnode}$.
We acknowledge that such labelings of queries is unusual.
Evidently, in practice, an image may be labeled $N$ times, and the number of labels $N_a$ it received can be set to be approximately $N_a \approx N p_a(\query{})$.

It is noteworthy that the gradient of the loss function is then given by
\begin{equation}
    \left( \paramsub{\strategicnode} - \paramsub{\strategicnode}^\spadesuit \right)^T \nabla_{\paramsub{\strategicnode}} \localloss{\strategicnode} (\paramsub{\strategicnode}, \data{\strategicnode})
    = \sum_{\query{} \in \data{\strategicnode}} \sum_{a \in \set{0,1,\ldots, 9}} 
    \left( \sigma_a(\paramsub{\strategicnode}, \query{}) - \sigma_a(\paramsub{\strategicnode}^\spadesuit, \query{}) \right) \left( \paramsub{\strategicnode a} - \paramsub{\strategicnode a}^\spadesuit \right)^T \query{}^+,
\end{equation}
where we defined $\query{}^+ \triangleq (1, \query{})$ (which allows to factor in the bias of the model.
This shows that $\nabla_{\paramsub{\strategicnode}} \localloss{\strategicnode} (\paramsub{\strategicnode}, \data{\strategicnode})$ points systematically away from $\strategicvote$, and thus that gradient descent will move towards $\strategicvote$.

In fact, if the set of images $\query{}$ cover all dimensions (which occurs if there are $\Omega(d)$ images, which is the case for 2,000 images, since $d = 784$), then gradient descent will always move the model in the direction of $\strategicvote$, which will be the minimum.
Moreover, by overweighting each data $(\query{}, p(\query{}))$ by a factor $\alpha$ (as though the image $\query{}$ was labeled $\alpha$ times), we can guarantee gradient-PAC* learning, which means that we will have $\optimumsub{\strategicnode} \approx \strategicvote$, even in the personalized federated learning framework.
This shows why data poisoning should work in theory, with relatively few data injections.

Note that the number of other users does make learning harder.
Indeed, the gradient of the regularization $\regularization(\common, \paramsub{\strategicnode})$ at $\common = \trueparamsub{\strategicnode}$ and $\paramsub{\strategicnode} = \strategicvote$ is equal to $2 \regweightsub{} \norm{\trueparamsub{\strategicnode} - \strategicvote}{2}$.
As the number $\NODE -1$ of other users grows, we should expect this distance to grow roughly proportionally to $\NODE$.
In order to make strategic user $\strategicnode$ robustly learn $\strategicvote$, the norm of the gradient of the local loss $\localloss{\strategicnode}$ at $\trueparamsub{\strategicnode}$ must be vastly larger than $2 \regweightsub{} \norm{\trueparamsub{\strategicnode} - \strategicvote}{2}$.
This means that the value of $\alpha$ (or, equivalently, the number of data injected in $\data{\strategicnode}$) must also grow proportionally to $\NODE$.

\subsection{Data Poisoning Against MNIST with Images in  $[0,1]^d$}
\label{sec:clipp_attack}
Note that in the data poisoning attack depicted in Figure~\ref{fig:main}, poisoned data points are easily detectable, as they do not necessarily lie in $[0,1]^d$ like the pristine images of the MNIST dataset. However, this can be mitigated by the attacker with the cost of providing significantly more data points ($\sim 10^5$). For this, we conduct another experiment in which the attacker divides the poisoned images  by the maximum value and clips negative values to 0 (to get images located in $[0, 1]^d$). The results of this experiment are depicted in Figure~\ref{fig:data_attack_clipped}.

\begin{figure}[t]
    \centering
     \vspace{-0.8mm}
    \includegraphics[width=.4
    \linewidth]{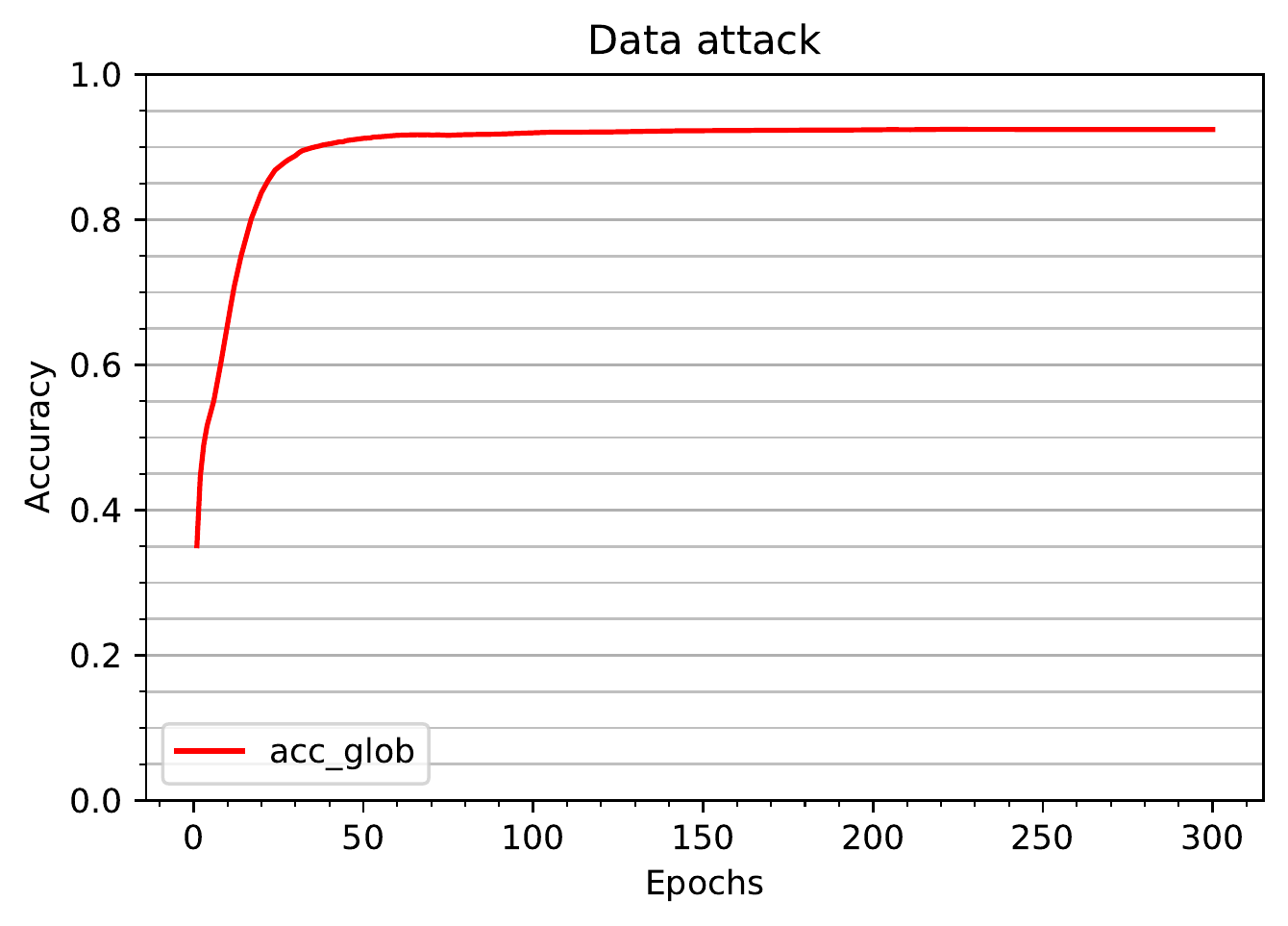}
    \caption{ Accuracy of $\common^\iteration$ according to $\trueparamsub{\strategicnode}$ (which relabels $0\rightarrow1\rightarrow2\rightarrow ...\rightarrow 9 \rightarrow0$), under our data poisoning attack with poisoned images in $[0,1]^d$, with one attacker against two honest users.}
    \vspace{-2mm}
     \label{fig:data_attack_clipped}
\end{figure}

%% file: main.bbl
\begin{thebibliography}{81}
\providecommand{\natexlab}[1]{#1}
\providecommand{\url}[1]{\texttt{#1}}
\expandafter\ifx\csname urlstyle\endcsname\relax
  \providecommand{\doi}[1]{doi: #1}\else
  \providecommand{\doi}{doi: \begingroup \urlstyle{rm}\Url}\fi

\bibitem[Acharya et~al.(2022)Acharya, Hashemi, Jain, Sanghavi, Dhillon, and
  Topcu]{AcharyaH0SDT22}
Acharya, A., Hashemi, A., Jain, P., Sanghavi, S., Dhillon, I.~S., and Topcu, U.
\newblock Robust training in high dimensions via block coordinate geometric
  median descent.
\newblock In Camps{-}Valls, G., Ruiz, F. J.~R., and Valera, I. (eds.),
  \emph{International Conference on Artificial Intelligence and Statistics,
  {AISTATS} 2022, 28-30 March 2022, Virtual Event}, volume 151 of
  \emph{Proceedings of Machine Learning Research}, pp.\  11145--11168. {PMLR},
  2022.

\bibitem[Aghakhani et~al.(2021)Aghakhani, Meng, Wang, Kruegel, and
  Vigna]{aghakhani2021bullseye}
Aghakhani, H., Meng, D., Wang, Y.-X., Kruegel, C., and Vigna, G.
\newblock Bullseye polytope: A scalable clean-label poisoning attack with
  improved transferability, 2021.

\bibitem[Barreno et~al.(2006)Barreno, Nelson, Sears, Joseph, and
  Tygar]{Barreno06}
Barreno, M., Nelson, B., Sears, R., Joseph, A.~D., and Tygar, J.~D.
\newblock Can machine learning be secure?
\newblock In \emph{Proceedings of the 2006 ACM Symposium on Information,
  Computer and Communications Security}, ASIACCS '06, pp.\  16–25, New York,
  NY, USA, 2006. Association for Computing Machinery.
\newblock ISBN 1595932720.
\newblock \doi{10.1145/1128817.1128824}.

\bibitem[Baruch et~al.(2019)Baruch, Baruch, and Goldberg]{baruch19}
Baruch, G., Baruch, M., and Goldberg, Y.
\newblock A little is enough: Circumventing defenses for distributed learning.
\newblock In Wallach, H., Larochelle, H., Beygelzimer, A., d\textquotesingle
  Alch\'{e}-Buc, F., Fox, E., and Garnett, R. (eds.), \emph{Advances in Neural
  Information Processing Systems}, volume~32. Curran Associates, Inc., 2019.

\bibitem[Ben-Porat \& Tennenholtz(2017)Ben-Porat and Tennenholtz]{ben17}
Ben-Porat, O. and Tennenholtz, M.
\newblock Best response regression.
\newblock In Guyon, I., Luxburg, U.~V., Bengio, S., Wallach, H., Fergus, R.,
  Vishwanathan, S., and Garnett, R. (eds.), \emph{Advances in Neural
  Information Processing Systems}, volume~30. Curran Associates, Inc., 2017.

\bibitem[Bender et~al.(2021)Bender, Gebru, McMillan{-}Major, and
  Shmitchell]{BenderGMS21}
Bender, E.~M., Gebru, T., McMillan{-}Major, A., and Shmitchell, S.
\newblock On the dangers of stochastic parrots: Can language models be too big?
\newblock In Elish, M.~C., Isaac, W., and Zemel, R.~S. (eds.), \emph{FAccT '21:
  2021 {ACM} Conference on Fairness, Accountability, and Transparency, Virtual
  Event / Toronto, Canada, March 3-10, 2021}, pp.\  610--623. {ACM}, 2021.

\bibitem[Biggio et~al.(2012)Biggio, Nelson, and Laskov]{BiggioNL12}
Biggio, B., Nelson, B., and Laskov, P.
\newblock Poisoning attacks against support vector machines.
\newblock In \emph{Proceedings of the 29th International Conference on Machine
  Learning, {ICML} 2012, Edinburgh, Scotland, UK, June 26 - July 1, 2012}.
  icml.cc / Omnipress, 2012.

\bibitem[Blanchard et~al.(2017)Blanchard, Mhamdi, Guerraoui, and
  Stainer]{BlanchardMGS17}
Blanchard, P., Mhamdi, E. M.~E., Guerraoui, R., and Stainer, J.
\newblock Machine learning with adversaries: {B}yzantine tolerant gradient
  descent.
\newblock In Guyon, I., von Luxburg, U., Bengio, S., Wallach, H.~M., Fergus,
  R., Vishwanathan, S. V.~N., and Garnett, R. (eds.), \emph{Advances in Neural
  Information Processing Systems 30: Annual Conference on Neural Information
  Processing Systems 2017, 4-9 December 2017, Long Beach, CA, {USA}}, pp.\
  119--129, 2017.

\bibitem[Blum et~al.(2017)Blum, Haghtalab, Procaccia, and Qiao]{BlumHPQ17}
Blum, A., Haghtalab, N., Procaccia, A.~D., and Qiao, M.
\newblock Collaborative {PAC} learning.
\newblock In Guyon, I., von Luxburg, U., Bengio, S., Wallach, H.~M., Fergus,
  R., Vishwanathan, S. V.~N., and Garnett, R. (eds.), \emph{Advances in Neural
  Information Processing Systems 30: Annual Conference on Neural Information
  Processing Systems 2017, December 4-9, 2017, Long Beach, CA, {USA}}, pp.\
  2392--2401, 2017.

\bibitem[Bradshaw \& Howard(2019)Bradshaw and Howard]{bradshaw19}
Bradshaw, S. and Howard, P.~N.
\newblock \emph{The global disinformation order: 2019 global inventory of
  organised social media manipulation}.
\newblock Project on Computational Propaganda, 2019.

\bibitem[Brown et~al.(2020)Brown, Mann, Ryder, Subbiah, Kaplan, Dhariwal,
  Neelakantan, Shyam, Sastry, Askell, Agarwal, Herbert{-}Voss, Krueger,
  Henighan, Child, Ramesh, Ziegler, Wu, Winter, Hesse, Chen, Sigler, Litwin,
  Gray, Chess, Clark, Berner, McCandlish, Radford, Sutskever, and
  Amodei]{BrownMRSKDNSSAA20}
Brown, T.~B., Mann, B., Ryder, N., Subbiah, M., Kaplan, J., Dhariwal, P.,
  Neelakantan, A., Shyam, P., Sastry, G., Askell, A., Agarwal, S.,
  Herbert{-}Voss, A., Krueger, G., Henighan, T., Child, R., Ramesh, A.,
  Ziegler, D.~M., Wu, J., Winter, C., Hesse, C., Chen, M., Sigler, E., Litwin,
  M., Gray, S., Chess, B., Clark, J., Berner, C., McCandlish, S., Radford, A.,
  Sutskever, I., and Amodei, D.
\newblock Language models are few-shot learners.
\newblock In Larochelle, H., Ranzato, M., Hadsell, R., Balcan, M., and Lin, H.
  (eds.), \emph{Advances in Neural Information Processing Systems 33: Annual
  Conference on Neural Information Processing Systems 2020, NeurIPS 2020,
  December 6-12, 2020, virtual}, 2020.

\bibitem[Cai et~al.(2015)Cai, Daskalakis, and Papadimitriou]{yang15}
Cai, Y., Daskalakis, C., and Papadimitriou, C.~H.
\newblock Optimum statistical estimation with strategic data sources.
\newblock In Gr{\"{u}}nwald, P., Hazan, E., and Kale, S. (eds.),
  \emph{Proceedings of The 28th Conference on Learning Theory, {COLT} 2015,
  Paris, France, July 3-6, 2015}, volume~40 of \emph{{JMLR} Workshop and
  Conference Proceedings}, pp.\  280--296. JMLR.org, 2015.

\bibitem[Chen et~al.(2018{\natexlab{a}})Chen, Zhang, and Zhou]{Chen0Z18}
Chen, J., Zhang, Q., and Zhou, Y.
\newblock Tight bounds for collaborative {PAC} learning via multiplicative
  weights.
\newblock In Bengio, S., Wallach, H.~M., Larochelle, H., Grauman, K.,
  Cesa{-}Bianchi, N., and Garnett, R. (eds.), \emph{Advances in Neural
  Information Processing Systems 31: Annual Conference on Neural Information
  Processing Systems 2018, NeurIPS 2018, December 3-8, 2018, Montr{\'{e}}al,
  Canada}, pp.\  3602--3611, 2018{\natexlab{a}}.

\bibitem[Chen et~al.(2018{\natexlab{b}})Chen, Podimata, Procaccia, and
  Shah]{ChenPPS18}
Chen, Y., Podimata, C., Procaccia, A.~D., and Shah, N.
\newblock Strategyproof linear regression in high dimensions.
\newblock In \emph{Proceedings of the 2018 ACM Conference on Economics and
  Computation}, EC '18, pp.\  9–26, New York, NY, USA, 2018{\natexlab{b}}.
  Association for Computing Machinery.
\newblock ISBN 9781450358293.
\newblock \doi{10.1145/3219166.3219175}.

\bibitem[Chen et~al.(2020)Chen, Liu, and Podimata]{chen2020}
Chen, Y., Liu, Y., and Podimata, C.
\newblock Learning strategy-aware linear classifiers.
\newblock In Larochelle, H., Ranzato, M., Hadsell, R., Balcan, M.~F., and Lin,
  H. (eds.), \emph{Advances in Neural Information Processing Systems},
  volume~33, pp.\  15265--15276. Curran Associates, Inc., 2020.

\bibitem[Collins et~al.(2021)Collins, Hassani, Mokhtari, and
  Shakkottai]{CollinsHMS21}
Collins, L., Hassani, H., Mokhtari, A., and Shakkottai, S.
\newblock Exploiting shared representations for personalized federated
  learning.
\newblock In Meila, M. and Zhang, T. (eds.), \emph{Proceedings of the 38th
  International Conference on Machine Learning, {ICML} 2021, 18-24 July 2021,
  Virtual Event}, volume 139 of \emph{Proceedings of Machine Learning
  Research}, pp.\  2089--2099. {PMLR}, 2021.

\bibitem[Dai et~al.(2019)Dai, Chen, and Li]{DaiCL19}
Dai, J., Chen, C., and Li, Y.
\newblock A backdoor attack against lstm-based text classification systems.
\newblock \emph{{IEEE} Access}, 7:\penalty0 138872--138878, 2019.

\bibitem[Dekel et~al.(2010)Dekel, Fischer, and Procaccia]{DEKEL2010759}
Dekel, O., Fischer, F., and Procaccia, A.~D.
\newblock Incentive compatible regression learning.
\newblock \emph{Journal of Computer and System Sciences}, 76\penalty0
  (8):\penalty0 759--777, 2010.
\newblock ISSN 0022-0000.
\newblock \doi{https://doi.org/10.1016/j.jcss.2010.03.003}.

\bibitem[Dinh et~al.(2020)Dinh, Tran, and Nguyen]{DinhTN20}
Dinh, C.~T., Tran, N.~H., and Nguyen, T.~D.
\newblock Personalized federated learning with moreau envelopes.
\newblock In Larochelle, H., Ranzato, M., Hadsell, R., Balcan, M., and Lin, H.
  (eds.), \emph{Advances in Neural Information Processing Systems 33: Annual
  Conference on Neural Information Processing Systems 2020, NeurIPS 2020,
  December 6-12, 2020, virtual}, 2020.

\bibitem[El{-}Mhamdi et~al.(2020)El{-}Mhamdi, Guerraoui, Guirguis, Hoang, and
  Rouault]{El-MhamdiGGHR20}
El{-}Mhamdi, E., Guerraoui, R., Guirguis, A., Hoang, L.~N., and Rouault, S.
\newblock Genuinely distributed {B}yzantine machine learning.
\newblock In Emek, Y. and Cachin, C. (eds.), \emph{{PODC} '20: {ACM} Symposium
  on Principles of Distributed Computing, Virtual Event, Italy, August 3-7,
  2020}, pp.\  355--364. {ACM}, 2020.

\bibitem[El{-}Mhamdi et~al.(2021{\natexlab{a}})El{-}Mhamdi, Farhadkhani,
  Guerraoui, Guirguis, Hoang, and Rouault]{collaborative_learning}
El{-}Mhamdi, E., Farhadkhani, S., Guerraoui, R., Guirguis, A., Hoang, L.~N.,
  and Rouault, S.
\newblock Collaborative learning in the jungle (decentralized, {B}yzantine,
  heterogeneous, asynchronous and nonconvex learning).
\newblock In \emph{Advances in Neural Information Processing Systems 34: Annual
  Conference on Neural Information Processing Systems 2021, December 6-14,
  2021}, 2021{\natexlab{a}}.

\bibitem[El{-}Mhamdi et~al.(2021{\natexlab{b}})El{-}Mhamdi, Farhadkhani,
  Guerraoui, and Hoang]{geometric_median}
El{-}Mhamdi, E., Farhadkhani, S., Guerraoui, R., and Hoang, L.~N.
\newblock Strategyproofness of the geometric median.
\newblock \emph{CoRR}, 2021{\natexlab{b}}.

\bibitem[El-Mhamdi et~al.(2021)El-Mhamdi, Guerraoui, and Rouault]{MhamdiGR20}
El-Mhamdi, E.-M., Guerraoui, R., and Rouault, S.
\newblock Distributed momentum for {B}yzantine-resilient stochastic gradient
  descent.
\newblock In \emph{9th International Conference on Learning Representations,
  {ICLR} 2021, Vienna, Austria, May 4–8, 2021}. OpenReview.net, 2021.

\bibitem[Fallah et~al.(2020)Fallah, Mokhtari, and Ozdaglar]{FallahMO20}
Fallah, A., Mokhtari, A., and Ozdaglar, A.~E.
\newblock Personalized federated learning with theoretical guarantees: {A}
  model-agnostic meta-learning approach.
\newblock In Larochelle, H., Ranzato, M., Hadsell, R., Balcan, M., and Lin, H.
  (eds.), \emph{Advances in Neural Information Processing Systems 33: Annual
  Conference on Neural Information Processing Systems 2020, NeurIPS 2020,
  December 6-12, 2020, virtual}, 2020.

\bibitem[Farhadkhani et~al.(2021)Farhadkhani, Guerraoui, and
  Hoang]{FarhadkhaniGH21}
Farhadkhani, S., Guerraoui, R., and Hoang, L.
\newblock Strategyproof learning: Building trustworthy user-generated datasets.
\newblock \emph{CoRR}, abs/2106.02398, 2021.

\bibitem[Fedus et~al.(2021)Fedus, Zoph, and Shazeer]{FedusZS21}
Fedus, W., Zoph, B., and Shazeer, N.
\newblock Switch transformers: Scaling to trillion parameter models with simple
  and efficient sparsity.
\newblock \emph{CoRR}, abs/2101.03961, 2021.

\bibitem[Fung \& Garcia(2019)Fung and Garcia]{facebook_fake_accounts}
Fung, B. and Garcia, A.
\newblock Facebook has shut down 5.4 billion fake accounts this year.
\newblock \emph{CNN Business}, 2019.

\bibitem[Geiping et~al.(2021)Geiping, Fowl, Huang, Czaja, Taylor, Moeller, and
  Goldstein]{geiping2021witches}
Geiping, J., Fowl, L.~H., Huang, W.~R., Czaja, W., Taylor, G., Moeller, M., and
  Goldstein, T.
\newblock Witches' brew: Industrial scale data poisoning via gradient matching.
\newblock In \emph{International Conference on Learning Representations}, 2021.

\bibitem[Goodfellow et~al.(2020)Goodfellow, Pouget{-}Abadie, Mirza, Xu,
  Warde{-}Farley, Ozair, Courville, and Bengio]{GoodfellowPMXWO20}
Goodfellow, I.~J., Pouget{-}Abadie, J., Mirza, M., Xu, B., Warde{-}Farley, D.,
  Ozair, S., Courville, A.~C., and Bengio, Y.
\newblock Generative adversarial networks.
\newblock \emph{Commun. {ACM}}, 63\penalty0 (11):\penalty0 139--144, 2020.

\bibitem[Hanzely \& Richtárik(2021)Hanzely and
  Richtárik]{hanzely2021federated}
Hanzely, F. and Richtárik, P.
\newblock Federated learning of a mixture of global and local models, 2021.

\bibitem[Hanzely et~al.(2020)Hanzely, Hanzely, Horv{\'{a}}th, and
  Richt{\'{a}}rik]{HanzelyHHR20}
Hanzely, F., Hanzely, S., Horv{\'{a}}th, S., and Richt{\'{a}}rik, P.
\newblock Lower bounds and optimal algorithms for personalized federated
  learning.
\newblock In Larochelle, H., Ranzato, M., Hadsell, R., Balcan, M., and Lin, H.
  (eds.), \emph{Advances in Neural Information Processing Systems 33: Annual
  Conference on Neural Information Processing Systems 2020, NeurIPS 2020,
  December 6-12, 2020, virtual}, 2020.

\bibitem[Hardt et~al.(2016)Hardt, Megiddo, Papadimitriou, and
  Wootters]{hardt16}
Hardt, M., Megiddo, N., Papadimitriou, C., and Wootters, M.
\newblock Strategic classification.
\newblock In \emph{Proceedings of the 2016 ACM Conference on Innovations in
  Theoretical Computer Science}, ITCS '16, pp.\  111–122, New York, NY, USA,
  2016. Association for Computing Machinery.
\newblock ISBN 9781450340571.
\newblock \doi{10.1145/2840728.2840730}.

\bibitem[He et~al.(2020)He, Karimireddy, and Jaggi]{He2020}
He, L., Karimireddy, S.~P., and Jaggi, M.
\newblock Byzantine-robust learning on heterogeneous datasets via resampling.
\newblock \emph{CoRR}, abs/2006.09365, 2020.

\bibitem[Hoang(2020)]{Hoang20}
Hoang, L.~N.
\newblock Science communication desperately needs more aligned recommendation
  algorithms.
\newblock \emph{Frontiers in Communication}, 5:\penalty0 115, 2020.

\bibitem[Hoang et~al.(2021)Hoang, Faucon, and El{-}Mhamdi]{HoangFE21}
Hoang, L.~N., Faucon, L., and El{-}Mhamdi, E.
\newblock Recommendation algorithms, a neglected opportunity for public health.
\newblock \emph{Revue M\'edecine et Philosophie}, 4\penalty0 (2):\penalty0
  16--24, 2021.

\bibitem[Horn \& Johnson(2012)Horn and Johnson]{horn2012}
Horn, R.~A. and Johnson, C.~R.
\newblock \emph{Matrix Analysis}.
\newblock Cambridge University Press, 2 edition, 2012.
\newblock \doi{10.1017/9781139020411}.

\bibitem[Huang et~al.(2020)Huang, Geiping, Fowl, Taylor, and
  Goldstein]{HuangGFTG20}
Huang, W.~R., Geiping, J., Fowl, L., Taylor, G., and Goldstein, T.
\newblock Metapoison: Practical general-purpose clean-label data poisoning.
\newblock In Larochelle, H., Ranzato, M., Hadsell, R., Balcan, M., and Lin, H.
  (eds.), \emph{Advances in Neural Information Processing Systems 33: Annual
  Conference on Neural Information Processing Systems 2020, NeurIPS 2020,
  December 6-12, 2020, virtual}, 2020.

\bibitem[Ie et~al.(2019)Ie, Jain, Wang, Narvekar, Agarwal, Wu, Cheng, Chandra,
  and Boutilier]{IeJWNAWCCB19}
Ie, E., Jain, V., Wang, J., Narvekar, S., Agarwal, R., Wu, R., Cheng, H.,
  Chandra, T., and Boutilier, C.
\newblock Slateq: {A} tractable decomposition for reinforcement learning with
  recommendation sets.
\newblock In Kraus, S. (ed.), \emph{Proceedings of the Twenty-Eighth
  International Joint Conference on Artificial Intelligence, {IJCAI} 2019,
  Macao, China, August 10-16, 2019}, pp.\  2592--2599. ijcai.org, 2019.

\bibitem[Jain \& Orlitsky(2020)Jain and Orlitsky]{JainO20a}
Jain, A. and Orlitsky, A.
\newblock A general method for robust learning from batches.
\newblock In Larochelle, H., Ranzato, M., Hadsell, R., Balcan, M., and Lin, H.
  (eds.), \emph{Advances in Neural Information Processing Systems 33: Annual
  Conference on Neural Information Processing Systems 2020, NeurIPS 2020,
  December 6-12, 2020, virtual}, 2020.

\bibitem[Johnson \& Diakopoulos(2021)Johnson and Diakopoulos]{JohnsonD21}
Johnson, D.~G. and Diakopoulos, N.
\newblock What to do about deepfakes.
\newblock \emph{Commun. {ACM}}, 64\penalty0 (3):\penalty0 33--35, 2021.

\bibitem[Kairouz et~al.(2021)Kairouz, McMahan, Avent, Bellet, Bennis, Bhagoji,
  Bonawitz, Charles, Cormode, Cummings, D'Oliveira, Eichner, Rouayheb, Evans,
  Gardner, Garrett, Gascón, Ghazi, Gibbons, Gruteser, Harchaoui, He, He, Huo,
  Hutchinson, Hsu, Jaggi, Javidi, Joshi, Khodak, Konečný, Korolova,
  Koushanfar, Koyejo, Lepoint, Liu, Mittal, Mohri, Nock, Özgür, Pagh,
  Raykova, Qi, Ramage, Raskar, Song, Song, Stich, Sun, Suresh, Tramèr,
  Vepakomma, Wang, Xiong, Xu, Yang, Yu, Yu, and Zhao]{kairouz2021advances}
Kairouz, P., McMahan, H.~B., Avent, B., Bellet, A., Bennis, M., Bhagoji, A.~N.,
  Bonawitz, K., Charles, Z., Cormode, G., Cummings, R., D'Oliveira, R. G.~L.,
  Eichner, H., Rouayheb, S.~E., Evans, D., Gardner, J., Garrett, Z., Gascón,
  A., Ghazi, B., Gibbons, P.~B., Gruteser, M., Harchaoui, Z., He, C., He, L.,
  Huo, Z., Hutchinson, B., Hsu, J., Jaggi, M., Javidi, T., Joshi, G., Khodak,
  M., Konečný, J., Korolova, A., Koushanfar, F., Koyejo, S., Lepoint, T.,
  Liu, Y., Mittal, P., Mohri, M., Nock, R., Özgür, A., Pagh, R., Raykova, M.,
  Qi, H., Ramage, D., Raskar, R., Song, D., Song, W., Stich, S.~U., Sun, Z.,
  Suresh, A.~T., Tramèr, F., Vepakomma, P., Wang, J., Xiong, L., Xu, Z., Yang,
  Q., Yu, F.~X., Yu, H., and Zhao, S.
\newblock Advances and open problems in federated learning, 2021.

\bibitem[Karimireddy et~al.(2021)Karimireddy, He, and Jaggi]{KarimireddyHJ21}
Karimireddy, S.~P., He, L., and Jaggi, M.
\newblock Learning from history for {B}yzantine robust optimization.
\newblock In Meila, M. and Zhang, T. (eds.), \emph{Proceedings of the 38th
  International Conference on Machine Learning, {ICML} 2021, 18-24 July 2021,
  Virtual Event}, volume 139 of \emph{Proceedings of Machine Learning
  Research}, pp.\  5311--5319. {PMLR}, 2021.

\bibitem[Konecn{\'{y}} et~al.(2015)Konecn{\'{y}}, McMahan, and
  Ramage]{KonecnyMR15}
Konecn{\'{y}}, J., McMahan, B., and Ramage, D.
\newblock Federated optimization: Distributed optimization beyond the
  datacenter.
\newblock \emph{CoRR}, abs/1511.03575, 2015.

\bibitem[Konstantinov et~al.(2020)Konstantinov, Frantar, Alistarh, and
  Lampert]{KonstantinovFAL20}
Konstantinov, N., Frantar, E., Alistarh, D., and Lampert, C.
\newblock On the sample complexity of adversarial multi-source {PAC} learning.
\newblock In \emph{Proceedings of the 37th International Conference on Machine
  Learning, {ICML} 2020, 13-18 July 2020, Virtual Event}, volume 119 of
  \emph{Proceedings of Machine Learning Research}, pp.\  5416--5425. {PMLR},
  2020.

\bibitem[Kumar et~al.(2020)Kumar, Nystr{\"{o}}m, Lambert, Marshall, Goertzel,
  Comissoneru, Swann, and Xia]{KumarNLMGCSX20}
Kumar, R. S.~S., Nystr{\"{o}}m, M., Lambert, J., Marshall, A., Goertzel, M.,
  Comissoneru, A., Swann, M., and Xia, S.
\newblock Adversarial machine learning-industry perspectives.
\newblock In \emph{2020 {IEEE} Security and Privacy Workshops, {SP} Workshops,
  San Francisco, CA, USA, May 21, 2020}, pp.\  69--75. {IEEE}, 2020.

\bibitem[Lehmann \& Buschek(2021)Lehmann and Buschek]{LehmannB21}
Lehmann, F. and Buschek, D.
\newblock Examining autocompletion as a basic concept for interaction with
  generative {AI}.
\newblock \emph{i-com}, 19\penalty0 (3):\penalty0 251--264, 2021.

\bibitem[Mahloujifar et~al.(2019)Mahloujifar, Mahmoody, and
  Mohammed]{MahloujifarMM19}
Mahloujifar, S., Mahmoody, M., and Mohammed, A.
\newblock Data poisoning attacks in multi-party learning.
\newblock In Chaudhuri, K. and Salakhutdinov, R. (eds.), \emph{Proceedings of
  the 36th International Conference on Machine Learning, {ICML} 2019, 9-15 June
  2019, Long Beach, California, {USA}}, volume~97 of \emph{Proceedings of
  Machine Learning Research}, pp.\  4274--4283. {PMLR}, 2019.

\bibitem[Mai et~al.(2019)Mai, Cao, Yuen, and Jain]{MaiCYJ19}
Mai, G., Cao, K., Yuen, P.~C., and Jain, A.~K.
\newblock On the reconstruction of face images from deep face templates.
\newblock \emph{{IEEE} Trans. Pattern Anal. Mach. Intell.}, 41\penalty0
  (5):\penalty0 1188--1202, 2019.

\bibitem[McGuffie \& Newhouse(2020)McGuffie and Newhouse]{McGuffieNewhouse20}
McGuffie, K. and Newhouse, A.
\newblock The radicalization risks of {GPT-3} and advanced neural language
  models.
\newblock \emph{CoRR}, abs/2009.06807, 2020.

\bibitem[Meir et~al.(2011)Meir, Almagor, Michaely, and Rosenschein]{Meir11}
Meir, R., Almagor, S., Michaely, A., and Rosenschein, J.~S.
\newblock Tight bounds for strategyproof classification.
\newblock In \emph{The 10th International Conference on Autonomous Agents and
  Multiagent Systems - Volume 1}, AAMAS '11, pp.\  319–326, Richland, SC,
  2011. International Foundation for Autonomous Agents and Multiagent Systems.
\newblock ISBN 0982657153.

\bibitem[Meir et~al.(2012)Meir, Procaccia, and Rosenschein]{MEIR2012123}
Meir, R., Procaccia, A.~D., and Rosenschein, J.~S.
\newblock Algorithms for strategyproof classification.
\newblock \emph{Artificial Intelligence}, 186:\penalty0 123--156, 2012.
\newblock ISSN 0004-3702.
\newblock \doi{https://doi.org/10.1016/j.artint.2012.03.008}.

\bibitem[Mhamdi et~al.(2018)Mhamdi, Guerraoui, and Rouault]{MhamdiGR18}
Mhamdi, E. M.~E., Guerraoui, R., and Rouault, S.
\newblock The hidden vulnerability of distributed learning in byzantium.
\newblock In Dy, J.~G. and Krause, A. (eds.), \emph{Proceedings of the 35th
  International Conference on Machine Learning, {ICML} 2018,
  Stockholmsm{\"{a}}ssan, Stockholm, Sweden, July 10-15, 2018}, volume~80 of
  \emph{Proceedings of Machine Learning Research}, pp.\  3518--3527. {PMLR},
  2018.

\bibitem[Mu{\~{n}}oz{-}Gonz{\'{a}}lez et~al.(2017)Mu{\~{n}}oz{-}Gonz{\'{a}}lez,
  Biggio, Demontis, Paudice, Wongrassamee, Lupu, and Roli]{Munoz-GonzalezB17}
Mu{\~{n}}oz{-}Gonz{\'{a}}lez, L., Biggio, B., Demontis, A., Paudice, A.,
  Wongrassamee, V., Lupu, E.~C., and Roli, F.
\newblock Towards poisoning of deep learning algorithms with back-gradient
  optimization.
\newblock In Thuraisingham, B.~M., Biggio, B., Freeman, D.~M., Miller, B., and
  Sinha, A. (eds.), \emph{Proceedings of the 10th {ACM} Workshop on Artificial
  Intelligence and Security, AISec@CCS 2017, Dallas, TX, USA, November 3,
  2017}, pp.\  27--38. {ACM}, 2017.

\bibitem[Neudert et~al.(2019)Neudert, Howard, and Kollanyi]{neudert2019}
Neudert, L.-M., Howard, P., and Kollanyi, B.
\newblock Sourcing and automation of political news and information during
  three european elections.
\newblock \emph{Social Media+ Society}, 5\penalty0 (3):\penalty0
  2056305119863147, 2019.

\bibitem[Nguyen \& Zakynthinou(2018)Nguyen and Zakynthinou]{NguyenZ18}
Nguyen, H.~L. and Zakynthinou, L.
\newblock Improved algorithms for collaborative {PAC} learning.
\newblock In Bengio, S., Wallach, H.~M., Larochelle, H., Grauman, K.,
  Cesa{-}Bianchi, N., and Garnett, R. (eds.), \emph{Advances in Neural
  Information Processing Systems 31: Annual Conference on Neural Information
  Processing Systems 2018, NeurIPS 2018, December 3-8, 2018, Montr{\'{e}}al,
  Canada}, pp.\  7642--7650, 2018.

\bibitem[Perote \& Perote-Peña(2004)Perote and Perote-Peña]{PEROTE2004153}
Perote, J. and Perote-Peña, J.
\newblock Strategy-proof estimators for simple regression.
\newblock \emph{Mathematical Social Sciences}, 47\penalty0 (2):\penalty0
  153--176, 2004.
\newblock ISSN 0165-4896.
\newblock \doi{https://doi.org/10.1016/S0165-4896(03)00085-4}.

\bibitem[Perote \& Sevilla(2003)Perote and Sevilla]{Perote04}
Perote, J. and Sevilla, O.
\newblock The impossibility of strategy-proof clustering.
\newblock \emph{Economics Bulletin}, 2003.

\bibitem[Phan(2021)]{huy_phan_2021_4431043}
Phan, H.
\newblock huyvnphan/pytorch\_cifar10, January 2021.

\bibitem[Qiao(2018)]{Qiao18}
Qiao, M.
\newblock Do outliers ruin collaboration?
\newblock In Dy, J.~G. and Krause, A. (eds.), \emph{Proceedings of the 35th
  International Conference on Machine Learning, {ICML} 2018,
  Stockholmsm{\"{a}}ssan, Stockholm, Sweden, July 10-15, 2018}, volume~80 of
  \emph{Proceedings of Machine Learning Research}, pp.\  4177--4184. {PMLR},
  2018.

\bibitem[Ricci et~al.(2011)Ricci, Rokach, and Shapira]{RicciRS11}
Ricci, F., Rokach, L., and Shapira, B.
\newblock Introduction to recommender systems handbook.
\newblock In Ricci, F., Rokach, L., Shapira, B., and Kantor, P.~B. (eds.),
  \emph{Recommender Systems Handbook}, pp.\  1--35. Springer, 2011.

\bibitem[Schwarzschild et~al.(2021)Schwarzschild, Goldblum, Gupta, Dickerson,
  and Goldstein]{schwarzschild21a}
Schwarzschild, A., Goldblum, M., Gupta, A., Dickerson, J.~P., and Goldstein, T.
\newblock Just how toxic is data poisoning? a unified benchmark for backdoor
  and data poisoning attacks.
\newblock In Meila, M. and Zhang, T. (eds.), \emph{Proceedings of the 38th
  International Conference on Machine Learning}, volume 139 of
  \emph{Proceedings of Machine Learning Research}, pp.\  9389--9398. PMLR,
  18--24 Jul 2021.

\bibitem[Severi et~al.(2021)Severi, Meyer, Coull, and Oprea]{SeveriMCO21}
Severi, G., Meyer, J., Coull, S., and Oprea, A.
\newblock Explanation-guided backdoor poisoning attacks against malware
  classifiers.
\newblock In Bailey, M. and Greenstadt, R. (eds.), \emph{30th {USENIX} Security
  Symposium, {USENIX} Security 2021, August 11-13, 2021}, pp.\  1487--1504.
  {USENIX} Association, 2021.

\bibitem[Shafahi et~al.(2018)Shafahi, Huang, Najibi, Suciu, Studer, Dumitras,
  and Goldstein]{ShafahiHNSSDG18}
Shafahi, A., Huang, W.~R., Najibi, M., Suciu, O., Studer, C., Dumitras, T., and
  Goldstein, T.
\newblock Poison frogs! targeted clean-label poisoning attacks on neural
  networks.
\newblock In Bengio, S., Wallach, H.~M., Larochelle, H., Grauman, K.,
  Cesa{-}Bianchi, N., and Garnett, R. (eds.), \emph{Advances in Neural
  Information Processing Systems 31: Annual Conference on Neural Information
  Processing Systems 2018, NeurIPS 2018, December 3-8, 2018, Montr{\'{e}}al,
  Canada}, pp.\  6106--6116, 2018.

\bibitem[Shejwalkar et~al.(2022)Shejwalkar, Houmansadr, Kairouz, and
  Ramage]{ShejwalkarHKR21}
Shejwalkar, V., Houmansadr, A., Kairouz, P., and Ramage, D.
\newblock Back to the drawing board: {A} critical evaluation of poisoning
  attacks on federated learning.
\newblock In \emph{2022 IEEE Symposium on Security and Privacy}, 2022.

\bibitem[Shum et~al.(2018)Shum, He, and Li]{ShumHL18}
Shum, H., He, X., and Li, D.
\newblock From eliza to xiaoice: challenges and opportunities with social
  chatbots.
\newblock \emph{Frontiers Inf. Technol. Electron. Eng.}, 19\penalty0
  (1):\penalty0 10--26, 2018.

\bibitem[Smith et~al.(2013)Smith, Saint{-}Amand, Plamada, Koehn,
  Callison{-}Burch, and Lopez]{SmithSPKCL13}
Smith, J.~R., Saint{-}Amand, H., Plamada, M., Koehn, P., Callison{-}Burch, C.,
  and Lopez, A.
\newblock Dirt cheap web-scale parallel text from the common crawl.
\newblock In \emph{Proceedings of the 51st Annual Meeting of the Association
  for Computational Linguistics, {ACL} 2013, 4-9 August 2013, Sofia, Bulgaria,
  Volume 1: Long Papers}, pp.\  1374--1383. The Association for Computer
  Linguistics, 2013.

\bibitem[Suya et~al.(2021)Suya, Mahloujifar, Suri, Evans, and Tian]{SuyaMS0021}
Suya, F., Mahloujifar, S., Suri, A., Evans, D., and Tian, Y.
\newblock Model-targeted poisoning attacks with provable convergence.
\newblock In Meila, M. and Zhang, T. (eds.), \emph{Proceedings of the 38th
  International Conference on Machine Learning, {ICML} 2021, 18-24 July 2021,
  Virtual Event}, volume 139 of \emph{Proceedings of Machine Learning
  Research}, pp.\  10000--10010. {PMLR}, 2021.

\bibitem[Truong et~al.(2020)Truong, Jones, Hutchinson, August, Praggastis,
  Jasper, Nichols, and Tuor]{TruongJHAPJNT20}
Truong, L., Jones, C., Hutchinson, B., August, A., Praggastis, B., Jasper, R.,
  Nichols, N., and Tuor, A.
\newblock Systematic evaluation of backdoor data poisoning attacks on image
  classifiers.
\newblock In \emph{2020 {IEEE/CVF} Conference on Computer Vision and Pattern
  Recognition, {CVPR} Workshops 2020, Seattle, WA, USA, June 14-19, 2020}, pp.\
   3422--3431. Computer Vision Foundation / {IEEE}, 2020.

\bibitem[Valiant(1984)]{Valiant84}
Valiant, L.~G.
\newblock A theory of the learnable.
\newblock \emph{Commun. {ACM}}, 27\penalty0 (11):\penalty0 1134--1142, 1984.

\bibitem[Vershynin(2018)]{vershynin18}
Vershynin, R.
\newblock \emph{High-dimensional probability: An introduction with applications
  in data science}, volume~47.
\newblock Cambridge university press, 2018.

\bibitem[Wainwright(2019)]{wainwright19}
Wainwright, M.~J.
\newblock \emph{High-Dimensional Statistics: A Non-Asymptotic Viewpoint}.
\newblock Cambridge Series in Statistical and Probabilistic Mathematics.
  Cambridge University Press, 2019.
\newblock \doi{10.1017/9781108627771}.

\bibitem[Wang et~al.(2019{\natexlab{a}})Wang, Pruksachatkun, Nangia, Singh,
  Michael, Hill, Levy, and Bowman]{WangPNSMHLB19}
Wang, A., Pruksachatkun, Y., Nangia, N., Singh, A., Michael, J., Hill, F.,
  Levy, O., and Bowman, S.~R.
\newblock Superglue: {A} stickier benchmark for general-purpose language
  understanding systems.
\newblock In Wallach, H.~M., Larochelle, H., Beygelzimer, A.,
  d'Alch{\'{e}}{-}Buc, F., Fox, E.~B., and Garnett, R. (eds.), \emph{Advances
  in Neural Information Processing Systems 32: Annual Conference on Neural
  Information Processing Systems 2019, NeurIPS 2019, December 8-14, 2019,
  Vancouver, BC, Canada}, pp.\  3261--3275, 2019{\natexlab{a}}.

\bibitem[Wang et~al.(2019{\natexlab{b}})Wang, Singh, Michael, Hill, Levy, and
  Bowman]{WangSMHLB19}
Wang, A., Singh, A., Michael, J., Hill, F., Levy, O., and Bowman, S.~R.
\newblock {GLUE:} {A} multi-task benchmark and analysis platform for natural
  language understanding.
\newblock In \emph{7th International Conference on Learning Representations,
  {ICLR} 2019, New Orleans, LA, USA, May 6-9, 2019}. OpenReview.net,
  2019{\natexlab{b}}.

\bibitem[Wang et~al.(2019{\natexlab{c}})Wang, Wang, Zhang, Wang, Ma, and
  Gao]{WangWZWM019}
Wang, S., Wang, S., Zhang, X., Wang, S., Ma, S., and Gao, W.
\newblock Scalable facial image compression with deep feature reconstruction.
\newblock In \emph{2019 {IEEE} International Conference on Image Processing,
  {ICIP} 2019, Taipei, Taiwan, September 22-25, 2019}, pp.\  2691--2695.
  {IEEE}, 2019{\natexlab{c}}.

\bibitem[Wu et~al.(2020)Wu, Ngai, Wu, and Wu]{WuNWW20}
Wu, Y., Ngai, E. W.~T., Wu, P., and Wu, C.
\newblock Fake online reviews: Literature review, synthesis, and directions for
  future research.
\newblock \emph{Decis. Support Syst.}, 132:\penalty0 113280, 2020.

\bibitem[Xie et~al.(2019)Xie, Koyejo, and Gupta]{XieKG19}
Xie, C., Koyejo, O., and Gupta, I.
\newblock Fall of empires: Breaking {B}yzantine-tolerant {SGD} by inner product
  manipulation.
\newblock In Globerson, A. and Silva, R. (eds.), \emph{Proceedings of the
  Thirty-Fifth Conference on Uncertainty in Artificial Intelligence, {UAI}
  2019, Tel Aviv, Israel, July 22-25, 2019}, volume 115 of \emph{Proceedings of
  Machine Learning Research}, pp.\  261--270. {AUAI} Press, 2019.

\bibitem[Yang \& Li(2021)Yang and Li]{YangL21}
Yang, Y. and Li, W.
\newblock {BASGD:} buffered asynchronous {SGD} for {B}yzantine learning.
\newblock In Meila, M. and Zhang, T. (eds.), \emph{Proceedings of the 38th
  International Conference on Machine Learning, {ICML} 2021, 18-24 July 2021,
  Virtual Event}, volume 139 of \emph{Proceedings of Machine Learning
  Research}, pp.\  11751--11761. {PMLR}, 2021.

\bibitem[Yin et~al.(2018)Yin, Chen, Ramchandran, and Bartlett]{YinCRB18}
Yin, D., Chen, Y., Ramchandran, K., and Bartlett, P.~L.
\newblock Byzantine-robust distributed learning: Towards optimal statistical
  rates.
\newblock In Dy, J.~G. and Krause, A. (eds.), \emph{Proceedings of the 35th
  International Conference on Machine Learning, {ICML} 2018,
  Stockholmsm{\"{a}}ssan, Stockholm, Sweden, July 10-15, 2018}, volume~80 of
  \emph{Proceedings of Machine Learning Research}, pp.\  5636--5645. {PMLR},
  2018.

\bibitem[Zeiler \& Fergus(2014)Zeiler and Fergus]{ZeilerF14}
Zeiler, M.~D. and Fergus, R.
\newblock Visualizing and understanding convolutional networks.
\newblock In Fleet, D.~J., Pajdla, T., Schiele, B., and Tuytelaars, T. (eds.),
  \emph{Computer Vision - {ECCV} 2014 - 13th European Conference, Zurich,
  Switzerland, September 6-12, 2014, Proceedings, Part {I}}, volume 8689 of
  \emph{Lecture Notes in Computer Science}, pp.\  818--833. Springer, 2014.

\bibitem[Zhao et~al.(2020)Zhao, Ma, Zheng, Bailey, Chen, and
  Jiang]{ZhaoMZ0CJ20}
Zhao, S., Ma, X., Zheng, X., Bailey, J., Chen, J., and Jiang, Y.
\newblock Clean-label backdoor attacks on video recognition models.
\newblock In \emph{2020 {IEEE/CVF} Conference on Computer Vision and Pattern
  Recognition, {CVPR} 2020, Seattle, WA, USA, June 13-19, 2020}, pp.\
  14431--14440. Computer Vision Foundation / {IEEE}, 2020.

\bibitem[Zhu et~al.(2019)Zhu, Huang, Li, Taylor, Studer, and
  Goldstein]{ZhuHLTSG19}
Zhu, C., Huang, W.~R., Li, H., Taylor, G., Studer, C., and Goldstein, T.
\newblock Transferable clean-label poisoning attacks on deep neural nets.
\newblock In Chaudhuri, K. and Salakhutdinov, R. (eds.), \emph{Proceedings of
  the 36th International Conference on Machine Learning, {ICML} 2019, 9-15 June
  2019, Long Beach, California, {USA}}, volume~97 of \emph{Proceedings of
  Machine Learning Research}, pp.\  7614--7623. {PMLR}, 2019.

\end{thebibliography}
